%% file: iclr2027_conference.tex
\documentclass{article} 
\usepackage{iclr2027_conference,times}

\input{math_commands.tex}

\usepackage{hyperref}
\usepackage{url}

\usepackage{booktabs}
\usepackage{cleveref} 
\usepackage{multirow}
\usepackage{graphicx}
\usepackage{amsmath}

\usepackage{subcaption}
\usepackage{enumitem}
\usepackage{verbatim}
\usepackage{amsthm}

\usepackage{wrapfig}
\usepackage{subcaption}
\usepackage{float}


\title{Investigating Single-Block Recurrence in \\ Vision Transformers for Image Recognition}

\author{\textbf{Michal Byra}\textsuperscript{1,2},
\textbf{Pawel Olszowiec}\textsuperscript{1},
\textbf{Grzegorz Stefanski}\textsuperscript{1}, 
\textbf{Grzegorz Gruszczynski}\textsuperscript{1}, \\
\textbf{Alberto Presta}\textsuperscript{1} \\
\textsuperscript{1}Samsung AI Center, Warsaw, Poland \\
\textsuperscript{2}Institute of Fundamental Technological Research, Polish Academy of Sciences, Warsaw, Poland \\
Correspondence: \texttt{byra.michal@gmail.com}
}

\iclrfinalcopy 
\begin{document}

\maketitle

\begin{abstract}
Vision Transformers (ViTs) implement depth by stacking independently parameterized blocks, but it remains unclear how much of this parameterization is necessary and how much can be replaced by recurrent reuse. We study this question with bViT, a single-block recurrent ViT that repeatedly applies the same transformer block while preserving the iterative computation of a deep model. On ImageNet-1K, bViT-B reaches 0.779 validation accuracy compared with 0.789 for ViT-B under the same training recipe and computational budget, while using 8.6M rather than 86.6M parameters. This correspondence becomes stronger with model width, while narrow recurrent models exhibit a substantial performance gap. Beyond classification, the single-block formulation provides a controlled testbed for studying how transformer computation evolves with depth, since the same heads, neurons, and weight matrices can be tracked across recurrent steps. Analyses of attention, activation patterns, and step-conditioned spectral pruning reveal temporally organized behavior and step-dependent utilization of the shared parameters. bViT also transfers competitively to downstream tasks while enabling highly parameter-efficient adaptation. Our work shows that much of the performance associated with independently parameterized ViT depth can be recovered through recurrent reuse of a single sufficiently wide transformer block.
\end{abstract}

\section{Introduction}

Transformers have become a central architecture in deep learning, achieving strong performance across language, vision, and multimodal tasks~\citep{radford2021learningtransferablevisualmodels,vaswani2017attention,xu2023multimodal}. In computer vision, Vision Transformers (\textbf{ViTs}) have proven particularly effective for image classification, transfer learning, and visual representation learning~\citep{caron2021emerging,dosovitskiy2020image,he2022masked,touvron2021training}. Standard ViTs implement depth by stacking transformer blocks with the same architectural form but independently learned parameters. This raises a basic question. To what extent does transformer depth require distinct parameters at different stages, and to what extent can it be realized by repeatedly applying the same transformation?

Several observations motivate this question. Transformer blocks operate on a common token representation and residual stream, and computations across depth may be more closely related than their independent parameterization suggests~\citep{sun2025transformer}. Empirical studies further show that transformer blocks can sometimes be merged~\citep{akiba2025evolutionary} or compressed with limited performance loss~\citep{garg2025revealing,hao2025token}, suggesting redundancy in depth-specific parameterization. Moreover, for a fixed model width, the number of transformer block parameters grows linearly with the number of independently parameterized blocks. Separating computational depth from parameter depth is therefore relevant both for understanding what distinct transformer layers contribute and for reducing the memory cost of storing, deploying, and adapting deep models.

In this work, we study a single-block recurrent ViT, termed \textbf{bViT}, in which the entire transformer stack is replaced by repeated application of one shared block. This preserves computational depth while removing depth-specific block parameterization. We train this fully shared architecture from scratch and compare it directly with standard ViTs under matched training and computational budgets. This setting allows us to study the extreme full-sharing limit of ViT depth and determine when recurrent reuse can substitute for independently parameterized blocks.

The single-block formulation also provides a controlled setting for studying how transformer computation evolves with depth. Because bViT reuses the same attention heads, feed-forward network (\textbf{FFN}) neurons, and weight matrices at every recurrent step, the same components can be tracked directly throughout the computation. Prior work has shown that visual representations in standard ViTs evolve across depth, including a progression from relatively simple texture-like patterns toward more structured features~\citep{ghiasi2022vision}. In bViT, we can ask a related but more direct question of whether fixed components change their effective role as the hidden state evolves. This comparison is less direct in standard ViTs because components at different depths are independently parameterized and need not be aligned~\citep{theusgeneralized}. While a standard ViT can allocate different stages of computation to different parameter sets, bViT must realize such stage-dependent behavior through a single shared parameterization and an evolving hidden state. 

Our main contributions are as follows:

\begin{itemize}

\item We systematically study the extreme full-sharing limit of ViT depth, replacing the entire stack of independently parameterized transformer blocks with repeated application of a single shared block. This setting separates computational depth from parameter depth and isolates the role of depth-specific parameterization.

\item We characterize the capacity limits of recurrent depth sharing and identify representation width as an important factor in determining when whole-network sharing can replace independently parameterized depth. 

\item We contribute a mechanistic analysis of computation under whole-network block sharing. Exploiting the exact parameter alignment induced by single-block recurrence, we introduce a step-conditioned spectral analysis that tracks fixed parameter directions across depth and reveals persistent and temporally organized use of the shared parameter space.

\item We show that whole-network sharing changes the parameterization of downstream adaptation. A single LoRA update can be used across the entire recurrent depth, allowing bViT to adapt competitively to downstream tasks with substantially fewer trainable parameters than a standard ViT. 

\end{itemize}

\section{Related work} 

\paragraph{Recurrence in transformers.}

Recurrence has long been studied as a way to reuse computation across depth. The Universal Transformer applies a shared transformer block recurrently~\citep{dehghani2018universal}, while ALBERT uses cross-layer parameter sharing in language models~\citep{lan2019albert}. More recent looped transformers study iterative computation in reasoning, algorithmic tasks, and in-context learning~\citep{giannou2023looped,saunshi2025reasoning,schwarzschild2021can,yang2023looped}. These works establish that repeated application of shared parameters can provide effective computational depth and support iterative forms of computation~\citep{saunshi2025reasoning}.

We use recurrence to study a different question about transformer depth. By collapsing the entire ViT backbone to a single shared block while preserving the number of block evaluations, we ask how much performance requires depth-specific parameterization and when recurrent computation can substitute for independently parameterized layers. Full sharing additionally aligns the model components across depth, allowing us to study how the same attention heads, neurons, and parameter directions change their effective role as the hidden state evolves.

\paragraph{Parameter sharing in vision transformers.}

Several works explored parameter reuse in ViTs. MiniViT shares weights while retaining lightweight layer-specific transformations~\citep{zhang2022minivit}, while GroupedMLP shares FFN parameters between neighboring blocks but keeps attention block-specific~\citep{padmanaban2025parameter}. RViT combines convolutional feature extraction with recurrent application of a transformer encoder~\citep{messina2022recurrent}. SReT introduces local recursion within depth-specific transformer blocks~\citep{shen2022sliced}. Its main architecture additionally uses non-shared nonlinear projection layers between recursive applications, while different depth positions remain independently parameterized. SReT also considers simpler recursive variants without these projections, but does not study whole-network sharing through a single block.

More recent recurrent ViTs examine different settings. LoopViT studies recurrent computation for visual reasoning in ARC-AGI~\citep{shu2026loopvit}, while Raptor approximates pretrained ViTs with a small number of recurrently reused blocks using distillation~\citep{jacobs2025block}. In contrast, we train the fully shared single-block architecture from scratch and compare it directly with standard ViTs under matched training and computational budgets. Our contribution is not recurrence itself, but a systematic study of the full-sharing limit of ViT depth, including when recurrent computation can replace independently parameterized depth, how this trade-off depends on representation width, how the resulting representations transfer, and how the same shared components acquire different roles across recurrent stages.

\section{Methods}

\begin{wrapfigure}{r}{7cm}
    \vspace{-4mm}
    \centering

    \includegraphics[width=7cm]{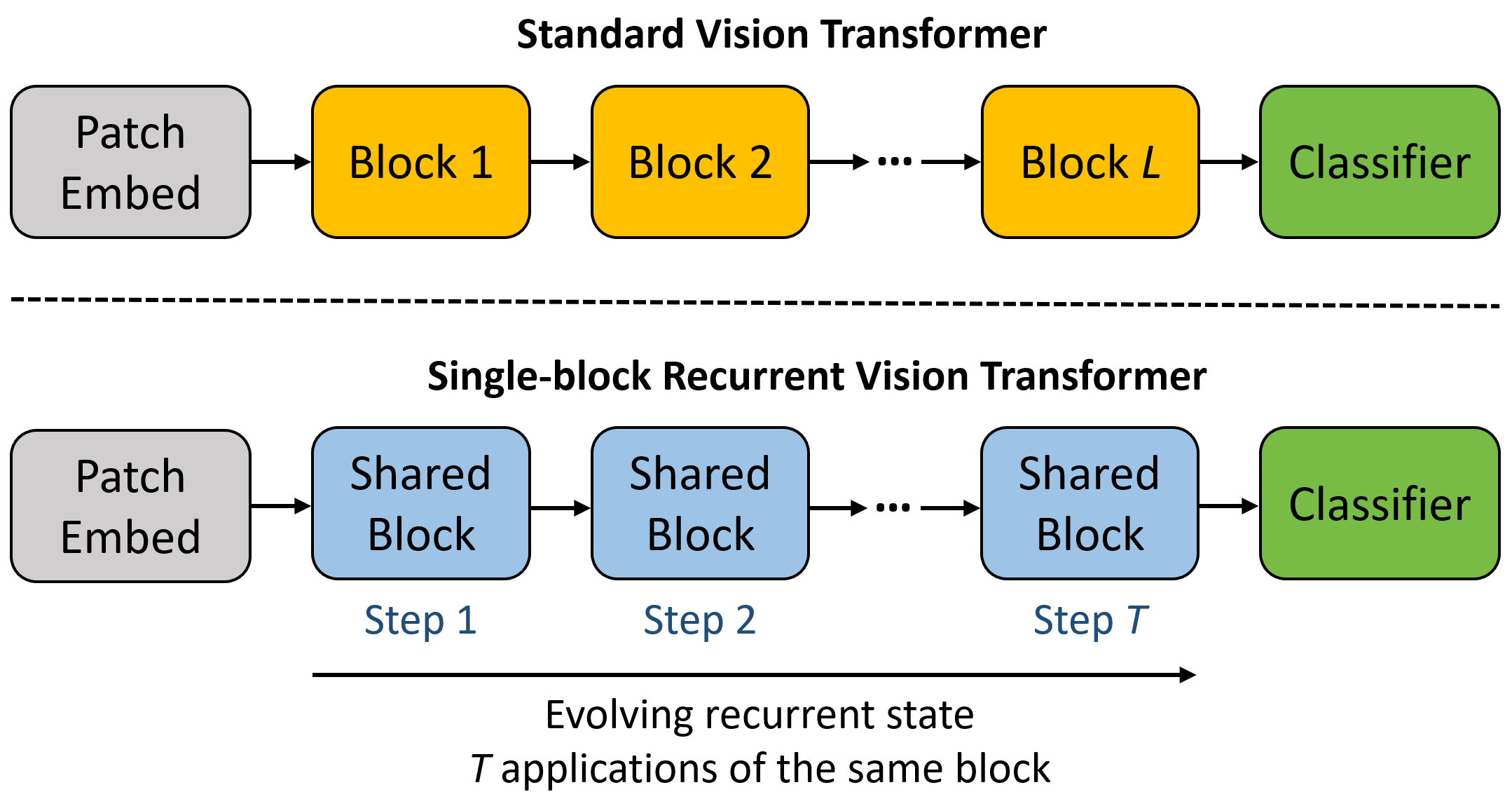}

\caption{We investigate a single-block vision transformer, bViT, in which the same transformer block is applied recurrently across depth. Reusing a single shared block enables mechanistic analysis by, for instance, allowing the same attention heads and FFN neurons to be tracked across recurrent steps.}
    \label{fig:arch_1}

    \vspace{-5mm}
\end{wrapfigure}

\subsection{Vision Transformers and single-block recurrence}

We consider a standard vision transformer (ViT)~\citep{dosovitskiy2020image}, which maps an image to a sequence of patch embeddings with learnable CLS token and positional embeddings. The sequence is processed by $L$ independently parameterized transformer blocks,
\begin{equation}
x^{\ell} = F_{\ell}(x^{\ell-1}), \qquad \ell = 1,\dots,L,
\label{eq:transformer_stack_of_blocks}
\end{equation}
where each $F_{\ell}$ consists of pre-norm multi-head self-attention (MHSA) and a feed-forward network (FFN) with residual connections. We use RMS normalization~\citep{zhang2019root} and GELU activations. The final CLS representation is passed to a linear classifier.

Our single-block vision transformer, termed bViT and illustrated in Fig.~\ref{fig:arch_1}, replaces the independently parameterized blocks $F_1,\dots,F_L$ with repeated application of a single shared block $F$:
\begin{equation}
x^{t} = F(x^{t-1}), \qquad t = 1,\dots,T.
\label{eq:iteraton}
\end{equation}
The shared block has the same architecture as a standard ViT block. Therefore, bViT can be viewed as an unrolled $T$-block ViT with complete parameter sharing across depth.

This formulation separates computational depth from parameter depth as increasing $T$ increases the number of sequential block evaluations without introducing additional block parameters. Consequently, any step-dependent computation must arise from applying the same transformation to an evolving hidden state rather than from depth-specific parameters. This makes bViT both an extreme case of depth sharing and a controlled setting for studying how fixed model components are reused across depth.

\subsection{Model training}

We investigate image classification on  ImageNet-1K, which contains approximately 1.28 million training images and 50,000 validation images from 1000 classes~\citep{imagenet}. We train ViT and bViT under the same training recipe to enable a controlled comparison between the standard and recurrent architectures. The recipe follows the established PyTorch ViT training implementation, with minor modifications for the recurrent setting~\citep{NEURIPS2019_9015}. We omit stochastic depth and aggressive dropout to avoid introducing layer or step-dependent stochasticity that could complicate the comparison. We investigate bViT-S, bViT-B, and bViT-L, which use the embedding dimensions and attention configurations of the corresponding ViT-S, ViT-B, and ViT-L models. For each bViT variant, the number of recurrent steps is set equal to the number of transformer blocks in the corresponding standard ViT. The two transformers therefore use the same number of block evaluations while differing in whether the block parameters are independent or shared across depth.

Models are trained on ImageNet-1K for 300 epochs using images of size $224 \times 224$ and patch size 16. Optimization is performed with AdamW using a base learning rate of $5 \times 10^{-4}$ and weight decay of 0.05~\citep{loshchilov2017decoupled}. The learning rate is linearly warmed up during the first 30 epochs from 0.033 of the base learning rate and subsequently annealed to zero using a cosine schedule. We optimize cross-entropy loss with label smoothing of 0.1 and employ Mixup with $\alpha=0.8$ and CutMix with $\alpha=1.0$. Training is performed on four NVIDIA H100 GPUs with a per-GPU batch size of 256. We maintain an exponential moving average of the model parameters with decay 0.99995. Data augmentation consists of RandomResizedCrop to $224 \times 224$ with scale $(0.08,1.0)$ and aspect ratio $(3/4,4/3)$, followed by random horizontal flipping, RandAugment, and Random Erasing. Images are normalized using the standard ImageNet mean and standard deviation. Models are implemented and trained in PyTorch~\citep{NEURIPS2019_9015}. Our code is available at [***].

\section{Experiments} 

We first evaluate how closely single-block recurrent models match standard ViTs on ImageNet-1K. We then use the shared architecture to analyze how classification behavior, attention, neuron selectivity, and parameter utilization evolve across recurrent depth. Finally, we assess whether the learned representations transfer to downstream vision tasks.


\subsection{Classification performance}
\label{sec:imagenet_results}

Table~\ref{tab:vit_imagenet} compares standard ViTs and bViTs on ImageNet-1K under matched training recipes and computational budgets. This comparison isolates the effect of replacing independently parameterized depth with recurrent reuse of a single transformer block. At sufficient model width, surprisingly little accuracy is lost despite removing almost all layer-specific block parameters. bViT-B achieves 0.779 validation accuracy compared with 0.789 for ViT-B, while using 8.6M rather than 86.6M parameters. The correspondence is even closer for the large model, with bViT-L reaching 0.805 compared with 0.808 for ViT-L, while reducing the parameter count from 304.3M to 14.6M. These results suggest that, at sufficient width, much of the performance benefit of transformer depth does not require distinct parameters at each depth and can instead emerge from repeated application of the same transformation.

This correspondence, however, depends strongly on representation width. At an embedding dimension of 384, bViT-S reaches only 0.681 accuracy compared with 0.782 for ViT-S. Controlled ablations in Appendix~\ref{app:width} show that this gap is not resolved by increasing recurrent depth during training or the number of attention heads, whereas increasing the embedding dimension substantially improves recurrent performance. The appendix further examines the trade-off between recurrent and independently parameterized depth under fixed-compute and parameter-matched settings. Additional bViT variants using explicit step embeddings, register tokens, and knowledge distillation are presented in Appendix~\ref{app:bvit_variants}.

\begin{table}[!b]
    \centering
    \caption{Comparison between the regular ViTs and the corresponding bViTs on ImageNet-1K.}
    \resizebox{0.8\textwidth}{!}{%
    \begin{tabular}{lccccccc}
        \hline
        Model & Emb dim & Heads & Blocks & Steps & Params & FLOPs & Val acc \\ 
        \hline

        ViT-S   & 384  & 6  & 12 & 1 & 22.0M    & 8.5G    & 0.782{\scriptsize$\pm$0.003} \\
        ViT-B   & 768  & 12 & 12 & 1 & 86.6M  & 33.7G   & 0.789{\scriptsize$\pm$0.003} \\
        ViT-L   & 1024 & 16 & 24 & 1 & 304.3M & 119.3G  & 0.808{\scriptsize$\pm$0.004}   \\ 
        \hline

        bViT-S  & 384  & 6  & 1  & 12 & 2.5M   & 8.5G    & 0.681{\scriptsize$\pm$0.003} \\
        bViT-B  & 768  & 12 & 1  & 12 & 8.6M   & 33.7G   & 0.779{\scriptsize$\pm$0.002} \\
        bViT-L  & 1024 & 16 & 1  & 24 & 14.6M  & 119.3G  & 0.805{\scriptsize$\pm$0.003}   \\
        
        \hline
    \end{tabular}
    }
    \label{tab:vit_imagenet}
\end{table}

\subsection{Mechanistic analysis of recurrent computation}
\label{sec:Interpretability}

The single-block formulation provides a controlled setting for studying how transformer computation changes with depth. Because the same components are reused at every recurrent step, we can track their behavior directly as the hidden state evolves. We first examine how classification and spatial attention develop across recurrent depth, then study the behavior of fixed FFN neurons, and finally test whether the shared parameters are utilized differently at different recurrent stages. Additional analyses of representation evolution using energy rank and CKA, together with the spectral rank of the shared weight matrices, are provided in Appendix~\ref{app:rank}.

\subsubsection{Temporal organization of classification and attention}
\label{sec:attention-patterns}

\begin{figure}[!t]
    \centering
    \resizebox{\textwidth}{!}{%
    \begin{minipage}{\textwidth}
        \centering

        \begin{subfigure}[t]{0.31\textwidth}
            \centering
            \includegraphics[width=\textwidth]{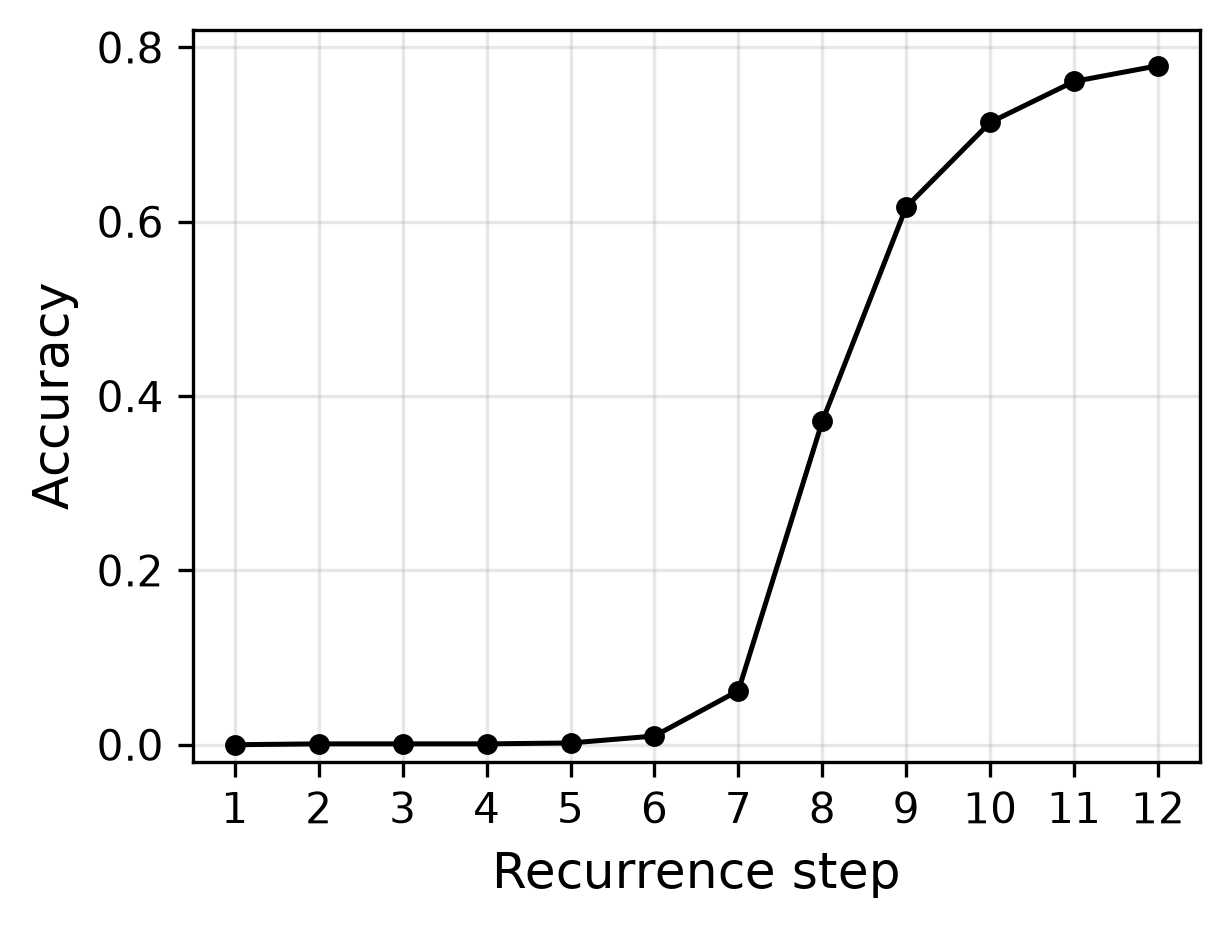}
            \caption{Validation accuracy.}
            \label{fig:attention_mean_heads}
        \end{subfigure}
        \hfill
        \begin{subfigure}[t]{0.31\textwidth}
            \centering
            \includegraphics[width=\textwidth]{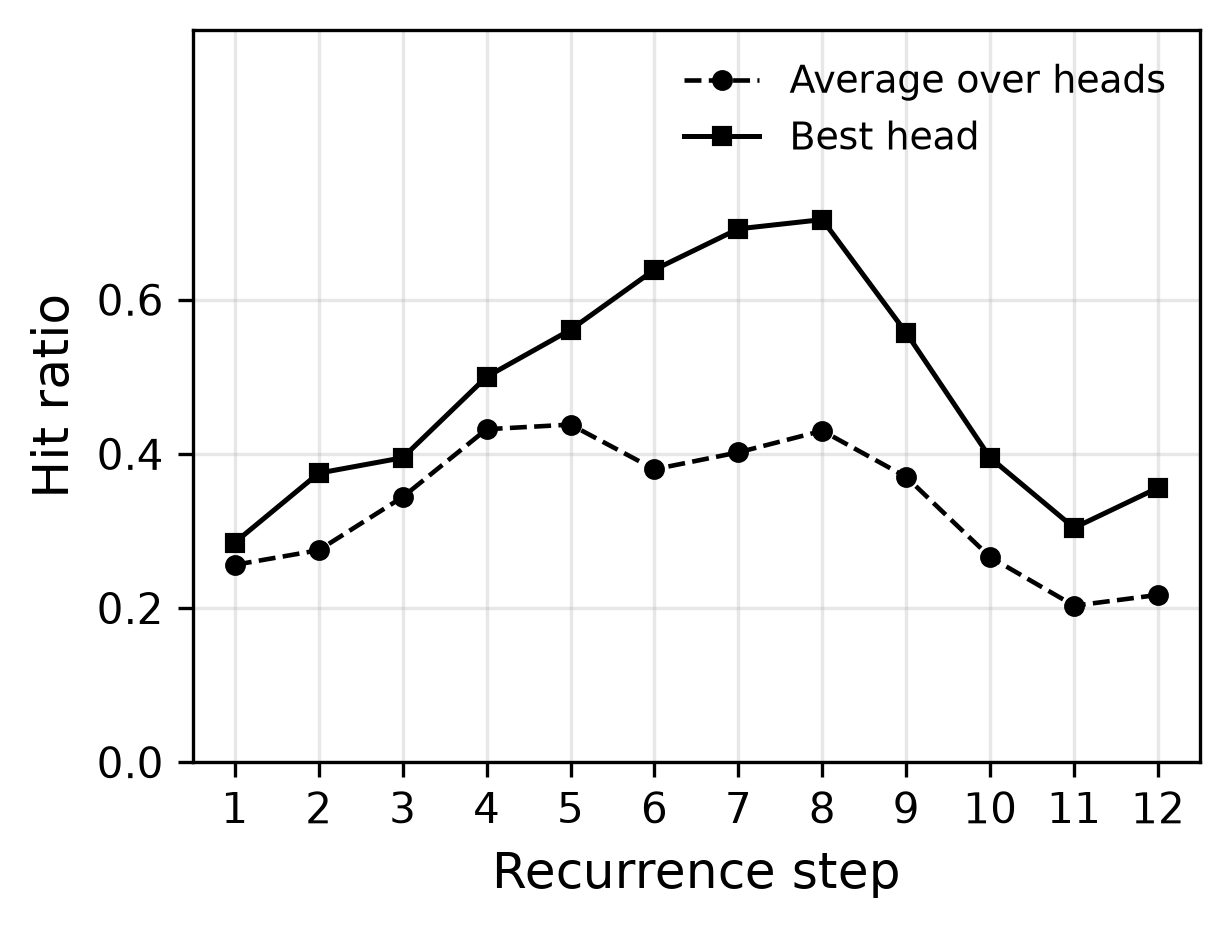}
            \caption{Pointing-game results}
            \label{fig:attention_best_heads}
        \end{subfigure}
        \hfill
\begin{subfigure}[t]{0.31\textwidth}
    \centering
    \includegraphics[height=0.72\linewidth]{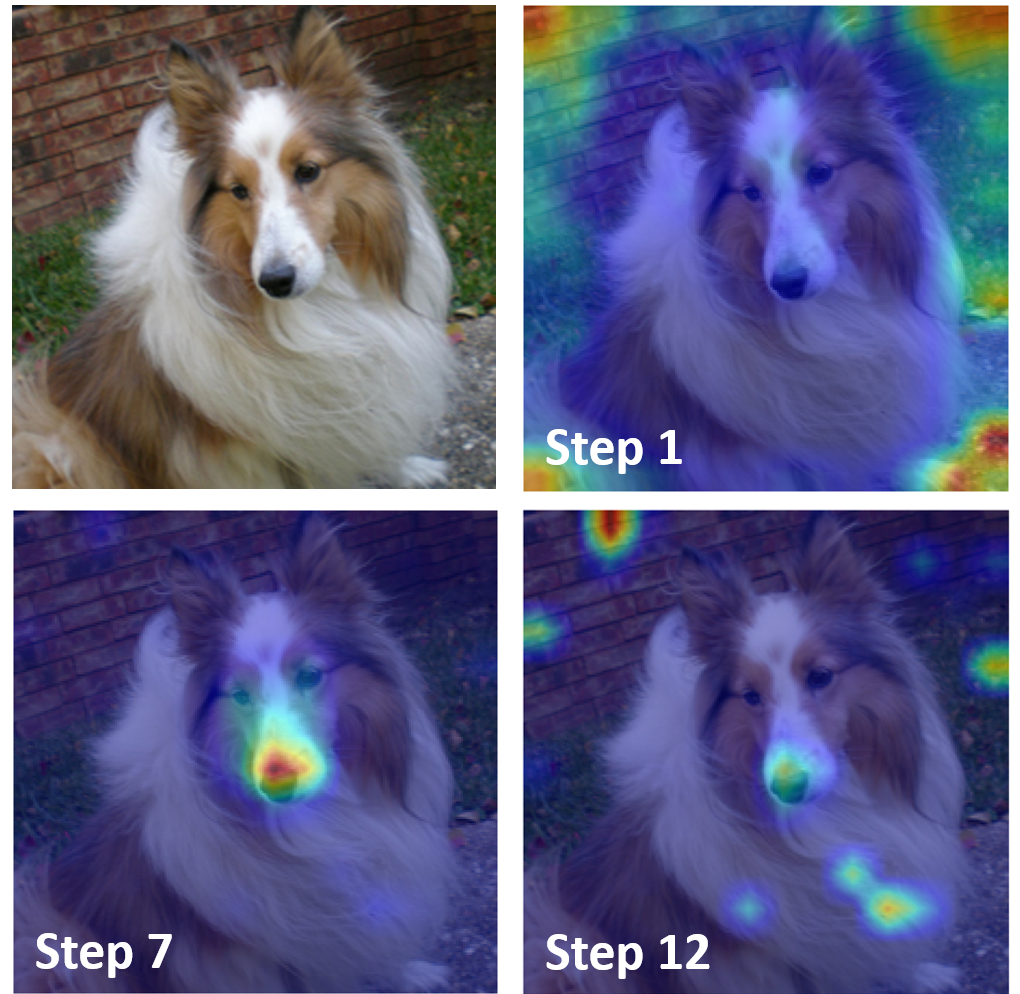}
    \caption{Example attention maps.}
    \label{fig:attention_steps_example}
\end{subfigure}

    \end{minipage}%
    }
\caption{Evolution of classification and attention across recurrent depth.
(a) ImageNet validation accuracy after each recurrent step.
(b) Pointing-game hit ratio averaged over heads and for the best-performing head. Attention localization peaks earlier than classification performance.
(c) Example attention maps for the best-performing head at steps 1, 7, and 12.}

\label{fig:attention_plots}
\end{figure}

We first examine how classification and spatial attention evolve across recurrent depth in bViT-B. Figure~\ref{fig:attention_plots}(a) reports ImageNet-1K validation accuracy after terminating inference at each recurrent step. Classification performance remains low during early recurrence and increases sharply after approximately six steps. Because bViT reuses the same attention heads at every step, their behavior can be tracked directly across depth. We evaluate object localization on the ImageNet-S validation set, containing segmentation masks for 12,419 images from 919 classes~\citep{gao2022large}, using the pointing-game metric~\citep{zhang2016top}: a hit occurs when the maximum of the CLS attention map falls inside the ground-truth mask.

Figure~\ref{fig:attention_plots}(b) reveals a different temporal profile. Object-centric attention emerges early and is strongest around the middle of the recurrence. The best-performing head reaches peak localization before the final steps and subsequently declines while classification accuracy continues to improve. The average over heads follows a similar, weaker trend.

This separation provides evidence that repeated application of the same block can support distinct stages of computation without depth-specific parameters. The shared attention heads first develop strong object-centric localization, whereas later applications increasingly support class discrimination even as localization weakens. Functional depth can therefore emerge from the interaction between a fixed operator and an evolving hidden state, rather than from assigning different parameters to different layers. The examples in Fig.~\ref{fig:attention_plots}(c) illustrate the same qualitative progression. Per-head results and additional attention maps are shown in Appendix~\ref{app:attention_heads}.

\subsubsection{Step-dependent activation patterns}

\begin{wrapfigure}{r}{7cm}
    \vspace{-5mm}
    \centering

    \includegraphics[width=6cm]{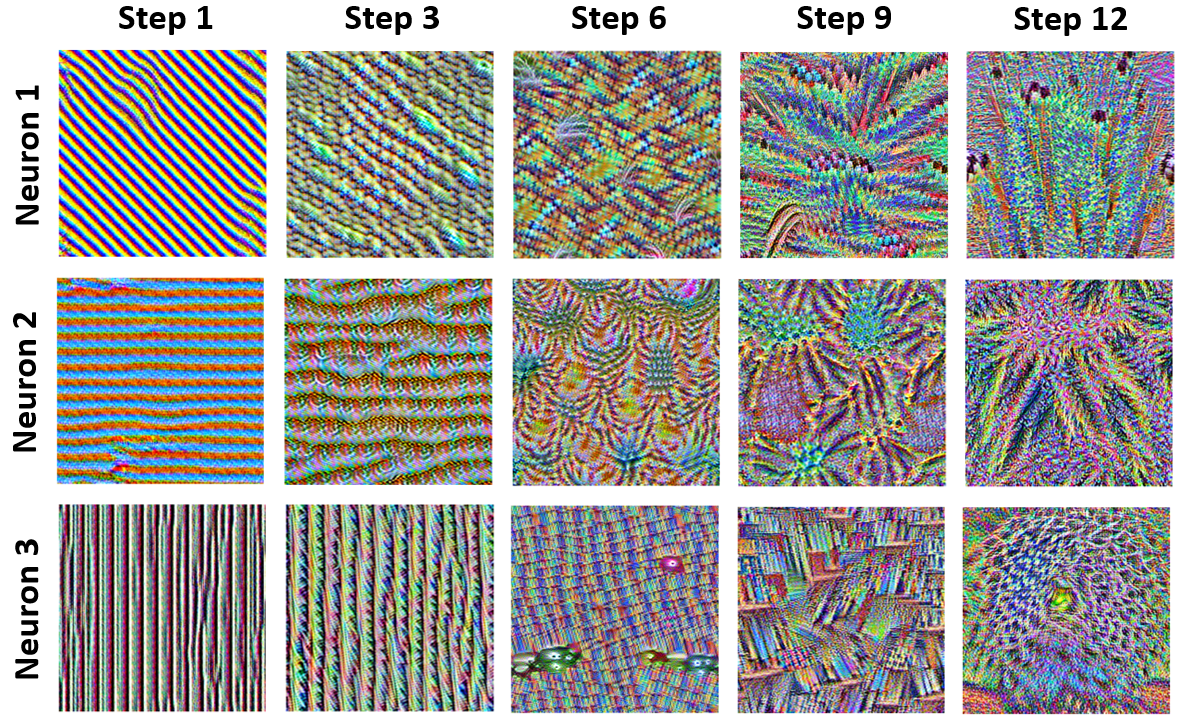}

    \caption{Optimization based visualizations of the same FFN neurons across bViT-B steps. Preferred activation patterns evolve from simpler texture-like structures toward more structured patterns.}
    \label{fig:activation_patterns}

    \vspace{-6mm}
\end{wrapfigure}

Neuron activation patterns in standard ViTs exhibit a depth-dependent progression, with earlier blocks responding to relatively simple patterns and later blocks responding to more structured visual features~\citep{ghiasi2022vision}. The single-block formulation allows us to examine an analogous progression while tracking exactly the same neurons across recurrent depth.

We unroll bViT-B over 12 recurrent steps and synthesize images that maximize the activation of selected neurons in the FFN output layer. Figure~\ref{fig:activation_patterns} follows three neurons across the recurrent computation. Their preferred patterns change with recurrent depth, with earlier steps producing predominantly simple periodic and texture-like structures and later steps producing more structured patterns. This provides a qualitative illustration that the effective selectivity of a fixed FFN neuron can change as the recurrent state evolves.

\subsubsection{Step-conditioned spectral utilization}
\label{sec:spectral_importance}

Exact parameter alignment enables direct comparison of the same spectral directions across recurrent stages. Rather than analyzing individual scalar weights, we use the singular-vector basis, which decomposes each matrix $W$ into orthogonal rank-one transformation directions. These directions remain fixed across recurrent steps and can therefore be tracked and selectively removed. For each linear transformation $W\in\mathbb{R}^{m\times n}$, we compute its singular value decomposition as:

\begin{equation}
    W
    =
    U\Sigma V^\top
    =
    \sum_{i=1}^{R}
    \sigma_i u_i v_i^\top ,
    \label{eq:shared_svd}
\end{equation}

where $R$ denotes the rank of $W$. To estimate the importance of spectral mode $i$ at recurrent step $t$, we associate it with a multiplicative gate $a_{t,i}$ as follows:

\begin{equation}
    W_t(a_t)
    =
    \sum_{i=1}^{R}
    a_{t,i}\sigma_i u_i v_i^\top ,
    \label{eq:spectral_gate}
\end{equation}

where $a_{t,i}=1$ recovers the original transformation. We unroll the model only for gradient computation, so that each use of $W$ has a separate gradient while retaining identical parameter values. Let $ G_t = \frac{\partial \mathcal{L}} {\partial W_t} $ denote the gradient of the classification loss with respect to the matrix used at step $t$. The sensitivity of the loss to the spectral gate can be expressed as: 

\begin{equation}
    \frac{\partial \mathcal{L}}
         {\partial a_{t,i}}
    =
    \sigma_i u_i^\top G_t v_i .
    \label{eq:spectral_gate_gradient}
\end{equation}

We define the step-conditioned importance of mode $i$ as the mean absolute sensitivity over the ImageNet-1K training data:

\begin{equation}
    I_{t,i}
    =
    \mathbb{E}
    \left[
        \left|
        \sigma_i u_i^\top G_t v_i
        \right|
    \right].
    \label{eq:spectral_importance}
\end{equation}

Our approach is inspired by the first-order sensitivity principle used in gradient-based pruning~\citep{molchanov2019importance}, where the importance of a component is estimated from the effect of perturbing or suppressing it on the loss. Here, we apply this principle to the singular modes of a shared weight matrix and condition the sensitivity on the recurrent step. This allows us to measure how the task relevance of the same parameter direction changes over recurrent depth. Here, we focus on the query projection Q, value projection V, and first feed-forward projection FFN\_IN. Results for the remaining projections are provided in Appendix~\ref{app:step_spectral}.

\paragraph{Token-dependent spectral organization.}

We first ask whether the relative importance of the shared spectral directions remains stable across recurrent depth. We separate the gradient contribution of the CLS token from the combined contribution of the patch tokens before projection into the shared spectral basis $ G_t
    =
    G^{\mathrm{CLS}}_t
    +
    G^{\mathrm{patch}}_t .$ This gives separate spectral importance profiles for CLS and patch processing. Figure~\ref{fig:spearman} shows the Spearman correlation between importance profiles at different recurrent steps. Patch-token profiles remain comparatively stable across recurrence, particularly for Q. In contrast, CLS profiles exhibit substantially stronger step-dependent variation. The effect is most pronounced for V and FFN\_IN, with FFN\_IN showing a marked reorganization around steps $8$--$9$. Figure~\ref{fig:clsvspatch} compares CLS and patch importance profiles directly at each step. Their similarity generally increases toward the end of the recurrence. The strongest separation occurs for FFN\_IN, where CLS and patch profiles remain weakly aligned during the early and middle stages and become strongly aligned only later.

These results indicate that the shared spectral basis is not used uniformly across recurrent depth. Patch processing is comparatively stable, while CLS processing undergoes stronger step-dependent changes. This is consistent with the earlier observation that object localization develops before classification performance improves.

\begin{figure*}[]
    \centering


    \begin{minipage}{0.32\textwidth}
        \centering
        \includegraphics[width=\textwidth]{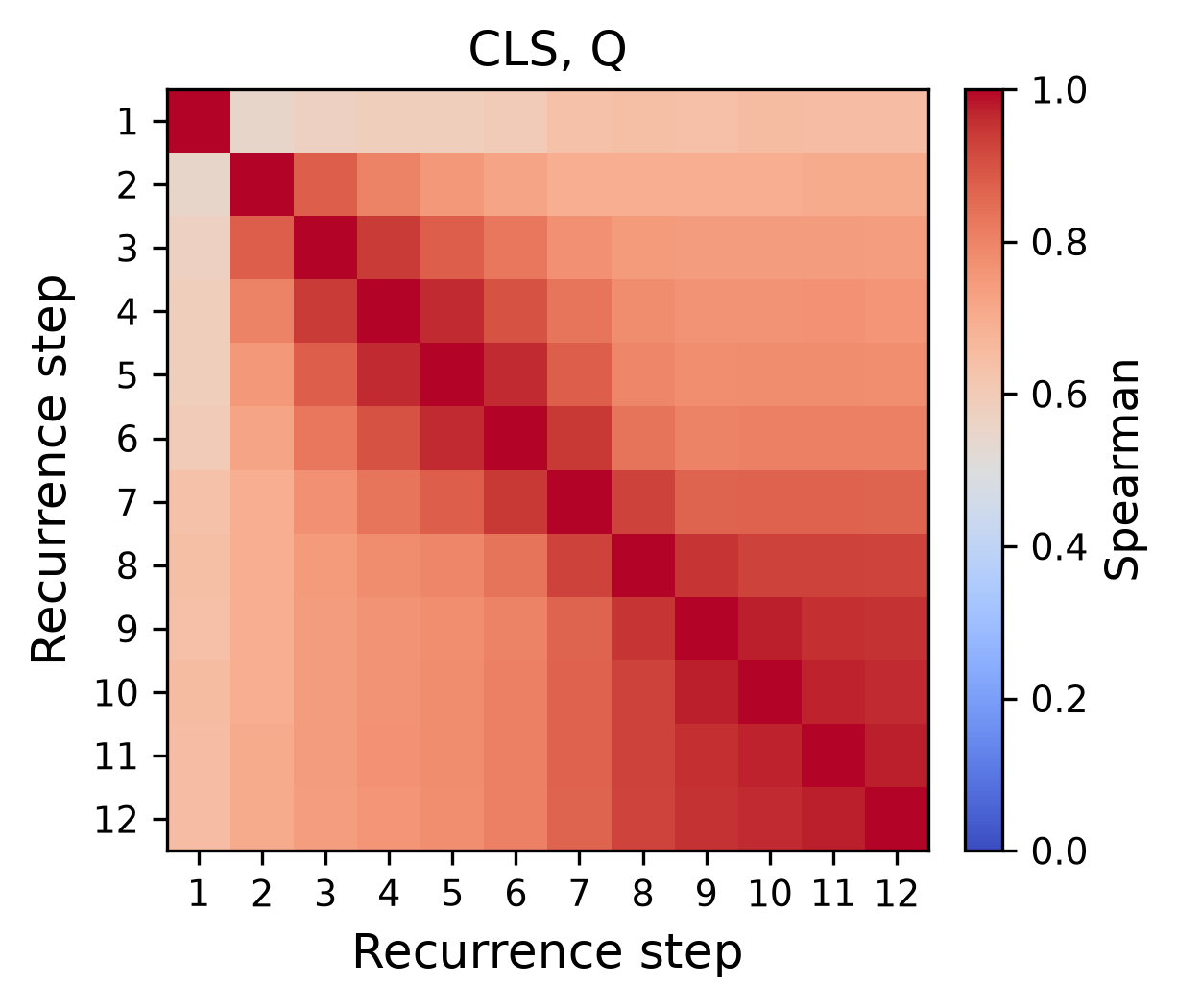}
    \end{minipage}
    \hfill
    \begin{minipage}{0.32\textwidth}
        \centering
        \includegraphics[width=\textwidth]{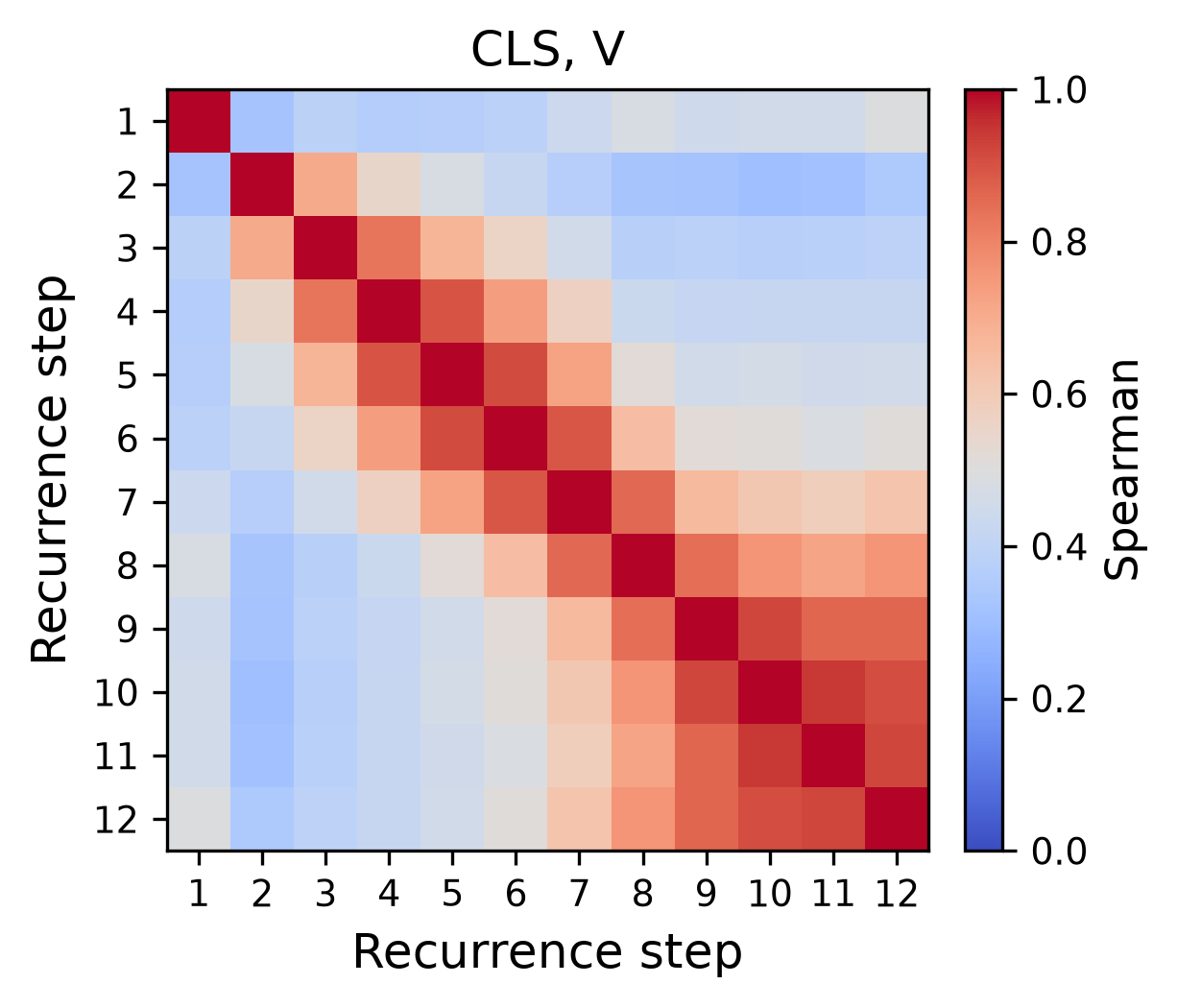}
    \end{minipage}
    \hfill
    \begin{minipage}{0.32\textwidth}
        \centering
        \includegraphics[width=\textwidth]{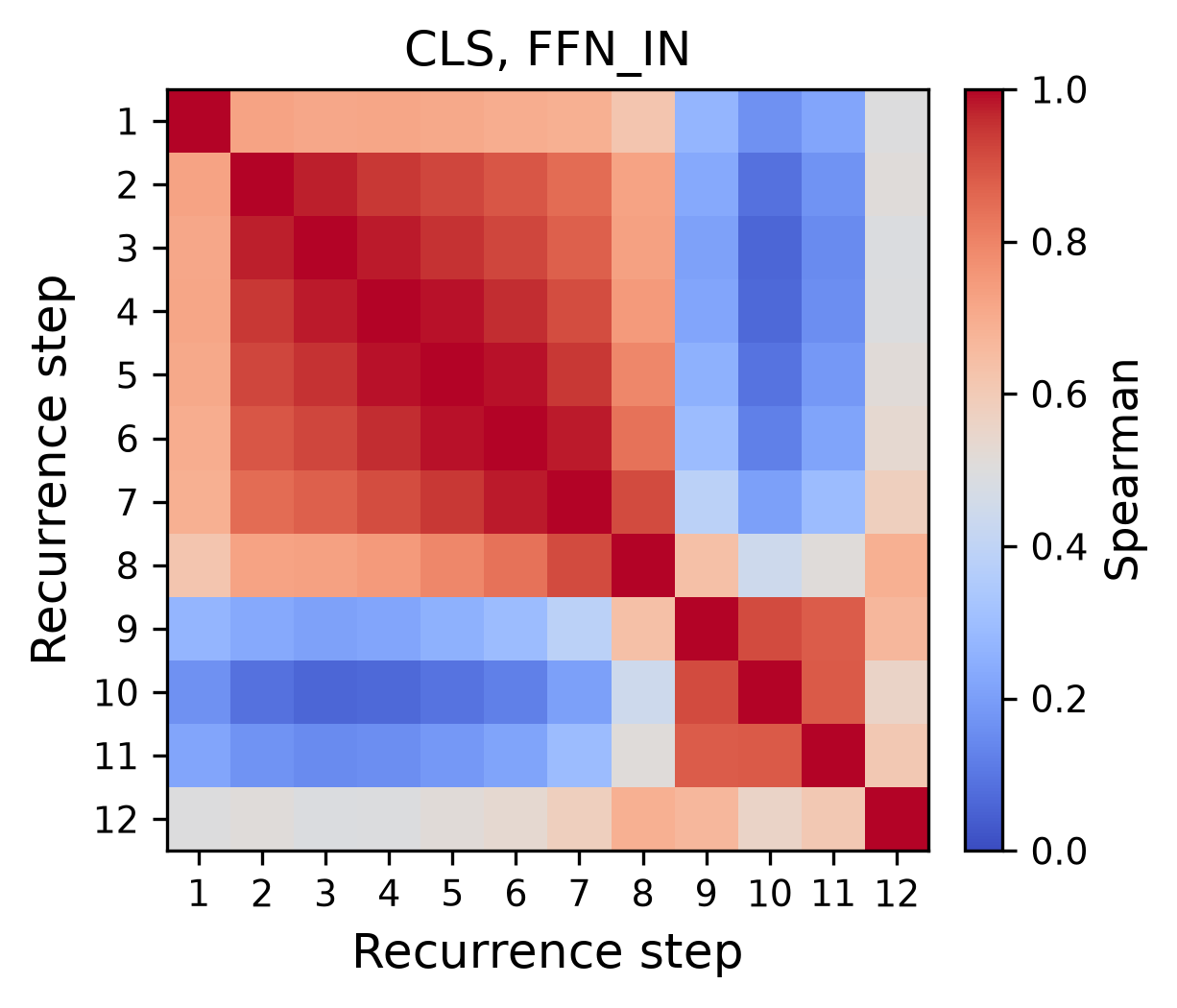}
    \end{minipage}

    \vspace{0.25cm}


    \begin{minipage}{0.32\textwidth}
        \centering
        \includegraphics[width=\textwidth]{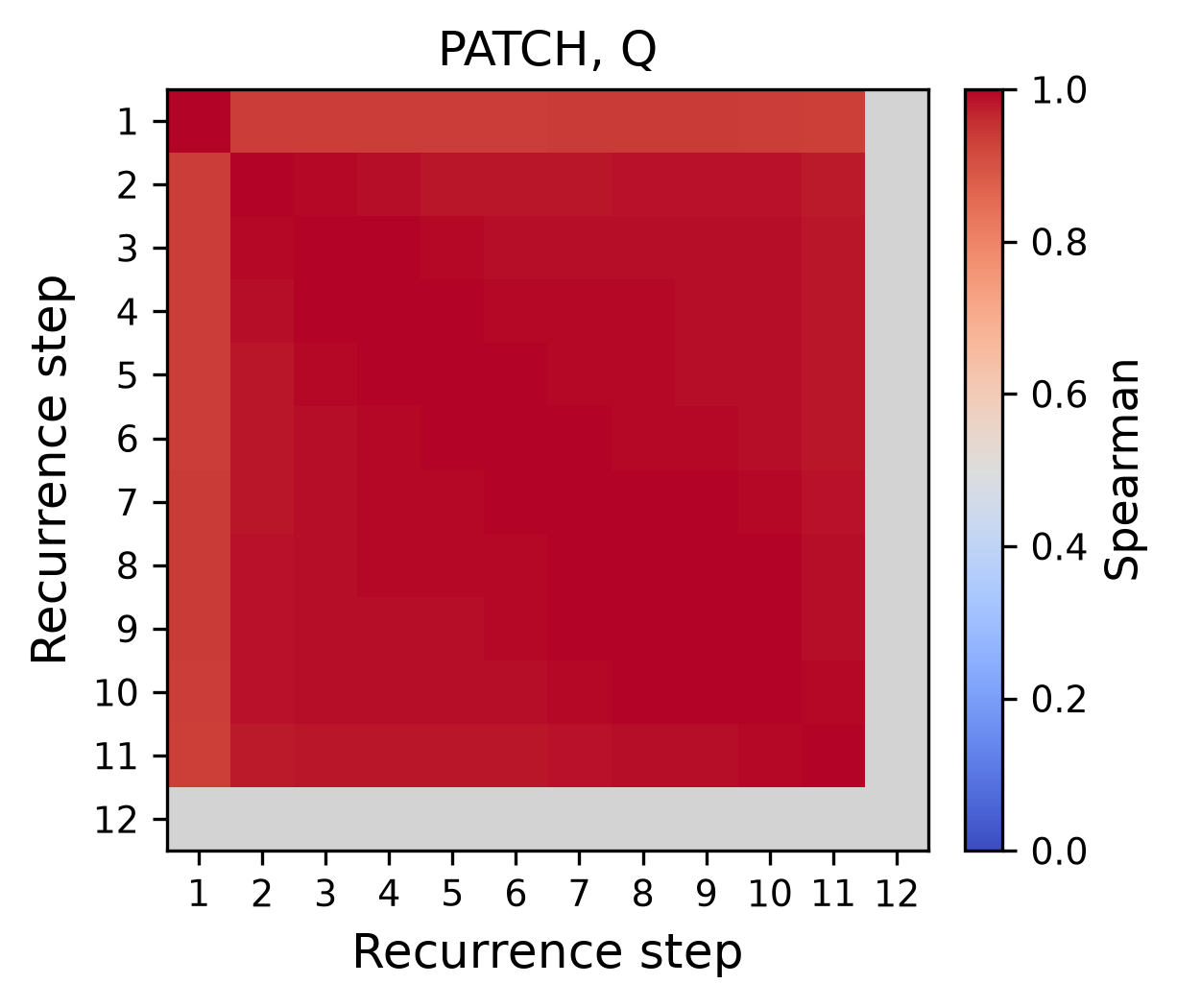}
    \end{minipage}
    \hfill
    \begin{minipage}{0.32\textwidth}
        \centering
        \includegraphics[width=\textwidth]{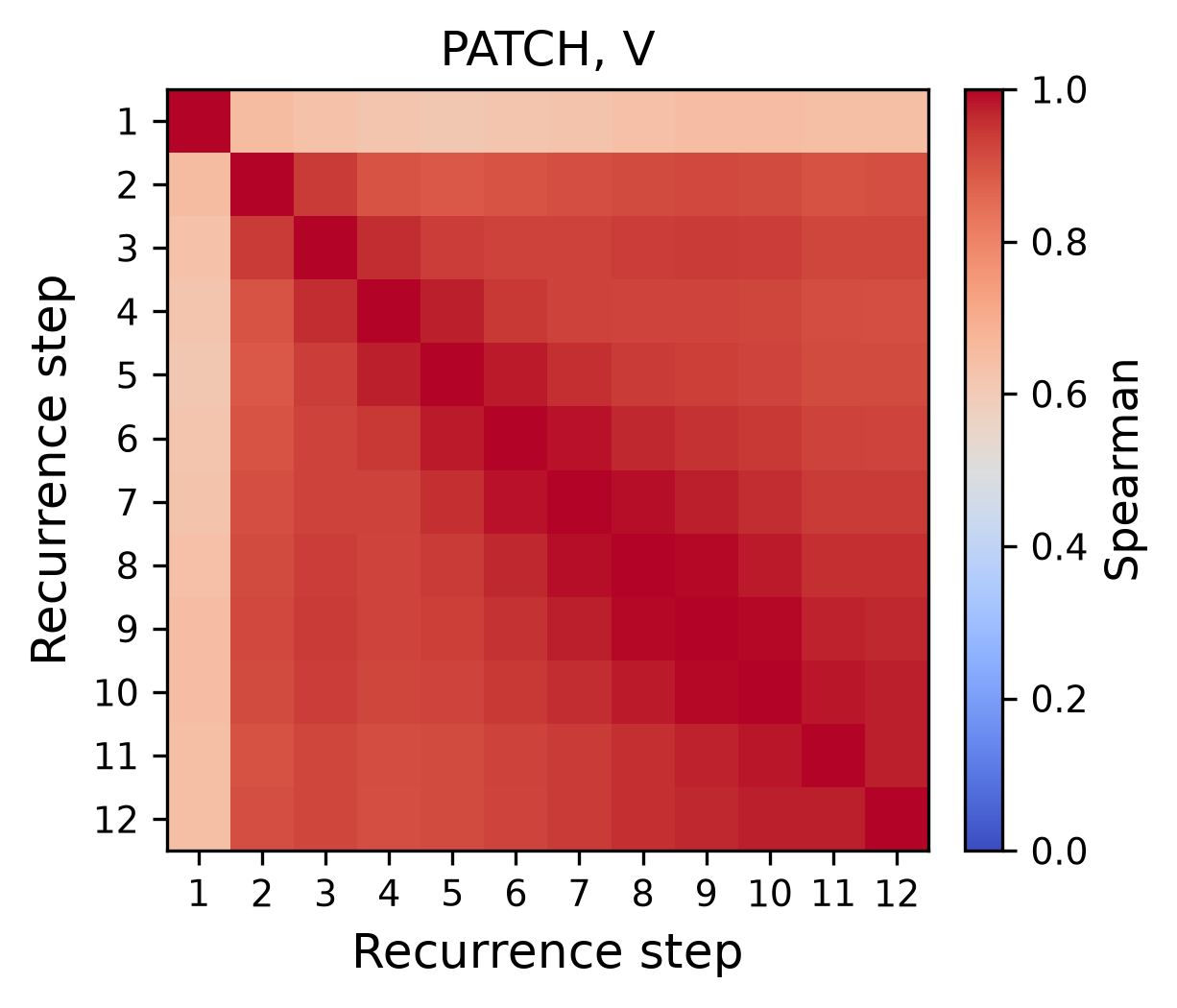}
    \end{minipage}
    \hfill
    \begin{minipage}{0.32\textwidth}
        \centering
        \includegraphics[width=\textwidth]{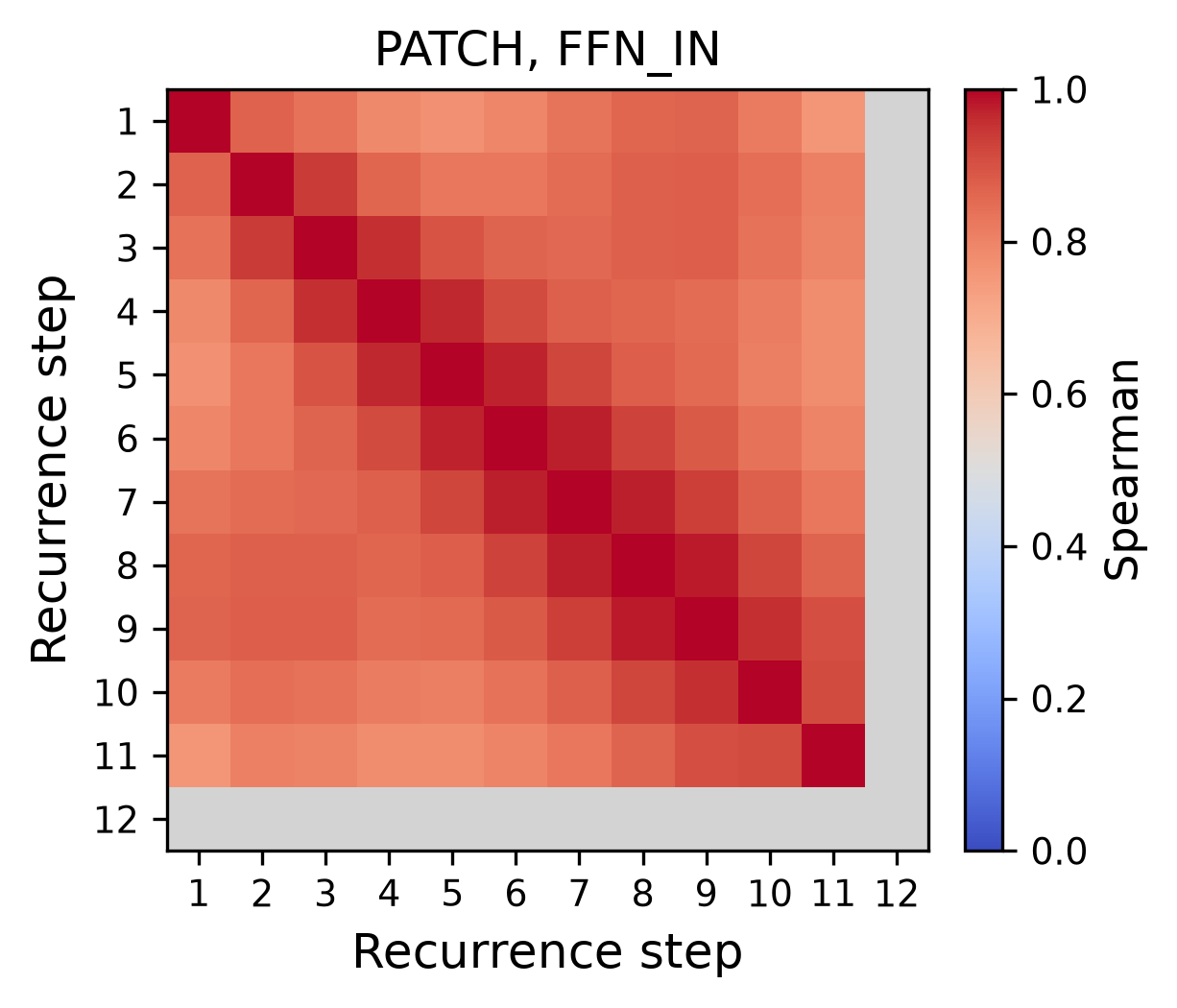}
    \end{minipage}

\caption{
Similarity of step-conditioned spectral importance profiles across recurrent depth for Q, V, and FFN\_IN. Each matrix reports the Spearman correlation between importance profiles at different recurrent steps. CLS profiles show stronger step-dependent variation, particularly for V and FFN\_IN, while patch-token profiles remain comparatively stable. Gray entries indicate undefined correlations (no gradient from the loss).
}
    \label{fig:spearman}
\end{figure*}

\paragraph{Step-conditioned spectral pruning.}

We next perform an intervention to test whether particular spectral directions are functionally required at specific recurrent stages, rather than  showing step-dependent gradient sensitivity. At each step, we retain the $k$ spectral modes with the highest importance scores according to $ 
S_t^{(k)} = \operatorname{TopK}(I_t,k) $ and reconstruct the corresponding transformation:

\begin{equation}
    \widetilde W_t
    =
    \sum_{i\in S_t^{(k)}}
    \sigma_i u_i v_i^\top .
    \label{eq:step_spectral_pruning}
\end{equation}

No retraining is performed after pruning. We compare step-conditioned pruning with global gradient pruning, which applies the same gradient based mask at every step, a largest-$\sigma$ baseline, and a permuted variant in which the step-conditioned masks are reassigned across recurrent steps. 

The top row of Fig.~\ref{fig:pruning_svd} shows that step-conditioned pruning is more robust once the spectral budget becomes restrictive, particularly for V and FFN\_IN. For instance, after pruning $60\%$ of the V directions, step-conditioned pruning retains approximately $0.54$ accuracy, compared with approximately $0.23$ for the global mask. The difference is smaller for Q, which remains robust until stronger pruning. The permuted masks also perform worse than the correctly aligned step-conditioned masks. Since the selected directions are unchanged and only their assignment to recurrent steps is modified, this result shows that their temporal alignment matters for model performance.

\begin{figure*}[]
    \centering

    \begin{minipage}{0.32\textwidth}
        \centering
        \includegraphics[width=\textwidth]{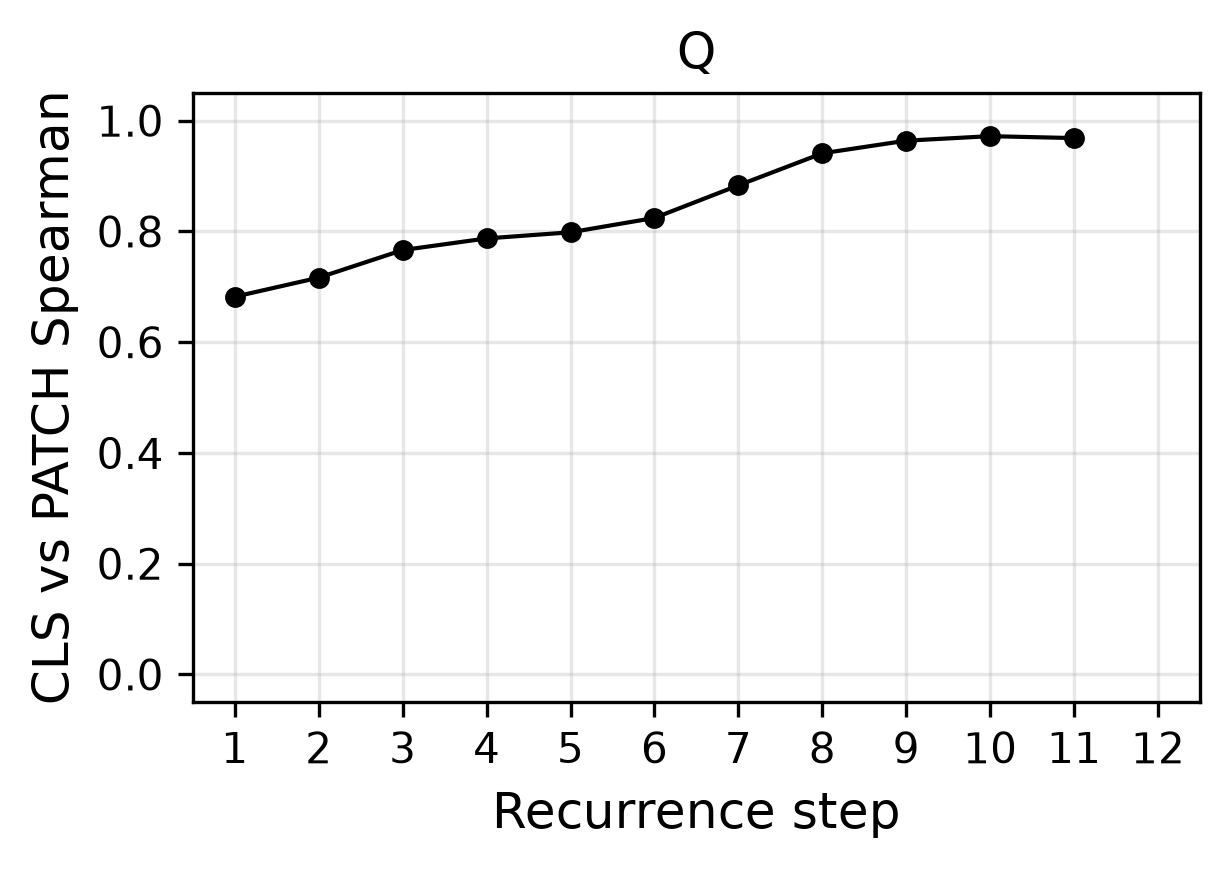}
    \end{minipage}
    \hfill
    \begin{minipage}{0.32\textwidth}
        \centering
        \includegraphics[width=\textwidth]{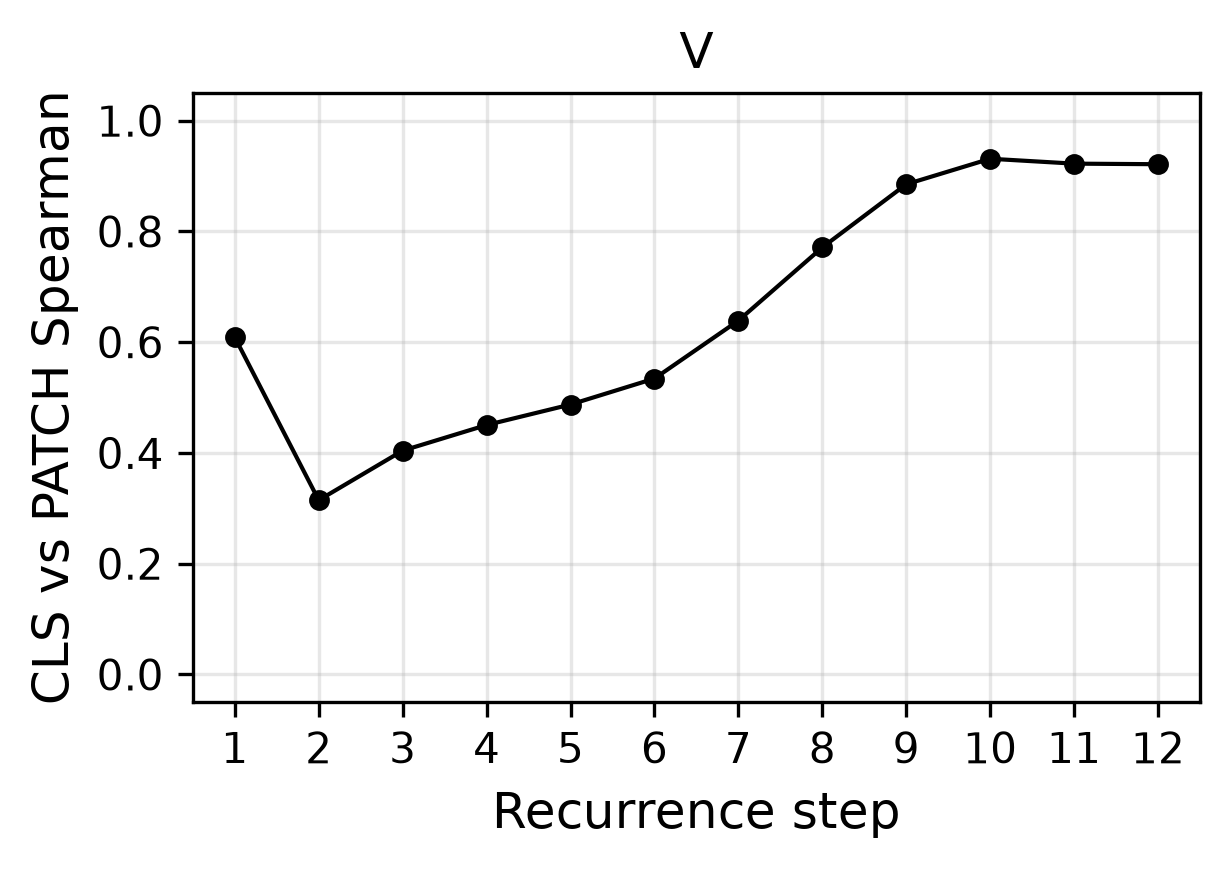}
    \end{minipage}
    \hfill
    \begin{minipage}{0.32\textwidth}
        \centering
        \includegraphics[width=\textwidth]{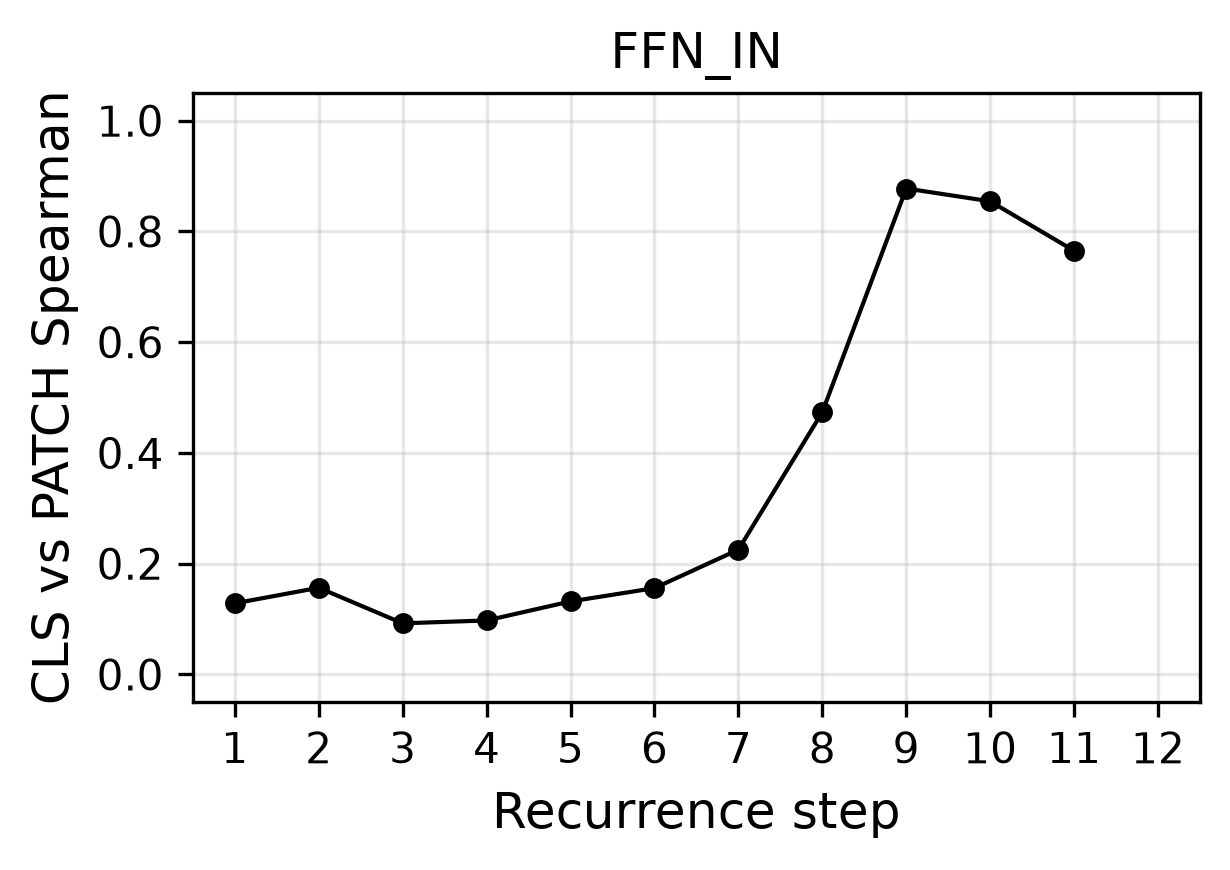}
    \end{minipage}

\caption{
Similarity between CLS and patch spectral importance profiles across recurrent depth for Q, V, and FFN\_IN. Each curve reports the Spearman correlation between the CLS and patch profiles at a given recurrent step. Their importance profiles become increasingly aligned during later recurrence, with the strongest early-to-late changes observed for V and FFN\_IN. }
    
    \label{fig:clsvspatch}
\end{figure*}

\paragraph{Persistent and stage-specific directions.}

The lower row of Fig.~\ref{fig:pruning_svd} examines how often the directions selected by step-conditioned pruning are reused across recurrent steps. We divide the selected directions into three groups according to whether they are retained in $1$--$4$, $5$--$8$, or $9$--$12$ recurrent steps. Across projections, a substantial fraction of the selected directions remains active throughout most of the recurrent computation. Compared with the random-selection baseline in Appendix~\ref{app:step_spectral}, directions retained in $9$--$12$ steps are substantially overrepresented, while directions with intermediate persistence are comparatively rare. The shared block therefore appears to contain a persistent spectral backbone whose contribution changes across recurrent stages. The performance degradation caused by reassigning step-conditioned masks shows that this temporal organization is functionally relevant. These results suggest that recurrent depth can emerge from stage-dependent use of stable shared transformations rather than from depth-specific parameters.

\begin{figure*}[]
    \centering


    \begin{minipage}{0.32\textwidth}
        \centering
        \includegraphics[width=\textwidth]{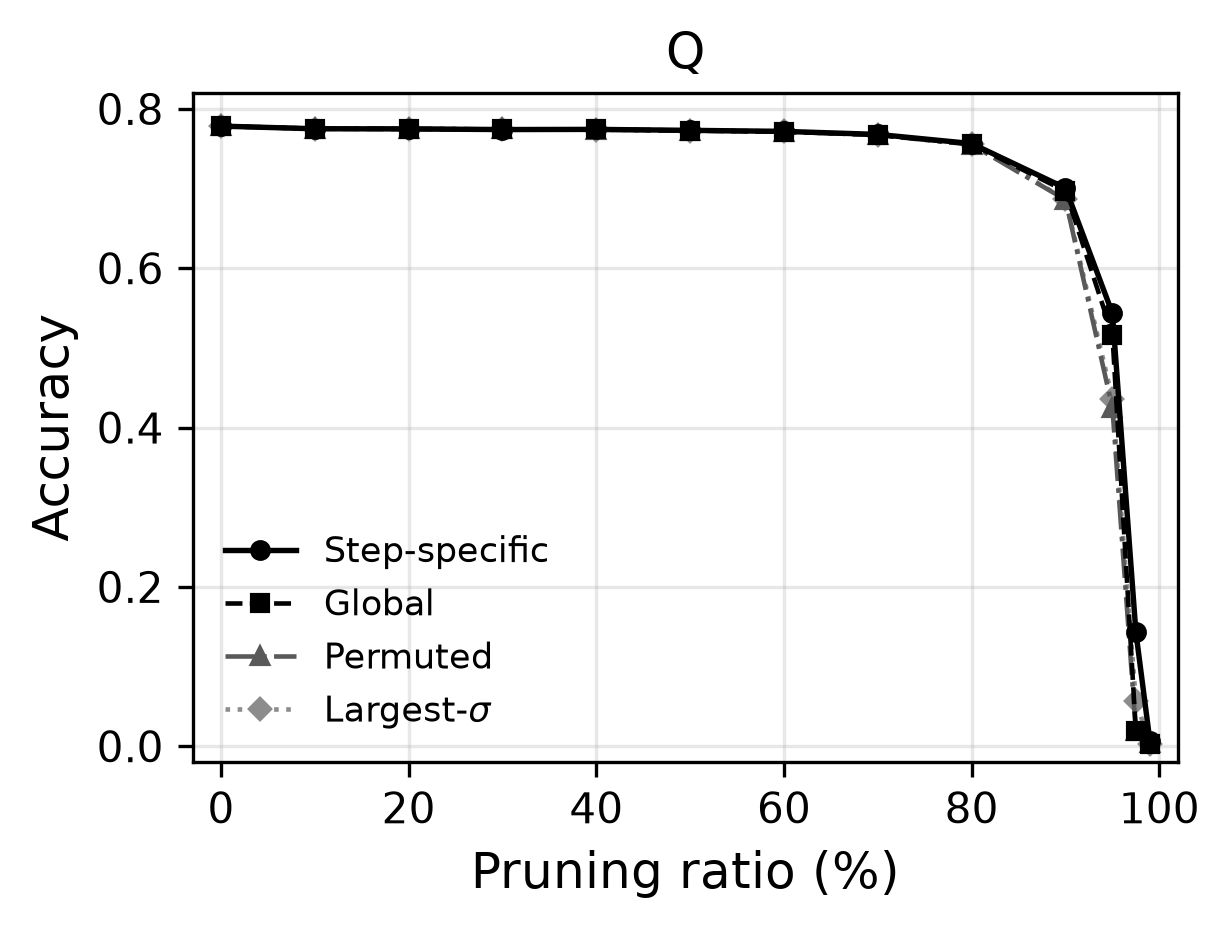}
    \end{minipage}
    \hfill
    \begin{minipage}{0.32\textwidth}
        \centering
        \includegraphics[width=\textwidth]{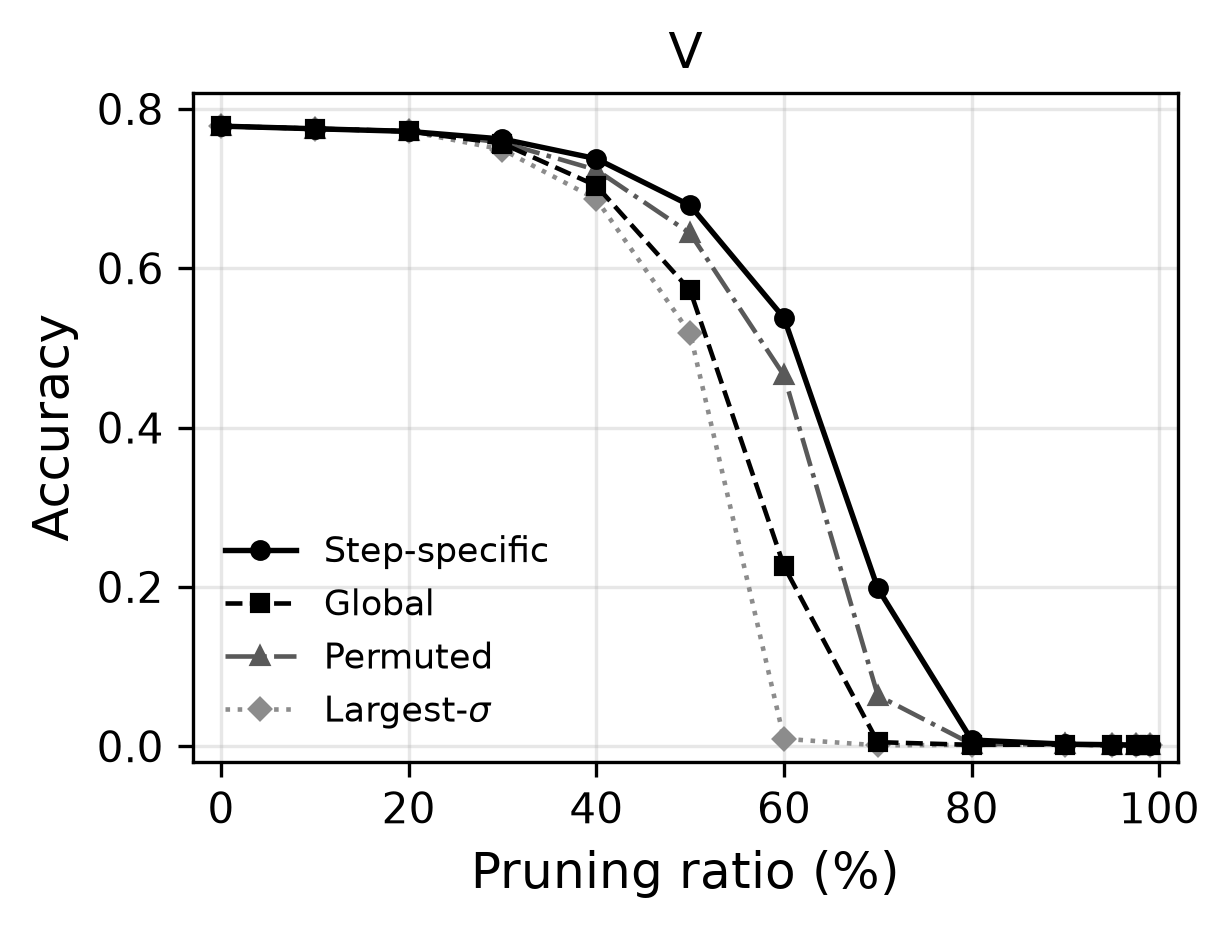}
    \end{minipage}
    \hfill
    \begin{minipage}{0.32\textwidth}
        \centering
        \includegraphics[width=\textwidth]{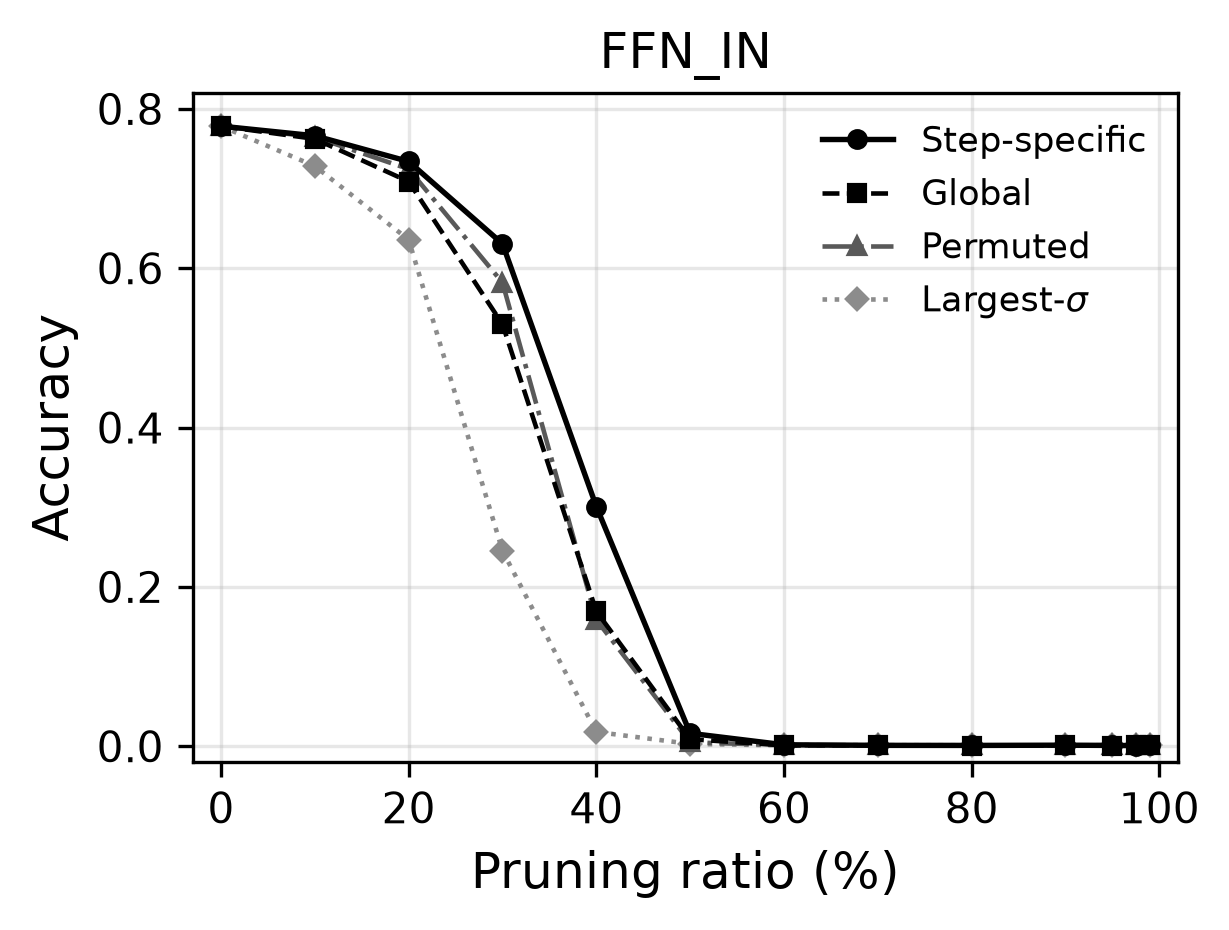}
    \end{minipage}

    \vspace{0.25cm}


    \begin{minipage}{0.32\textwidth}
        \centering
        \includegraphics[width=\textwidth]{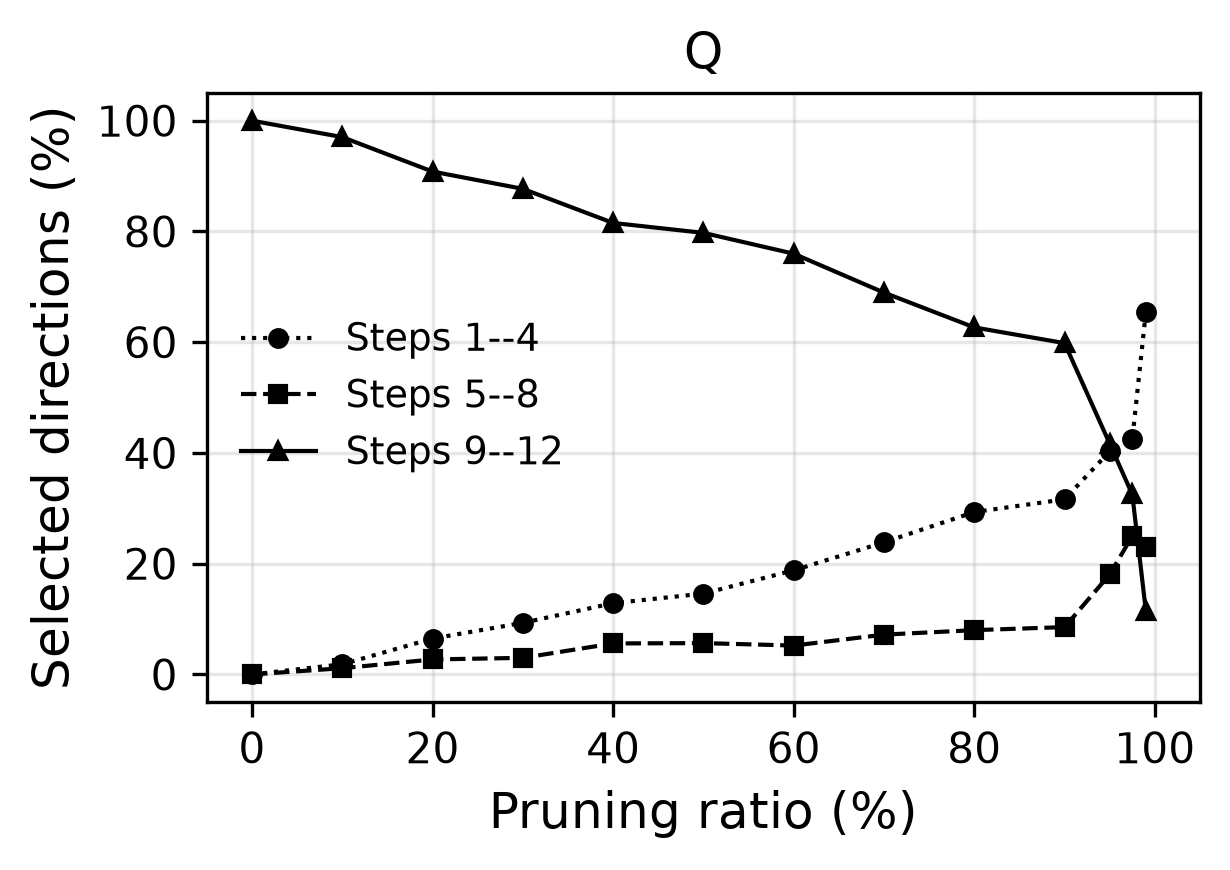}
    \end{minipage}
    \hfill
    \begin{minipage}{0.32\textwidth}
        \centering
        \includegraphics[width=\textwidth]{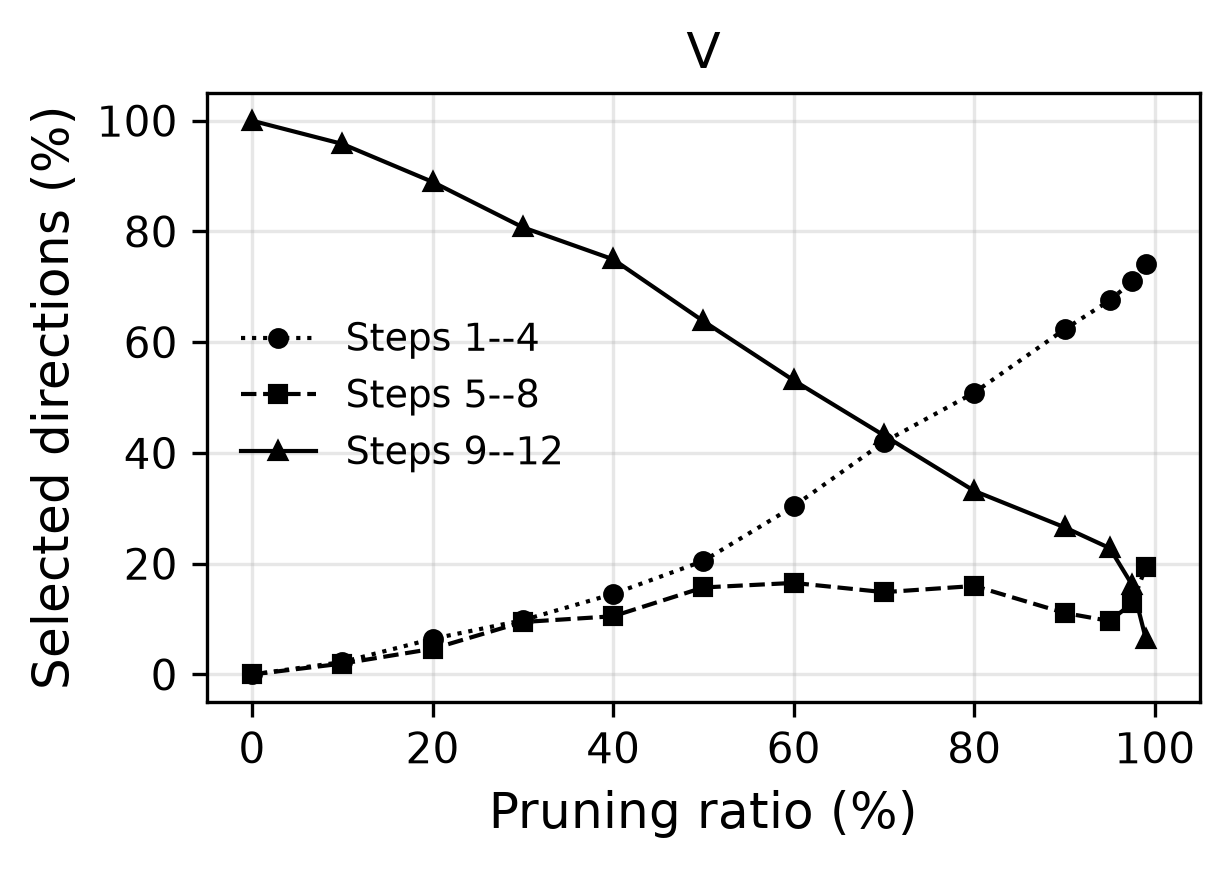}
    \end{minipage}
    \hfill
    \begin{minipage}{0.32\textwidth}
        \centering
        \includegraphics[width=\textwidth]{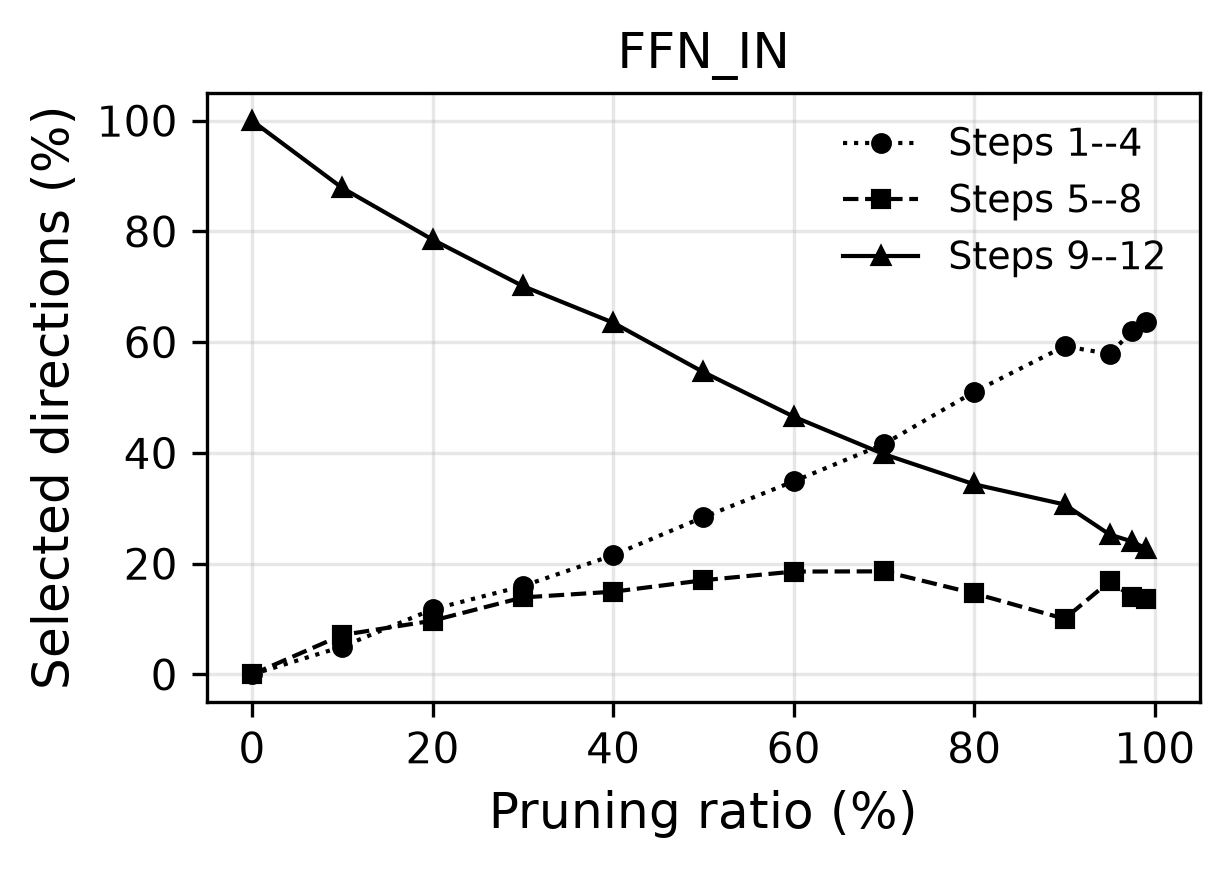}
    \end{minipage}

\caption{
Step-conditioned spectral utilization of the shared transformer block. Top row shows ImageNet-1K validation accuracy after pruning spectral directions in Q, V, and FFN\_IN. Step-conditioned pruning is compared with a global gradient mask, temporally permuted step-conditioned masks, and selection by singular-value magnitude. Step-conditioned masks are more robust under restrictive spectral budgets, particularly for V and FFN\_IN, while mask permutation reduces performance. Bottom row shows the persistence of directions selected by step-conditioned pruning. Selected directions are grouped according to whether they are retained in $1$--$4$, $5$--$8$, or $9$--$12$ recurrent steps. A substantial fraction of selected directions persists across most recurrent steps.
}

    \label{fig:pruning_svd}
\end{figure*}

\subsection{Transfer learning}
\label{sec:transfer}

We examine if representations learned by the single-block model transfer beyond ImageNet-1K classification and whether its shared parameterization changes the cost of downstream adaptation. We evaluate the pretrained bViT-B on eight downstream vision datasets, CIFAR10~\citep{Krizhevsky09learningmultiple}, CIFAR100~\citep{Krizhevsky09learningmultiple}, Oxford-IIIT Pets~\citep{parkhi2012cats}, DTD~\citep{cimpoi14describing}, Oxford Flowers102~\citep{nilsback2008automated}, RESISC45~\citep{cheng2017remote}, FGVC Aircraft~\citep{maji2013fine}, and EuroSAT~\citep{helber2019eurosat}. We consider linear probing, full fine-tuning, and rank-8 LoRA applied either to the query and value projections or to the FFN layers~\citep{hu2022lora}. We additionally evaluate time-embedding tuning. We equip bViT with a learnable embedding for each recurrent step, initialized to zero, while keeping the shared block frozen. This treats bViT as a fixed operator and adapts downstream behavior only through lightweight step-specific inputs that steer its computation across recurrent depth. Additional details are in Appendix~\ref{app:transfer}.

Table~\ref{tab:transfer} shows that bViT-B transfers competitively across the downstream tasks. In linear probing, bViT-B achieves an average accuracy of 0.803 compared with 0.784 for ViT-B, indicating that the recurrent backbone retains transferable representations without adaptation. Full fine-tuning reaches 0.885 for bViT-B and 0.889 for ViT-B. The shared architecture also substantially reduces the number of trainable parameters required for adaptation. Rank-8 LoRA on the query-value projections uses 25K parameters for bViT-B compared with 296K for ViT-B, while reaching average accuracies of 0.867 and 0.871, respectively. For FFN adaptation, bViT-B uses 62K parameters compared with 738K for ViT-B and reaches 0.871 compared with 0.876. Tuning only 9K recurrent time-embedding parameters reaches 0.844, improving substantially over linear probing without modifying the shared block. These results show that the recurrent representation remains competitive across diverse downstream tasks. Although a single low-rank update is reused across the entire  depth, its gap to full fine-tuning is comparable to that of per-layer LoRA in ViT-B, with average accuracy differences of 0.017 versus 0.018 for QV adaptation and 0.014 for FFN adaptation in both models. Our spectral analysis suggests a mechanism by which such shared adaptation can remain expressive. Because the same parameter directions have different functional importance across recurrent stages, a fixed low-rank update can have stage-dependent effects as it acts on the evolving hidden state. 

\begin{table}[]
    \centering
    \caption{Comparison of transfer learning capabilities of ImageNet-1K pretrained ViT-B and bViT-B. bViT-B transfers competitively across downstream tasks, while requiring less trainable parameters.}
    \resizebox{1\textwidth}{!}{ 
    \begin{tabular}{clcccccccccc} 
        \hline
        \textbf{Model} & \textbf{Method} & \textbf{Params} 
        & \textbf{CIFAR10} & \textbf{CIFAR100} & \textbf{Pets} 
        & \textbf{DTD} & \textbf{Flowers} 
        & \textbf{RESISC45} & \textbf{Aircraft} & \textbf{EuroSAT} 
        & \textbf{Avg} \\ \hline

        \multirow{4}{*}{ViT-B}  
        & Linear probing        
        & ---   
        & 0.935{\scriptsize$\pm$0.004} 
        & 0.788{\scriptsize$\pm$0.001} 
        & 0.919{\scriptsize$\pm$0.003} 
        & 0.661{\scriptsize$\pm$0.008} 
        & 0.830{\scriptsize$\pm$0.008} 
        & 0.892{\scriptsize$\pm$0.004} 
        & 0.292{\scriptsize$\pm$0.005} 
        & 0.954{\scriptsize$\pm$0.003}  
        & 0.784 \\

        & Fine-tuning           
        & 85.8M 
        & 0.983{\scriptsize$\pm$0.001} 
        & 0.888{\scriptsize$\pm$0.001} 
        & 0.935{\scriptsize$\pm$0.001} 
        & 0.708{\scriptsize$\pm$0.002} 
        & 0.929{\scriptsize$\pm$0.007} 
        & 0.963{\scriptsize$\pm$0.004}  
        & 0.719{\scriptsize$\pm$0.006} 
        & 0.989{\scriptsize$\pm$0.002}  
        & 0.889 \\

        & LoRA, QV              
        & 296K  
        & 0.978{\scriptsize$\pm$0.001} 
        & 0.866{\scriptsize$\pm$0.001} 
        & 0.929{\scriptsize$\pm$0.001} 
        & 0.695{\scriptsize$\pm$0.005} 
        & 0.919{\scriptsize$\pm$0.010} 
        & 0.946{\scriptsize$\pm$0.005}  
        & 0.651{\scriptsize$\pm$0.009}  
        & 0.987{\scriptsize$\pm$0.002}  
        & 0.871 \\   

        & LoRA, FFN             
        & 738K  
        & 0.979{\scriptsize$\pm$0.001} 
        & 0.870{\scriptsize$\pm$0.001} 
        & 0.929{\scriptsize$\pm$0.002} 
        & 0.700{\scriptsize$\pm$0.006} 
        & 0.918{\scriptsize$\pm$0.011} 
        & 0.947{\scriptsize$\pm$0.005}  
        & 0.673{\scriptsize$\pm$0.003}  
        & 0.988{\scriptsize$\pm$0.005}  
        & 0.876 \\    

        \hline  

        \multirow{5}{*}{bViT-B} 
        & Linear probing        
        & ---  
        & 0.931{\scriptsize$\pm$0.001} 
        & 0.772{\scriptsize$\pm$0.001} 
        & 0.924{\scriptsize$\pm$0.001} 
        & 0.656{\scriptsize$\pm$0.007} 
        & 0.867{\scriptsize$\pm$0.004} 
        & 0.895{\scriptsize$\pm$0.003} 
        & 0.415{\scriptsize$\pm$0.002} 
        & 0.963{\scriptsize$\pm$0.003}  
        & 0.803 \\

        & Time embedding tuning 
        & 9K    
        & 0.970{\scriptsize$\pm$0.001} 
        & 0.839{\scriptsize$\pm$0.001} 
        & 0.925{\scriptsize$\pm$0.001} 
        & 0.679{\scriptsize$\pm$0.003} 
        & 0.911{\scriptsize$\pm$0.003} 
        & 0.923{\scriptsize$\pm$0.004}   
        & 0.525{\scriptsize$\pm$0.004}  
        & 0.983{\scriptsize$\pm$0.002} 
        & 0.844 \\

        & Fine-tuning           
        & 7.8M  
        & 0.982{\scriptsize$\pm$0.001} 
        & 0.880{\scriptsize$\pm$0.002} 
        & 0.925{\scriptsize$\pm$0.002} 
        & 0.693{\scriptsize$\pm$0.012} 
        & 0.923{\scriptsize$\pm$0.004} 
        & 0.960{\scriptsize$\pm$0.003}  
        & 0.723{\scriptsize$\pm$0.004} 
        & 0.990{\scriptsize$\pm$0.003} 
        & 0.885 \\

        & LoRA, QV              
        & 25K   
        & 0.975{\scriptsize$\pm$0.001} 
        & 0.854{\scriptsize$\pm$0.001} 
        & 0.933{\scriptsize$\pm$0.001} 
        & 0.690{\scriptsize$\pm$0.012} 
        & 0.917{\scriptsize$\pm$0.002} 
        & 0.942{\scriptsize$\pm$0.004}  
        & 0.639{\scriptsize$\pm$0.009} 
        & 0.987{\scriptsize$\pm$0.004} 
        & 0.867 \\

        & LoRA, FFN             
        & 62K   
        & 0.977{\scriptsize$\pm$0.001} 
        & 0.859{\scriptsize$\pm$0.004} 
        & 0.928{\scriptsize$\pm$0.002} 
        & 0.695{\scriptsize$\pm$0.014} 
        & 0.926{\scriptsize$\pm$0.001} 
        & 0.948{\scriptsize$\pm$0.005}  
        & 0.644{\scriptsize$\pm$0.002}  
        & 0.989{\scriptsize$\pm$0.008}  
        & 0.871 \\ 

        \hline
    \end{tabular}
    }
    \label{tab:transfer}
\end{table}

\section{Conclusions}

In this work, we investigated how much transformer depth requires independently parameterized blocks by considering the extreme case in which the entire backbone is replaced by repeated application of a single shared block. bViT shows that much of the performance associated with independently parameterized depth can be recovered through recurrence when the representation is sufficiently wide. The resulting representations remain transferable, and a single low-rank update reused across recurrent depth can efficiently adapt bViT to downstream tasks. Full sharing also provides a controlled setting for studying how computation evolves with depth, since the same attention heads, FFN neurons, weight matrices, and spectral directions can be followed directly across recurrent steps. We find that parameter sharing does not imply functional invariance. Classification and spatial attention follow different temporal trajectories, fixed neurons change their preferred activation patterns, and shared spectral directions vary in importance across recurrence. We conduct spectral analysis and quantify how the shared parameter space is reused by measuring the persistence of selected directions.  Our results suggest that stage-dependent computation can emerge from the interaction between shared parameters and an evolving hidden state, helping explain why a fixed low-rank adaptation can remain expressive across recurrent depth.

Our work also has several limitations. Recurrent sharing reduces parameter depth but not computational depth, and narrow recurrent models remain substantially weaker than their independently parameterized counterparts. Our comparisons use a matched training recipe and are restricted to image classification and transfer learning. Nevertheless, the results suggest that independently parameterized layers are not necessary for stage-dependent transformer computation. They may also inform partially shared architectures by clarifying how the value of depth-specific parameterization depends on representation width and capacity allocation. Also, single-block recurrent models may provide a useful testbed for weight-space studies in which independently parameterized depth complicates analysis, including linear mode connectivity, model merging, unlearning, and parameter-space structure.

\bibliography{iclr2027_conference}
\bibliographystyle{iclr2027_conference}

\appendix

\clearpage

\clearpage
\section{Width and recurrence analysis}
\label{app:width}

We perform additional experiments based on the bViT-S configuration to investigate why full block sharing is less effective at small representation width. Figure~\ref{fig:ablation} separates the effects of recurrent depth and embedding dimension. Increasing the maximum recurrent depth used during both training and evaluation from 12 to 18 steps produces little change in validation accuracy. Increasing the embedding dimension from 384 to 768 yields a steady improvement. We also found that increasing the number of attention heads from 6 to 48 at fixed embedding dimension similarly leaves accuracy around 0.68. These results indicate that the limitation of the narrow recurrent model cannot be resolved simply by increasing the number of recurrent applications or attention heads.

\begin{figure}[h!]
    \centering

    \begin{subfigure}[b]{0.47\linewidth}
        \centering
        \includegraphics[width=\linewidth]{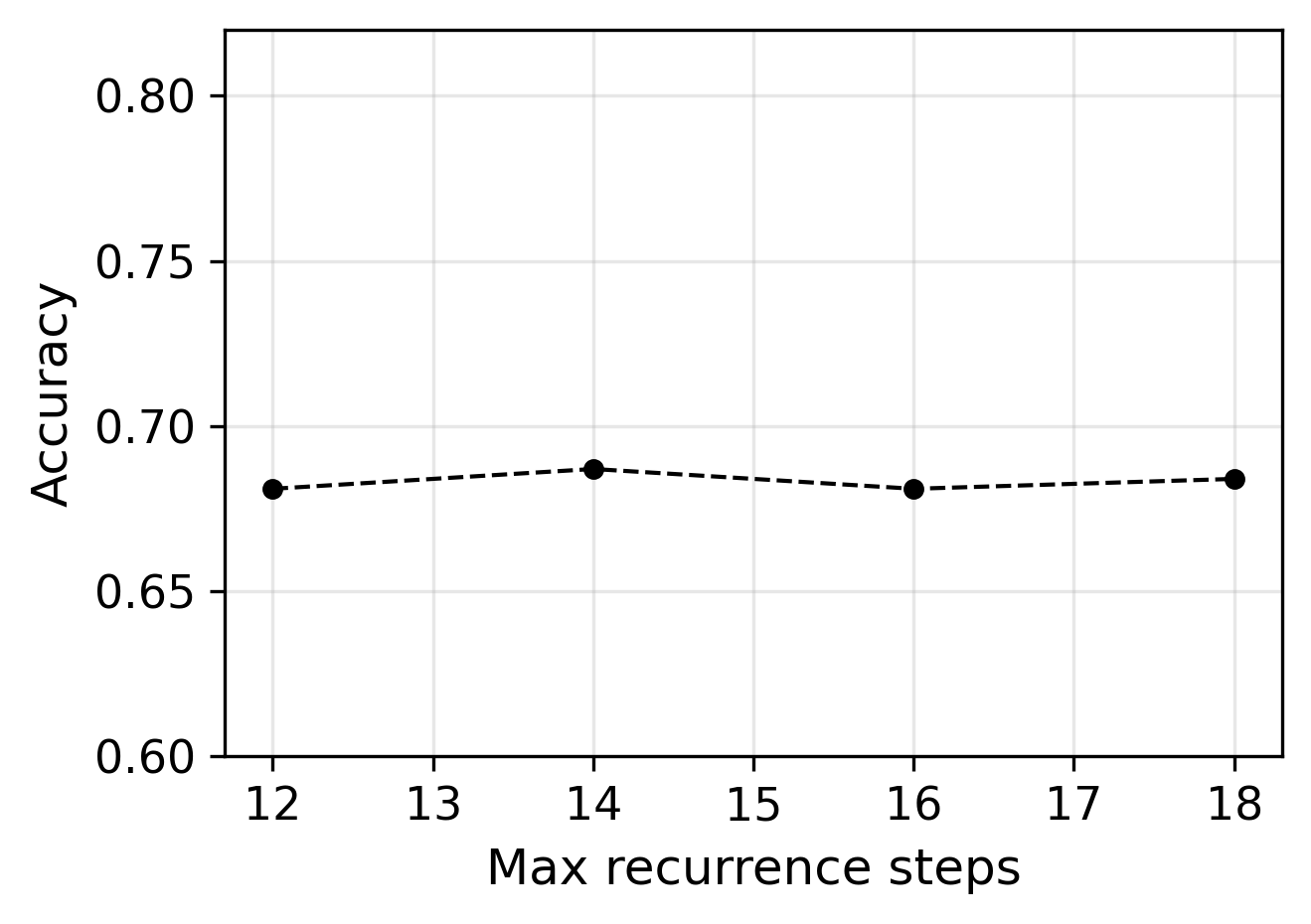}
        \caption{Accuracy with increasing recurrent depth.}
        \label{fig:steps_small}
    \end{subfigure}
    \hfill
    \begin{subfigure}[b]{0.47\linewidth}
        \centering
        \includegraphics[width=\linewidth]{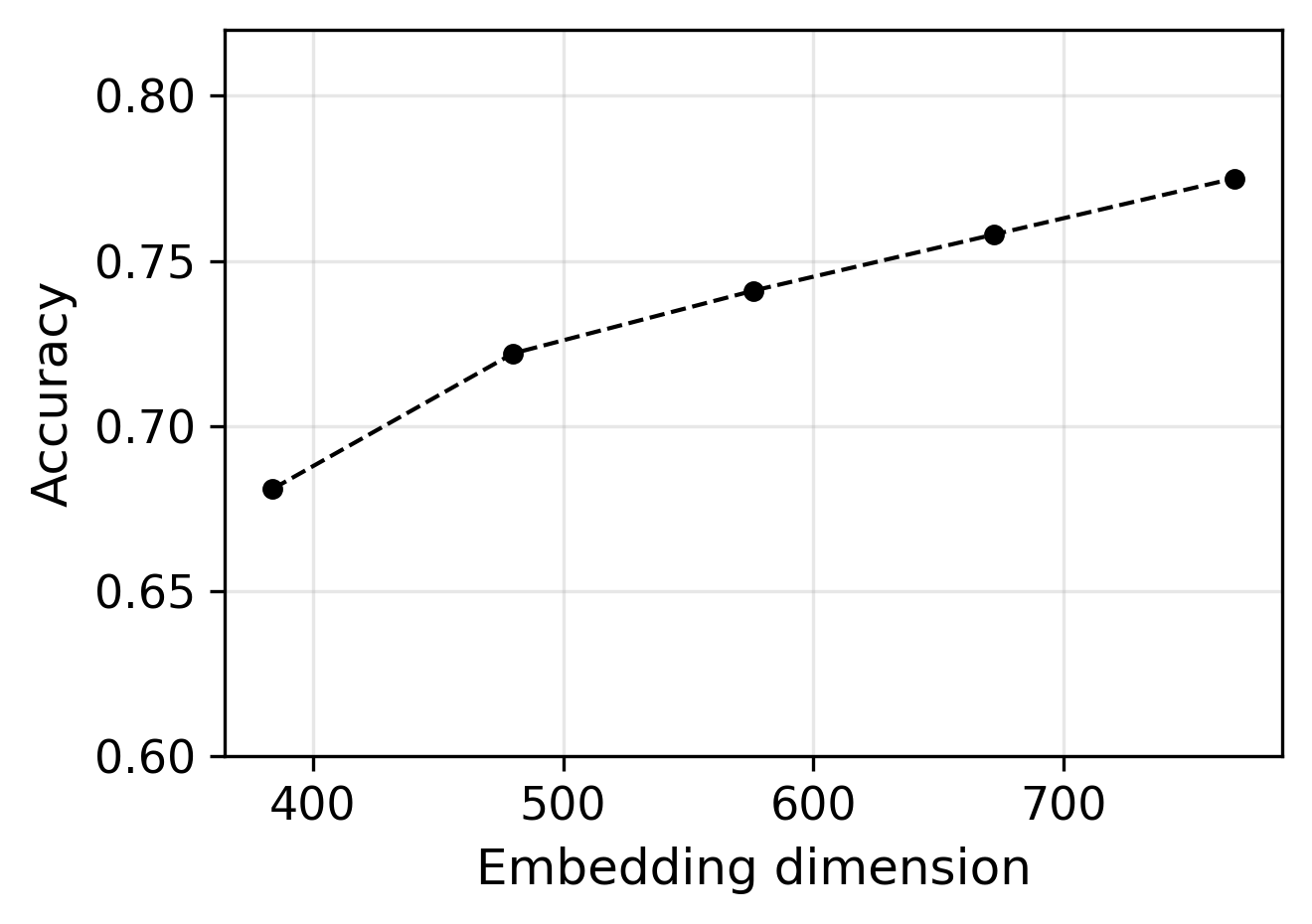}
        \caption{Accuracy with increasing embedding dimension.}
        \label{fig:embedding_small}
    \end{subfigure}

    \caption{Effect of recurrent depth and representation width in bViT-S. Increasing recurrence beyond 12 steps provides little benefit, whereas increasing the embedding dimension consistently improves accuracy.}
    \label{fig:ablation}
\end{figure}

Table~\ref{tab:block_ablation} provides a complementary comparison at fixed embedding dimension and approximately fixed computation. We keep 12 total block evaluations but vary how they are divided between independently parameterized blocks and recurrent applications. Accuracy improves progressively as more of the computation is assigned to distinct blocks, from 0.681 with one block repeated 12 times to 0.782 with 12 independently parameterized blocks. Together with the width experiment, this shows that recurrent reuse and independently parameterized depth are not interchangeable in the narrow regime. Increasing representation width, however, substantially reduces this limitation, supporting our main observation that width is an important factor in whole-network block sharing. We additionally construct a parameter matched family at approximately the bViT-B budget ($\approx$8.7M parameters), again using 12 block evaluations. As capacity is shifted from representation width toward independently parameterized blocks, accuracy decreases monotonically from 0.779 for bViT-B to 0.747 for a 12-block ViT with embedding dimension 240, as shown in Table~\ref{tab:param_matched_ablation}. Whereas Table~\ref{tab:block_ablation} fixes representation width, this comparison approximately fixes parameter count and shows that, under a small parameter budget, allocating capacity to width is more effective than allocating it to independently parameterized depth. This comparison is not FLOP matched, since the wide recurrent model requires more computation than the narrowest configuration.

\begin{table}[h!]
    \centering
    \caption{ViT-S variants with different allocations of computation between independently parameterized blocks and recurrent steps. All models use embedding dimension 384 and 12 total block evaluations. Accuracy increases as more of the computation is assigned to independently parameterized blocks.}
    \resizebox{0.9\textwidth}{!}{%
    \begin{tabular}{lccccccc}
        \hline
        Model & Emb dim & Heads & Blocks & Steps & Params & FLOPs & Val acc \\
        \hline
        bViT-S (1B,12S)  & 384 & 6 & 1  & 12 & 2.5M  & 8.5G & 0.681 \\
        ViT-S (2B,6S)    & 384 & 6 & 2  & 6  & 4.3M  & 8.5G & 0.729 \\
        ViT-S (3B,4S)    & 384 & 6 & 3  & 4  & 6.1M  & 8.5G & 0.754 \\
        ViT-S (4B,3S)    & 384 & 6 & 4  & 3  & 7.9M  & 8.5G & 0.764 \\
        ViT-S (6B,2S)    & 384 & 6 & 6  & 2  & 11.4M & 8.5G & 0.781 \\
        ViT-S (12B,1S)   & 384 & 6 & 12 & 1  & 22.0M & 8.5G & 0.782 \\
        \hline
    \end{tabular}
    }
    \label{tab:block_ablation}
\end{table}

\begin{table}[t]
    \centering
    \caption{Parameter-matched comparison of recurrent and independently parameterized depth. 
    All configurations perform 12 block evaluations and use approximately the parameter budget of bViT-B. 
    At this fixed parameter budget, allocating capacity to a wider recurrent representation yields higher validation accuracy.}
    \resizebox{0.78\textwidth}{!}{%
    \begin{tabular}{lccccc}
        \hline
        Model & Emb dim & Blocks & Steps & Params & Val acc \\
        \hline
        bViT-B (1B,12S) & 768 & 1  & 12 & 8.6M  & 0.779 \\
        ViT (2B,6S)     & 564 & 2  & 6  & $\sim$8.6M & 0.774 \\
        ViT (4B,3S)     & 408 & 4  & 3  & $\sim$8.8M & 0.769 \\
        ViT (12B,1S)    & 240 & 12 & 1  & $\sim$8.8M & 0.747 \\
        \hline
    \end{tabular}
    }
    \label{tab:param_matched_ablation}
\end{table}

\clearpage
\section{Additional bViT variants}
\label{app:bvit_variants}

We additionally examine whether bViT benefits from mechanisms that provide explicit step information, additional latent state, or supervision from a pretrained teacher. The results are summarized in Table~\ref{tab:vit_variants}.

\begin{wraptable}{r}{7cm}
    \vspace{-4mm}
    \centering
    \caption{Validation accuracy of additional bViT-B variants on ImageNet-1K.}
    \label{tab:vit_variants}

    \begin{tabular}{llc}
        \hline
        Model  & Modification & Val acc \\
        \hline
        bViT-B    & Baseline             & 0.779 \\
        bViT-B-TE & Time-step embeddings & 0.780 \\
        bViT-B-R  & Register tokens      & 0.781 \\
        Dist bViT     & Distillation         & 0.791 \\
        \hline
    \end{tabular}

    \vspace{-3mm}
\end{wraptable}

We first augment bViT-B with learned time-step embeddings, following recurrent transformer architectures that explicitly condition computation on the iteration index~\citep{dehghani2018universal,shu2026loopvit}. At recurrent step $t$, the corresponding learnable embedding is added to the token representation before applying the shared block:

\begin{equation}
x^{t+1} = F(x^t + TE^t),
\label{eq:bViT-TE}
\end{equation}

where $TE^t \in \mathbb{R}^{d}$ denotes the embedding associated with step $t$. We also evaluate a variant with four register tokens, which provide additional latent state that can participate in the recurrent computation~\citep{darcet2024vision,jolicoeur2025less}. Both modifications lead to similar performance compared with the baseline, indicating that strong recurrent performance does not require either explicit step embeddings or additional register tokens.

We further examine whether bViT benefits from knowledge distillation. Following DeiT~\citep{touvron2021training}, we introduce a distillation token and use hard-label distillation with a pretrained \texttt{regnet\_y\_16gf} teacher. The resulting model reaches a validation accuracy of 0.791 compared with 0.779 for bViT-B trained without distillation. This shows that the recurrent architecture is compatible with standard distillation and that teacher supervision can further improve its classification performance.

\clearpage
\section{Rank analysis}\label{app:rank}
\label{app:rank}

To complement the analyses in the main paper, we examine the spectral structure of the learned weight matrices and the evolution of latent representations across depth. We compare bViT-B and ViT-B, both with embedding dimension 768.

\paragraph{Energy rank of the learned transformations.}

For each query, key, value, attention output, FFN input, and FFN output projection, we compute the $95\%$ energy rank using singular value decomposition. Given the singular values $\{\sigma_i\}$ of a weight matrix, we define the energy rank as the smallest $r$ for which the leading $r$ singular values account for at least $95\%$ of the total squared singular-value energy:

\[
r_{0.95}
=
\min
\left\{
r
\;\middle|\;
\frac{\sum_{i=1}^{r}\sigma_i^2}
{\sum_i \sigma_i^2}
\geq 0.95
\right\}.
\]

Figure~\ref{fig_ranks_various} compares the resulting ranks across recurrent steps of bViT-B and blocks of ViT-B. For bViT-B, each projection is represented by a single shared matrix and therefore has the same energy rank at every recurrent step. These shared matrices generally exhibit high spectral rank. In ViT-B, the corresponding projections are independently parameterized and their energy rank tends to increase with depth. For several transformations, including the attention and FFN projections, the ranks of the later ViT-B blocks approach those of the shared bViT-B matrices.

The comparison shows that full block sharing does not lead bViT-B to rely on spectrally narrow transformations. Instead, the recurrent model learns high-rank shared operators that are repeatedly applied as the hidden representation evolves.

\begin{figure*}[h!]
    \centering

    \begin{minipage}{0.32\textwidth}
        \centering
        \includegraphics[width=\textwidth]{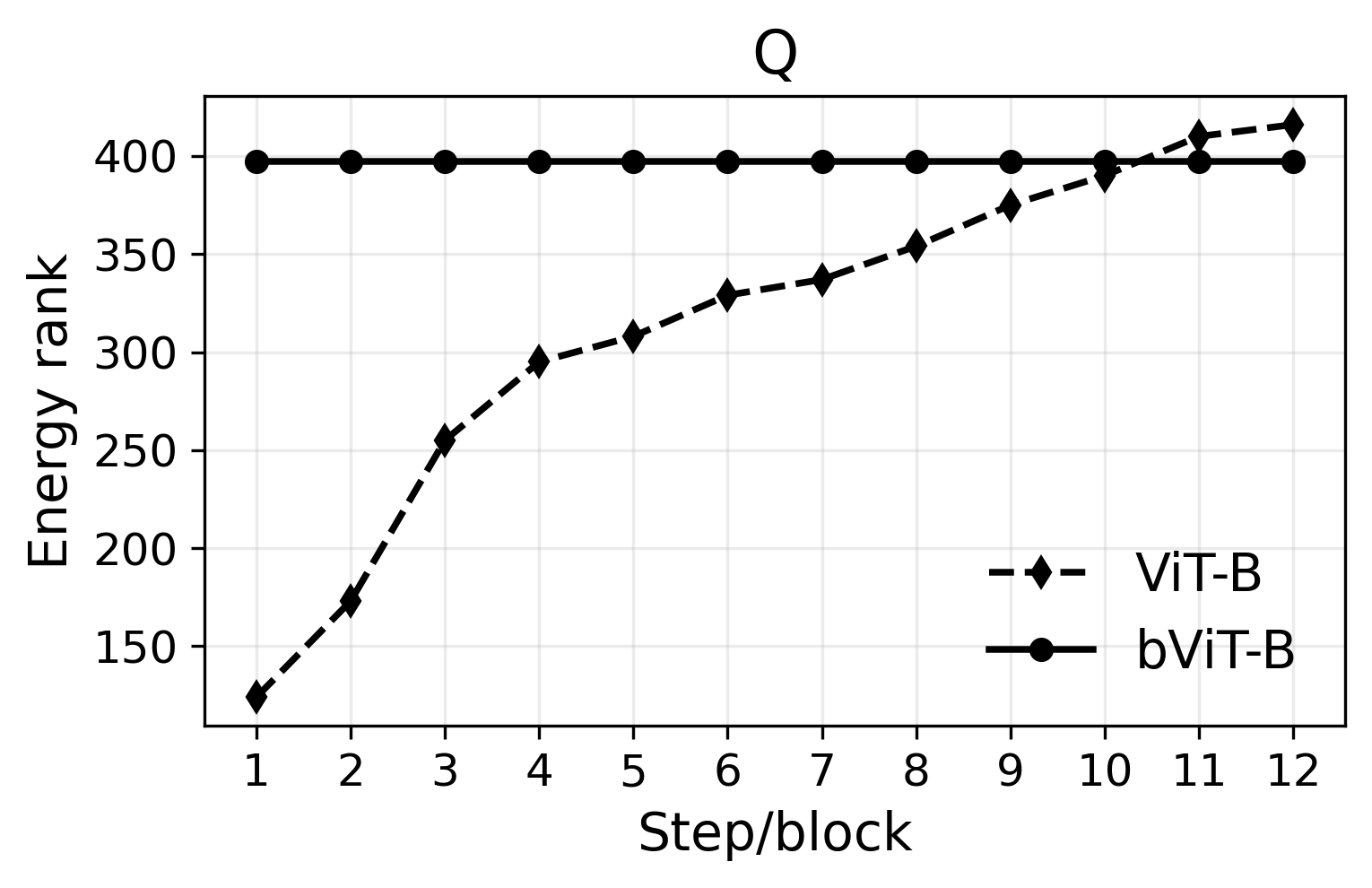}
    \end{minipage}
    \hfill
    \begin{minipage}{0.32\textwidth}
        \centering
        \includegraphics[width=\textwidth]{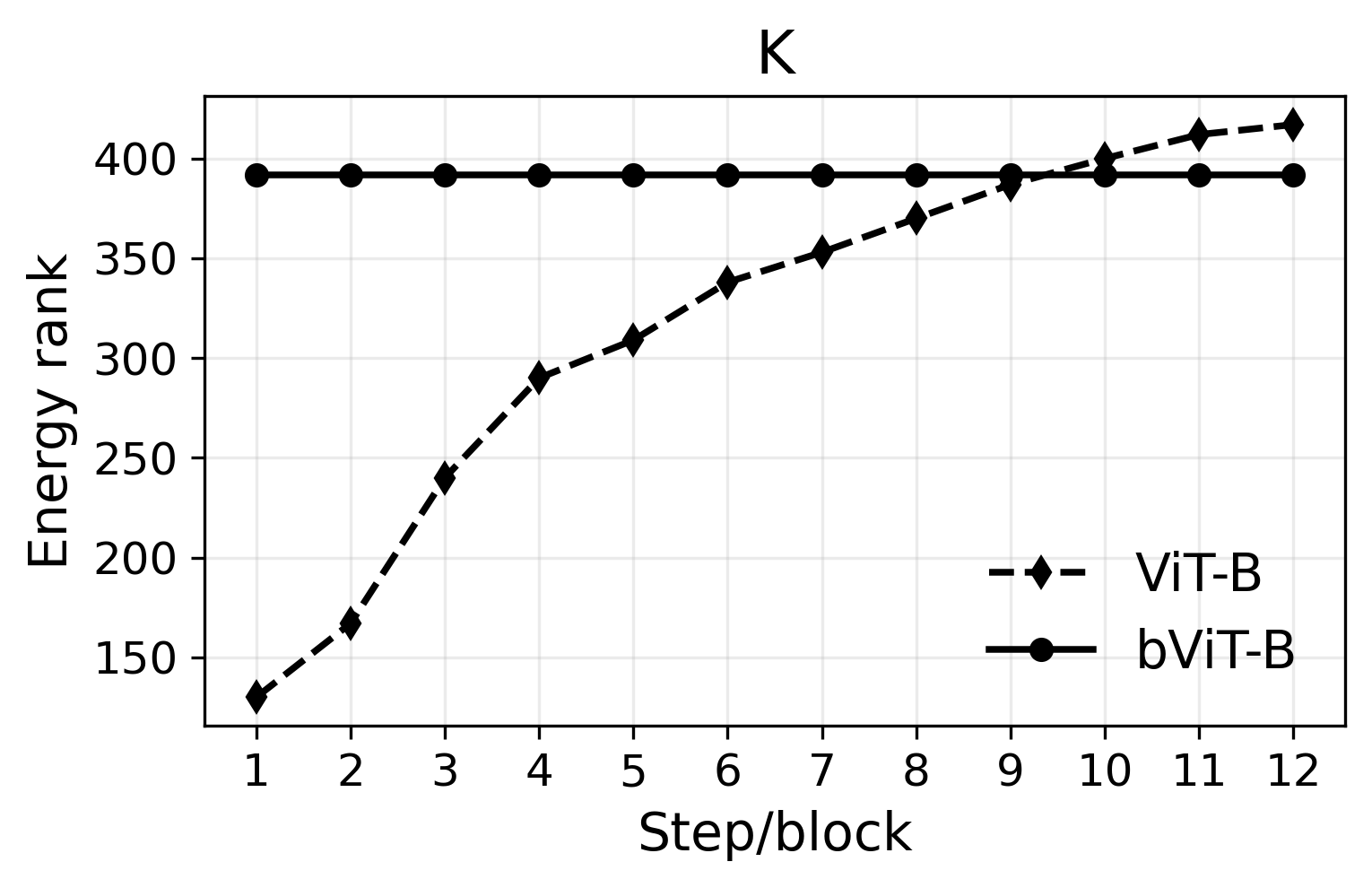}
    \end{minipage}
    \hfill
    \begin{minipage}{0.32\textwidth}
        \centering
        \includegraphics[width=\textwidth]{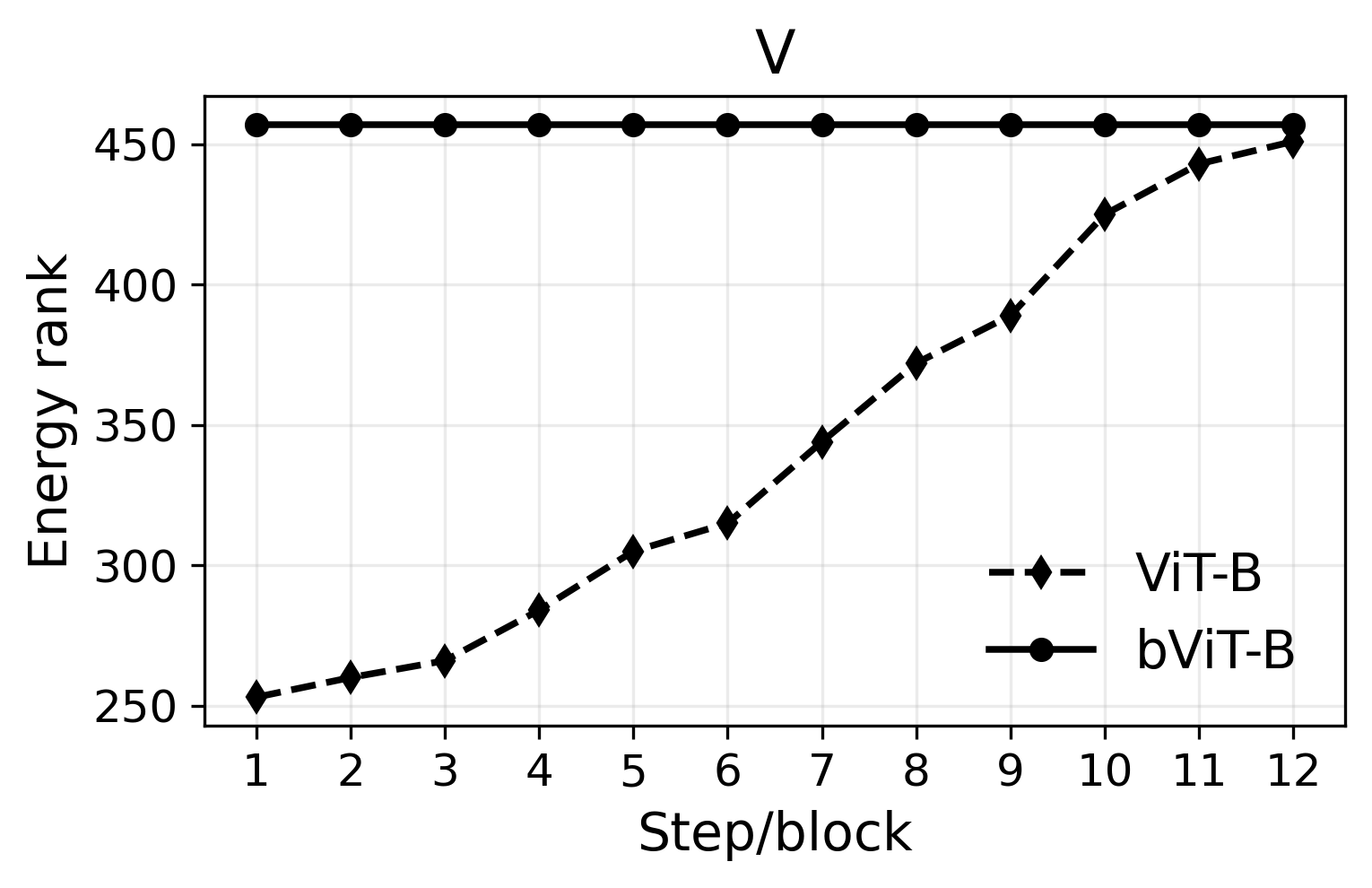}
    \end{minipage}

    \vspace{0.25cm}

    \begin{minipage}{0.32\textwidth}
        \centering
        \includegraphics[width=\textwidth]{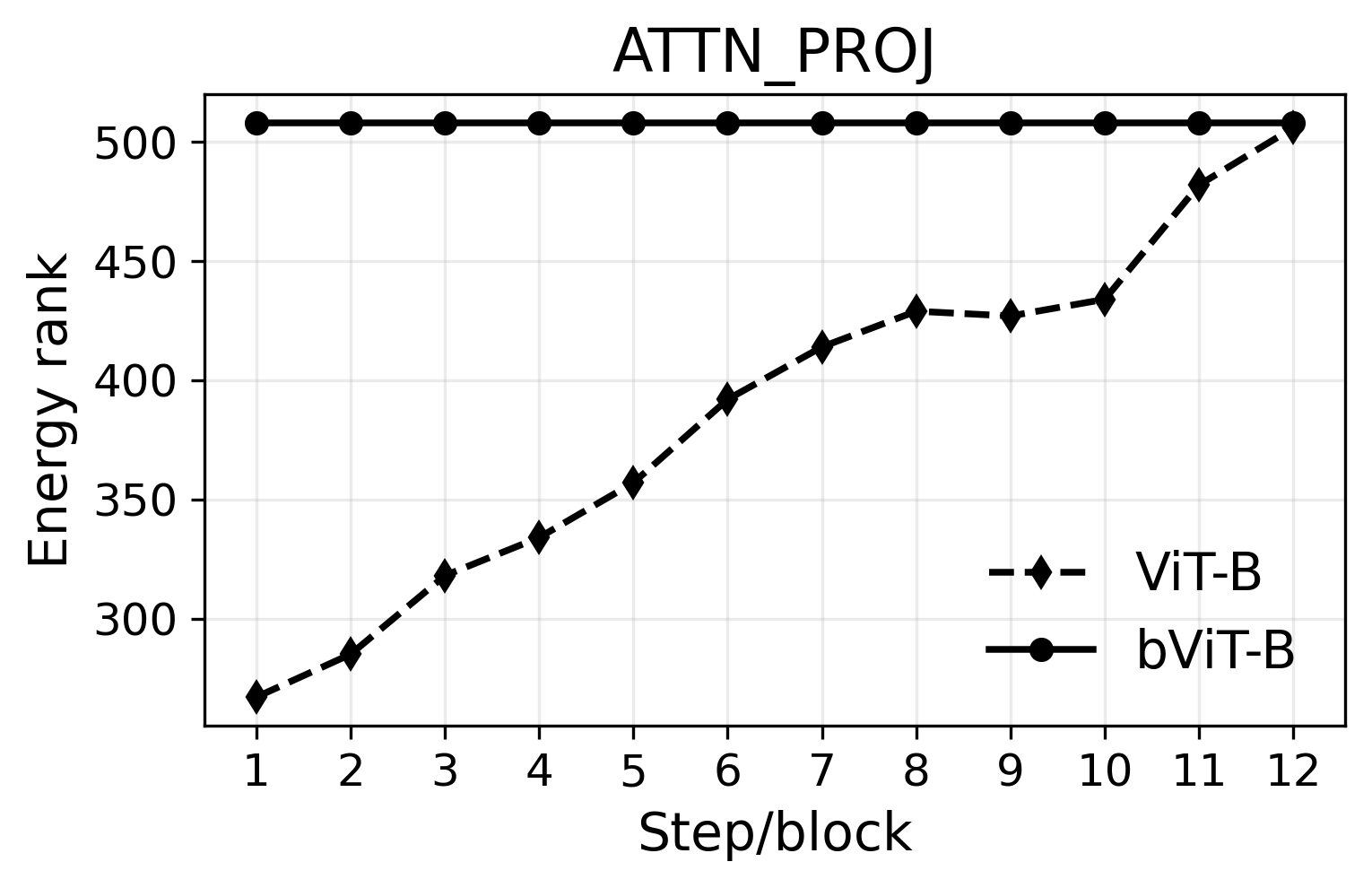}
    \end{minipage}
    \hfill
    \begin{minipage}{0.32\textwidth}
        \centering
        \includegraphics[width=\textwidth]{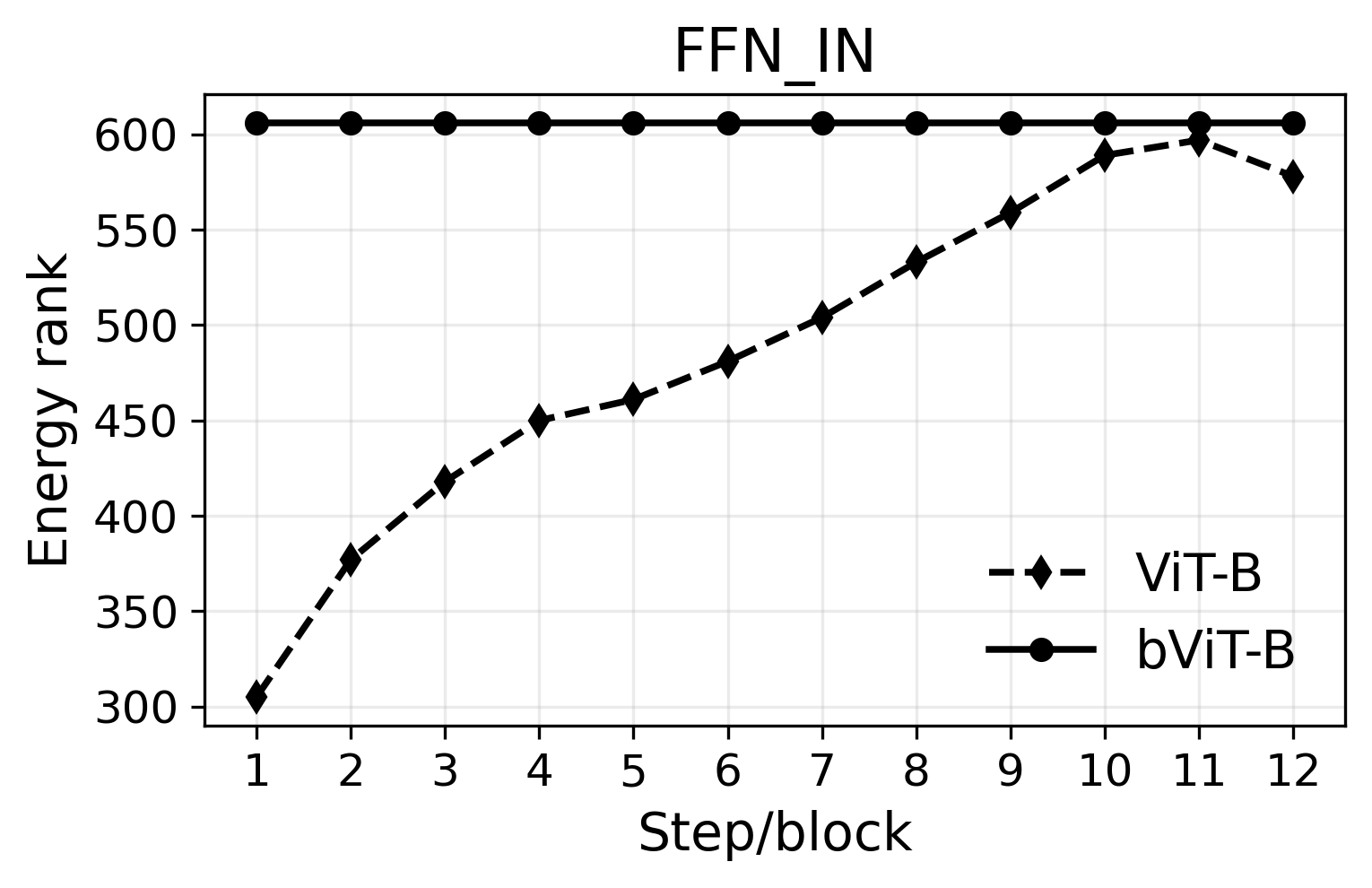}
    \end{minipage}
    \hfill
    \begin{minipage}{0.32\textwidth}
        \centering
        \includegraphics[width=\textwidth]{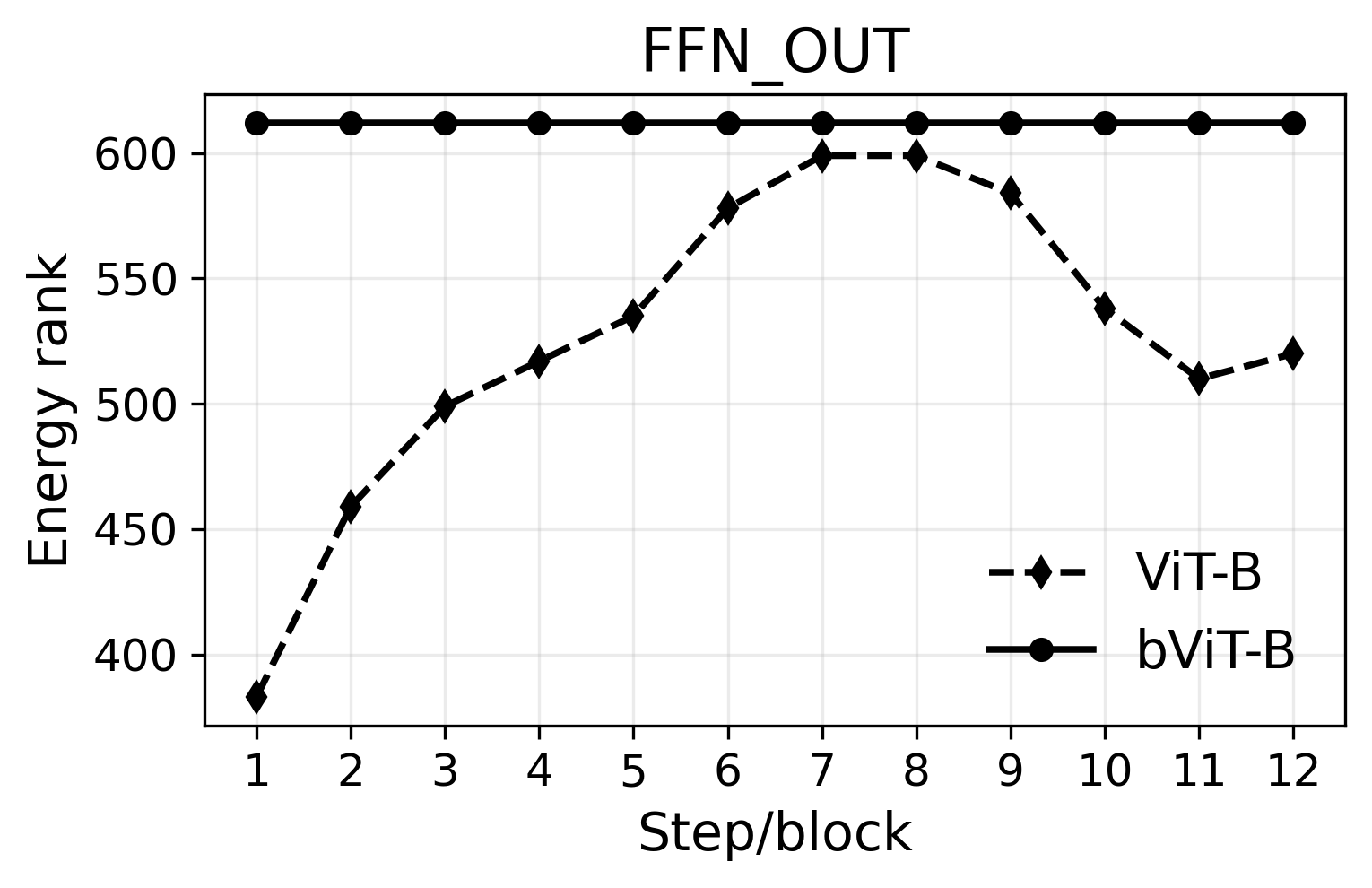}
    \end{minipage}

    \caption{
    $95\%$ energy rank of the query, key, value, attention output, FFN input, and FFN output projections in ViT-B and bViT-B. Each ViT-B point corresponds to an independently parameterized block. In bViT-B, the same matrix is reused at every recurrent step and therefore has constant energy rank across depth. The shared bViT-B matrices generally exhibit high spectral rank, while the corresponding ViT-B projections tend to increase in rank toward later blocks.
    }
    \label{fig_ranks_various}
\end{figure*}

\paragraph{Evolution of latent representations.}

We next examine whether the latent representation continues to evolve when the same transformer block is reused across depth. We collect the \textsc{cls} representations from the validation set after each recurrent step of bViT-B and after each block of ViT-B. For each stage, we form a matrix containing the corresponding CLS representations and compute its $95\%$ energy rank using the same criterion as above.

The left panel of Fig.~\ref{fig_representation_analysis} shows that the energy rank of the CLS representation increases progressively in both architectures. In bViT-B, this increase occurs despite every recurrent step applying exactly the same transformer block. The final representations of bViT-B and ViT-B reach comparable energy ranks.

We additionally measure representation similarity across depth using CKA. The middle and right panels of Fig.~\ref{fig_representation_analysis} show pairwise CKA similarities between recurrent steps of bViT-B and blocks of ViT-B. Both models exhibit a structured progression in representation space, with nearby stages generally remaining more similar and larger changes accumulating across depth. bViT-B changes comparatively smoothly across recurrence, whereas ViT-B shows a slightly sharper representational transition in its final blocks. Importantly, bViT-B exhibits substantial representation evolution despite keeping the block parameters fixed. Reusing the same transformation therefore does not imply repeated representations, as successive applications operate on different hidden states and progressively transform the latent representation.

\begin{figure*}[h!]
    \centering

    \begin{minipage}{0.35\textwidth}
        \centering
        \includegraphics[width=\textwidth]{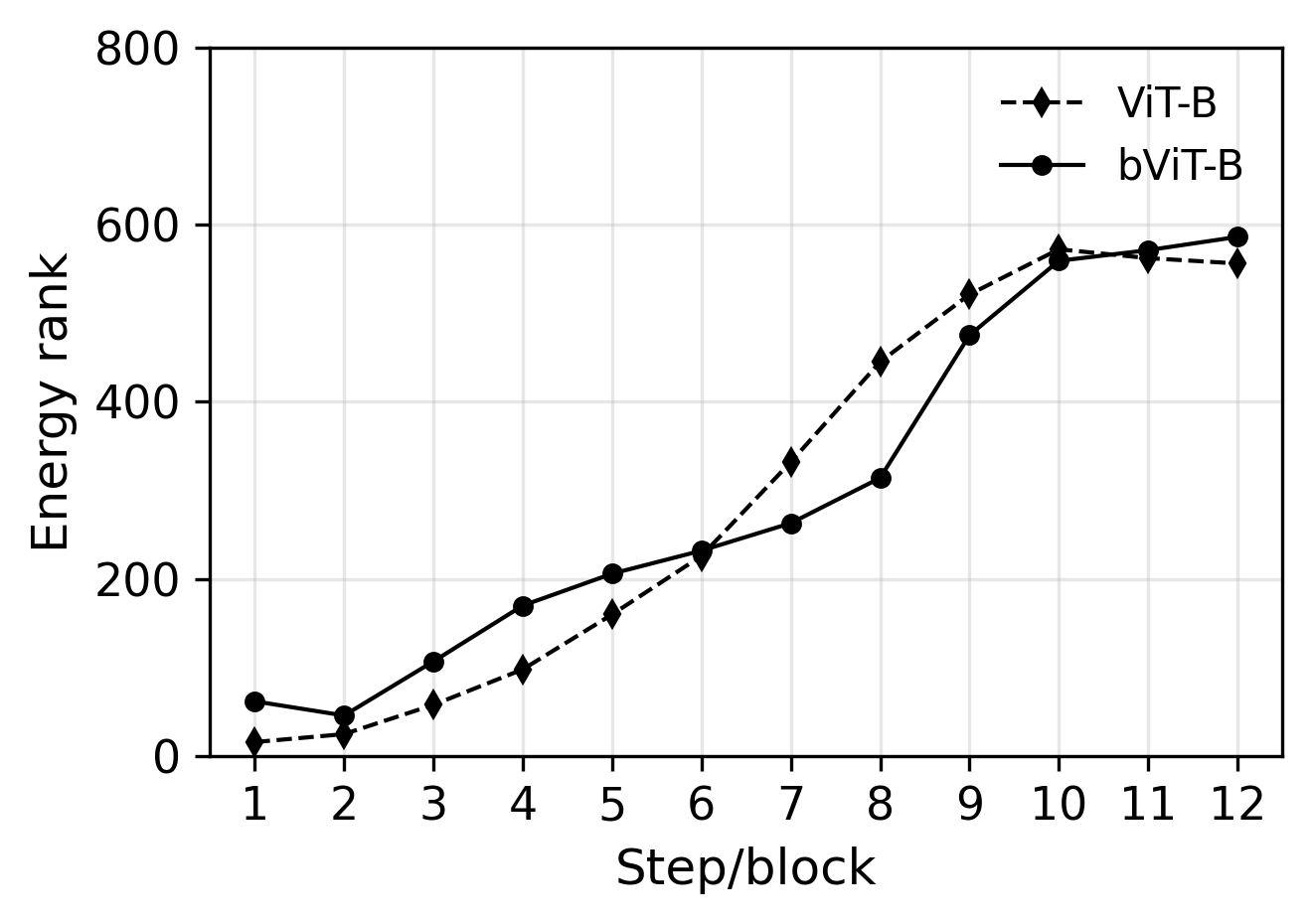}
    \end{minipage}
    \hfill
    \begin{minipage}{0.28\textwidth}
        \centering
        \includegraphics[width=\textwidth]{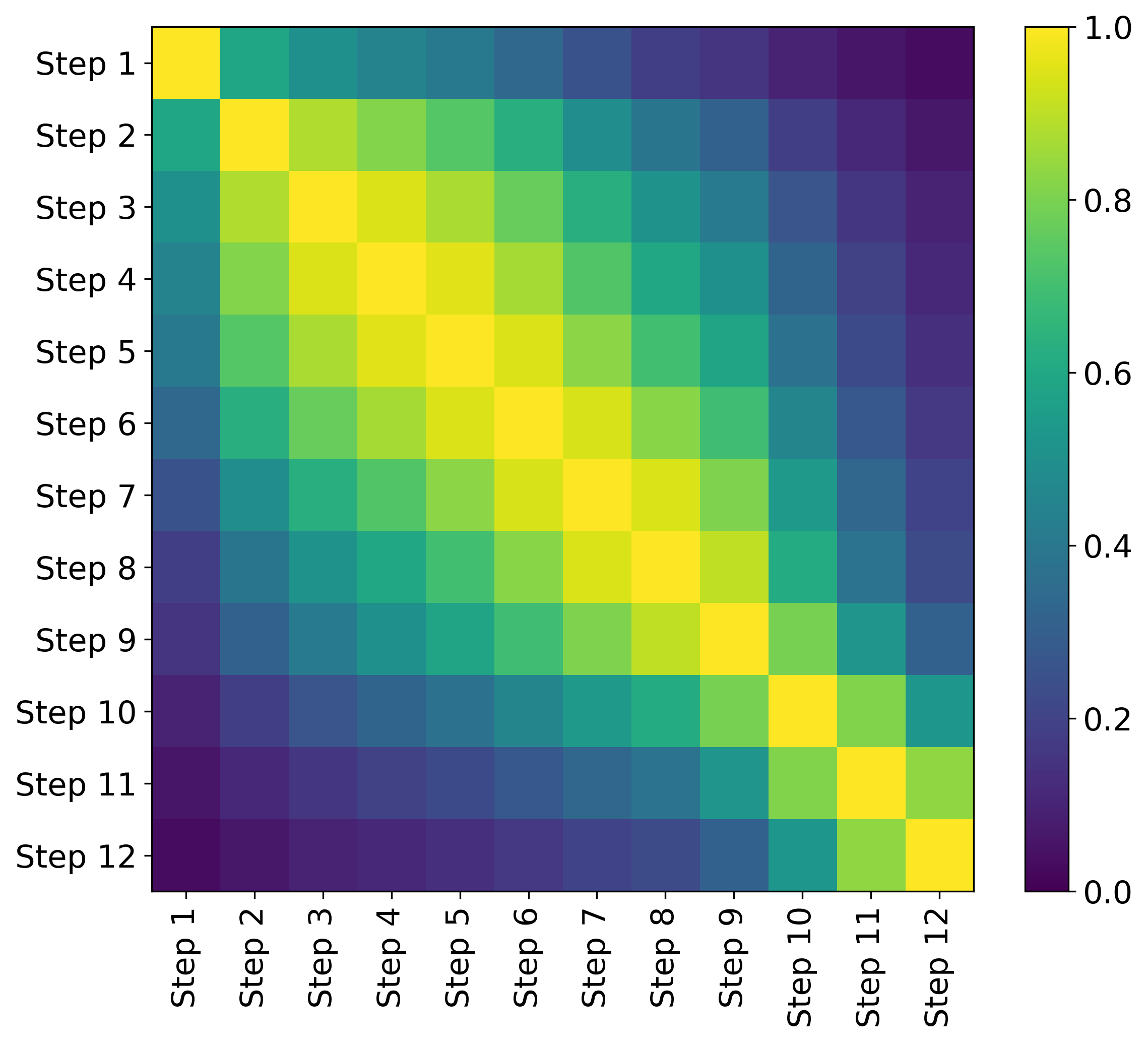}
    \end{minipage}
    \hfill
    \begin{minipage}{0.28\textwidth}
        \centering
        \includegraphics[width=\textwidth]{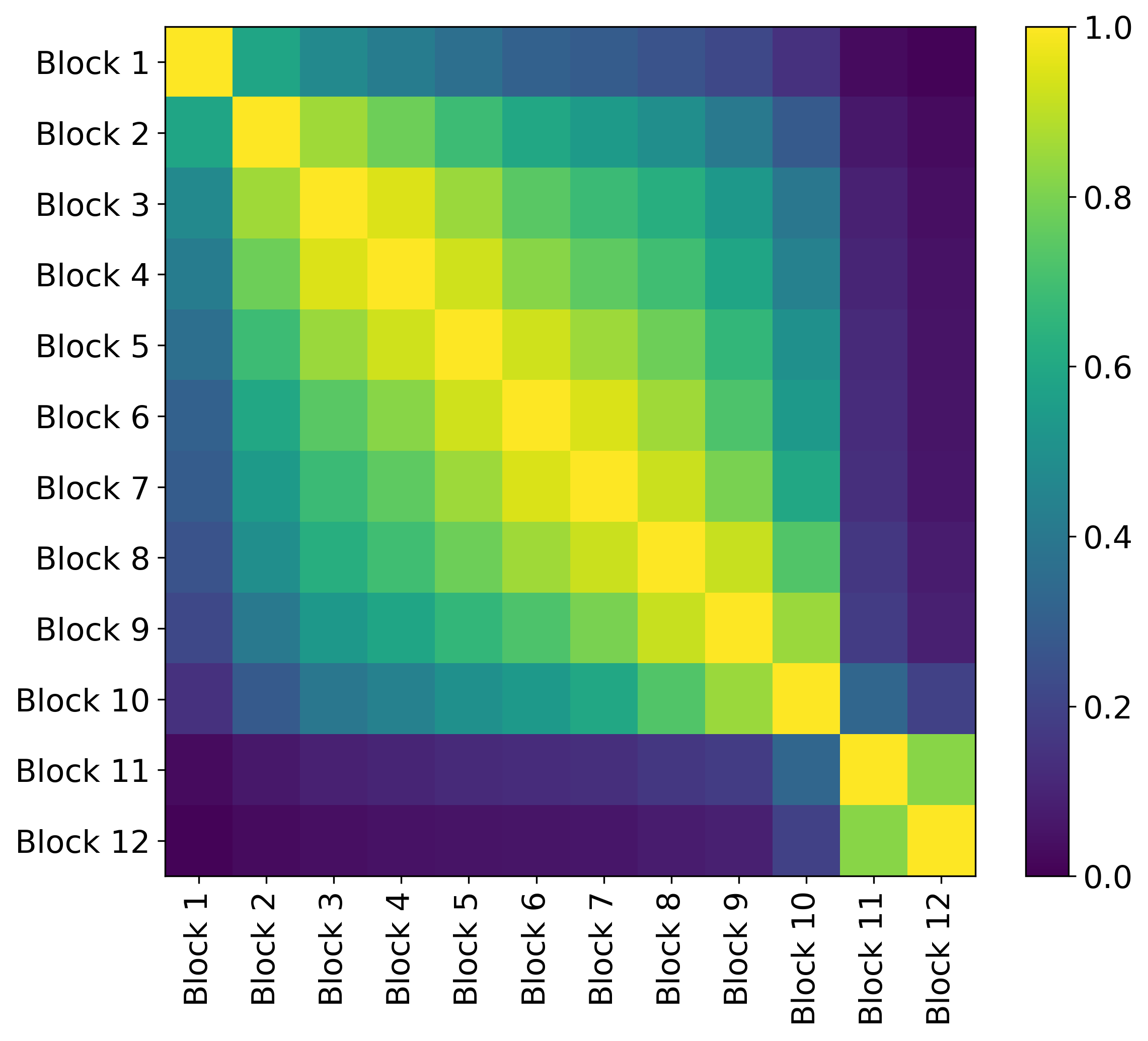}
    \end{minipage}

    \caption{
    Evolution of latent representations across recurrent steps and transformer blocks. The left panel shows the $95\%$ energy rank of the \textsc{cls} representations in bViT-B and ViT-B. The middle and right panels show CKA similarity between representations obtained at different recurrent steps of bViT-B and different blocks of ViT-B. 
    }
    \label{fig_representation_analysis}
\end{figure*}

\clearpage
\section{Additional attention analysis}
\label{app:attention_heads}

We provide additional results for the attention analysis presented in the main text. Figure~\ref{fig:base_heads} reports the pointing-game hit ratio separately for all attention heads of bViT-B across recurrent steps. The heads exhibit distinct temporal profiles, with several showing their strongest object localization at intermediate recurrent stages rather than at the final step. This complements the aggregate results in the main text and shows that the evolution of spatial attention is not driven by a single head.

Figure~\ref{fig:attention_maps} provides qualitative examples for the head with the highest pointing-game performance. Because the same attention head is reused throughout the recurrence, its spatial behavior can be followed directly across steps. The examples show substantial changes in spatial focus during recurrent computation, with object-centric attention often becoming more concentrated at intermediate or later stages.

\begin{figure}[h!]
  \centering
  \includegraphics[width=1\linewidth]{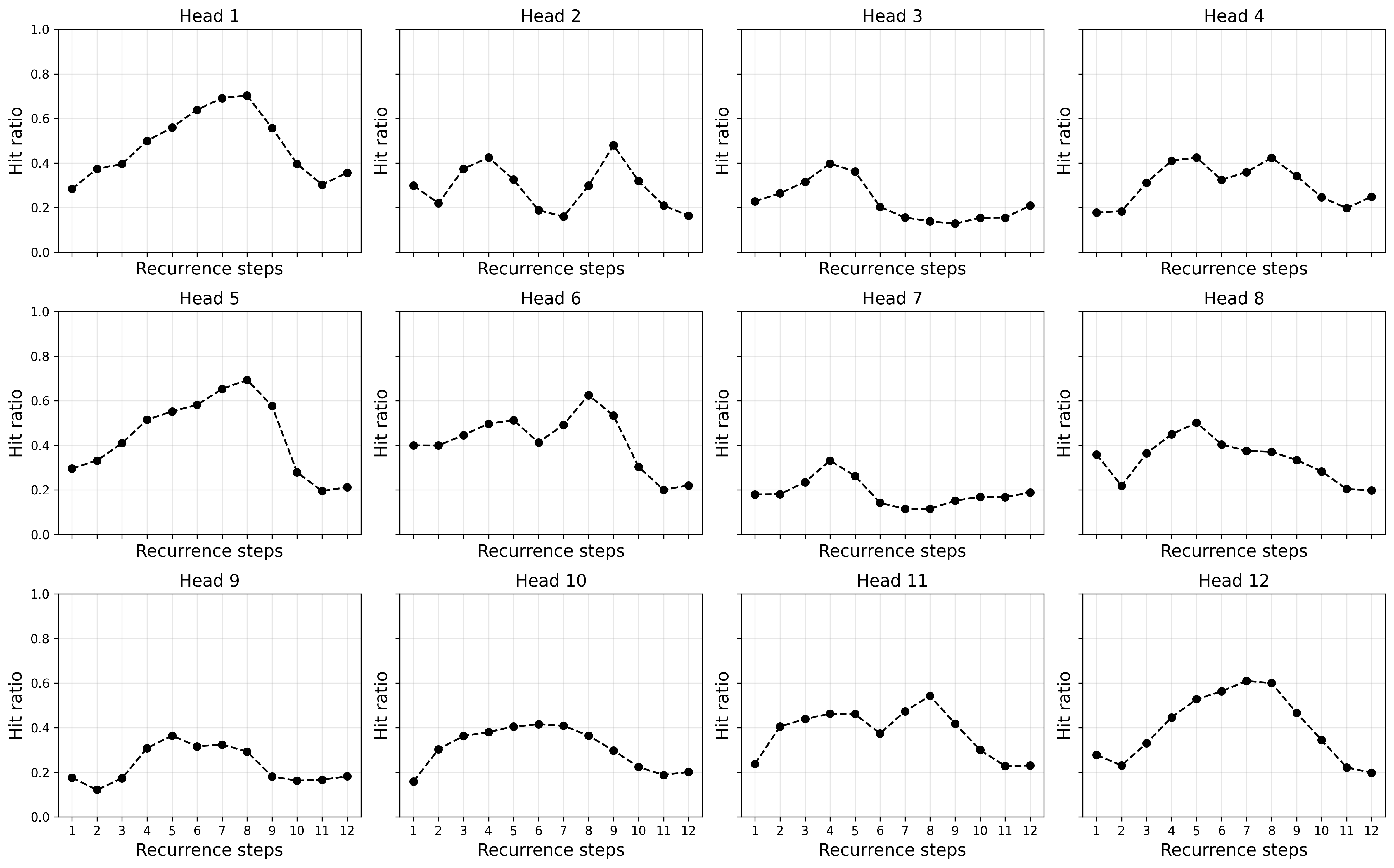}
  \caption{Pointing-game hit ratio for individual attention heads of bViT-B across recurrent steps. Each subplot follows the same shared attention head throughout the recurrent computation. Different heads exhibit distinct temporal profiles, and several reach their highest localization performance at intermediate recurrent stages.}
  \label{fig:base_heads}
\end{figure}

\begin{figure}[!t]
  \centering
  \includegraphics[width=0.8\linewidth]{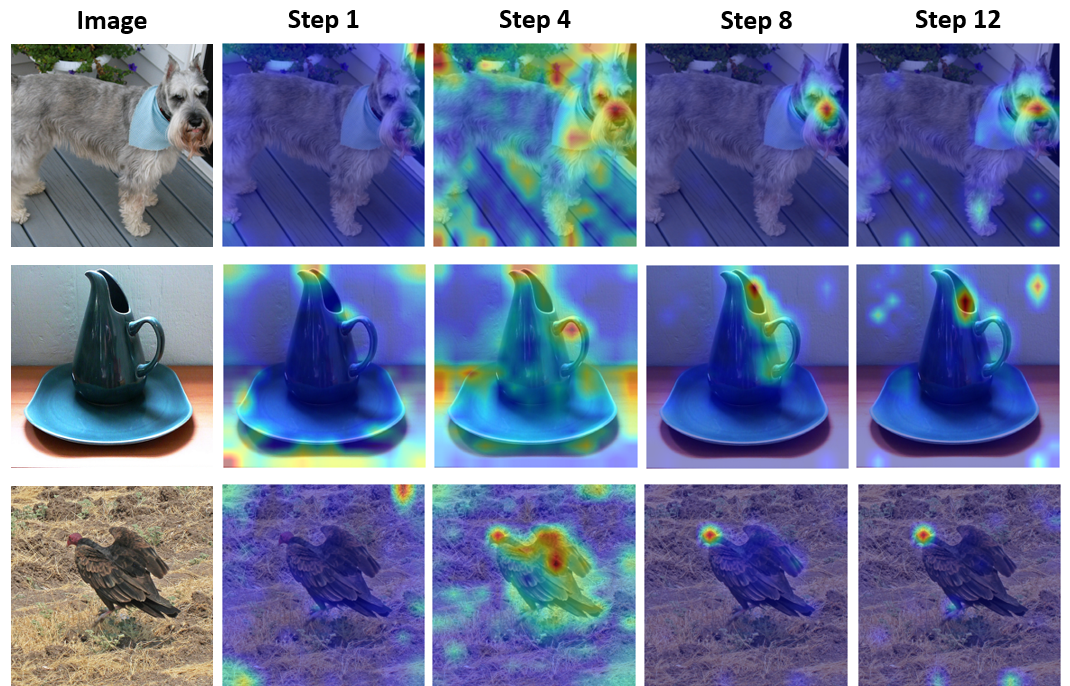}
  \caption{Qualitative attention maps for the bViT-B head with the highest pointing-game performance. Columns show the input image and attention maps at recurrent steps 1, 4, 8, and 12. The spatial focus of the same attention head changes substantially across recurrence, with object-centric attention often becoming more concentrated during intermediate or later stages.}
  \label{fig:attention_maps}
\end{figure}

\clearpage
\section{Additional step-conditioned spectral analysis}
\label{app:step_spectral}

The main text analyzes step-conditioned spectral utilization for Q, V, and FFN\_IN. Here, we provide the corresponding results for the remaining K, attention output, and FFN\_OUT projections using the same spectral importance, similarity, and pruning procedures. Figure~\ref{fig:spearman_app} reports the similarity of step-conditioned importance profiles across recurrent depth, Fig.~\ref{fig:clsvspatch_app} compares CLS and patch importance profiles at each recurrent step, and Fig.~\ref{fig:pruning_svd_app} presents the corresponding spectral pruning and temporal persistence results.

We also report a chance baseline for the temporal persistence analysis in Fig.~\ref{fig:null_persistence}, where each recurrent step independently retains the same fraction of spectral directions at random. Under this null model, stronger pruning naturally shifts probability toward directions selected at only a few recurrent steps. The observed step-conditioned masks depart systematically from this baseline. Across all six projections, directions retained in $9$--$12$ steps are substantially overrepresented, while those retained in $5$--$8$ steps are substantially underrepresented relative to chance, as shown in Figs.~\ref{fig:pruning_svd} and~\ref{fig:pruning_svd_app}.

\begin{figure*}[h!]
    \centering


    \begin{minipage}{0.32\textwidth}
        \centering
        \includegraphics[width=\textwidth]{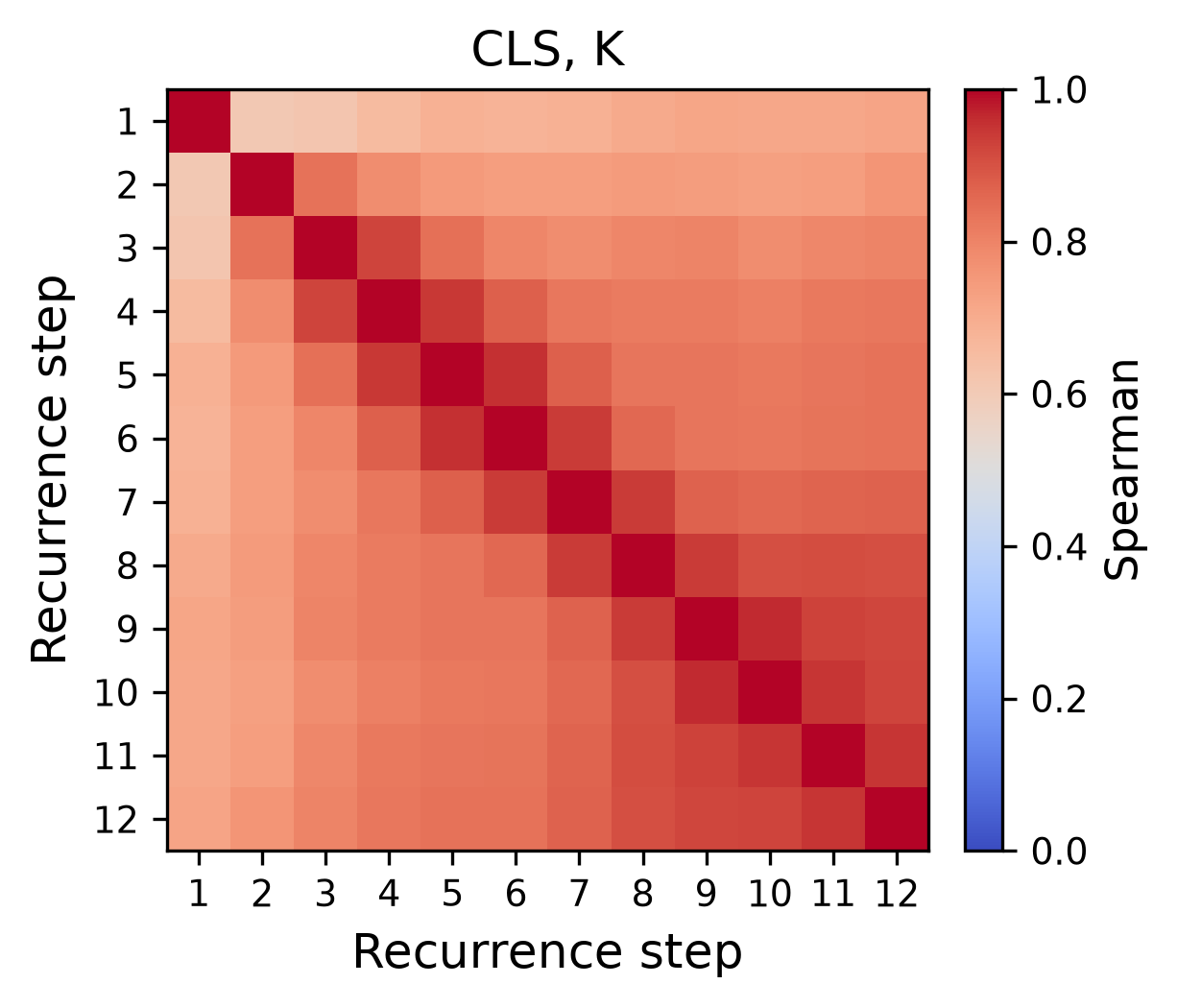}
    \end{minipage}
    \hfill
    \begin{minipage}{0.32\textwidth}
        \centering
        \includegraphics[width=\textwidth]{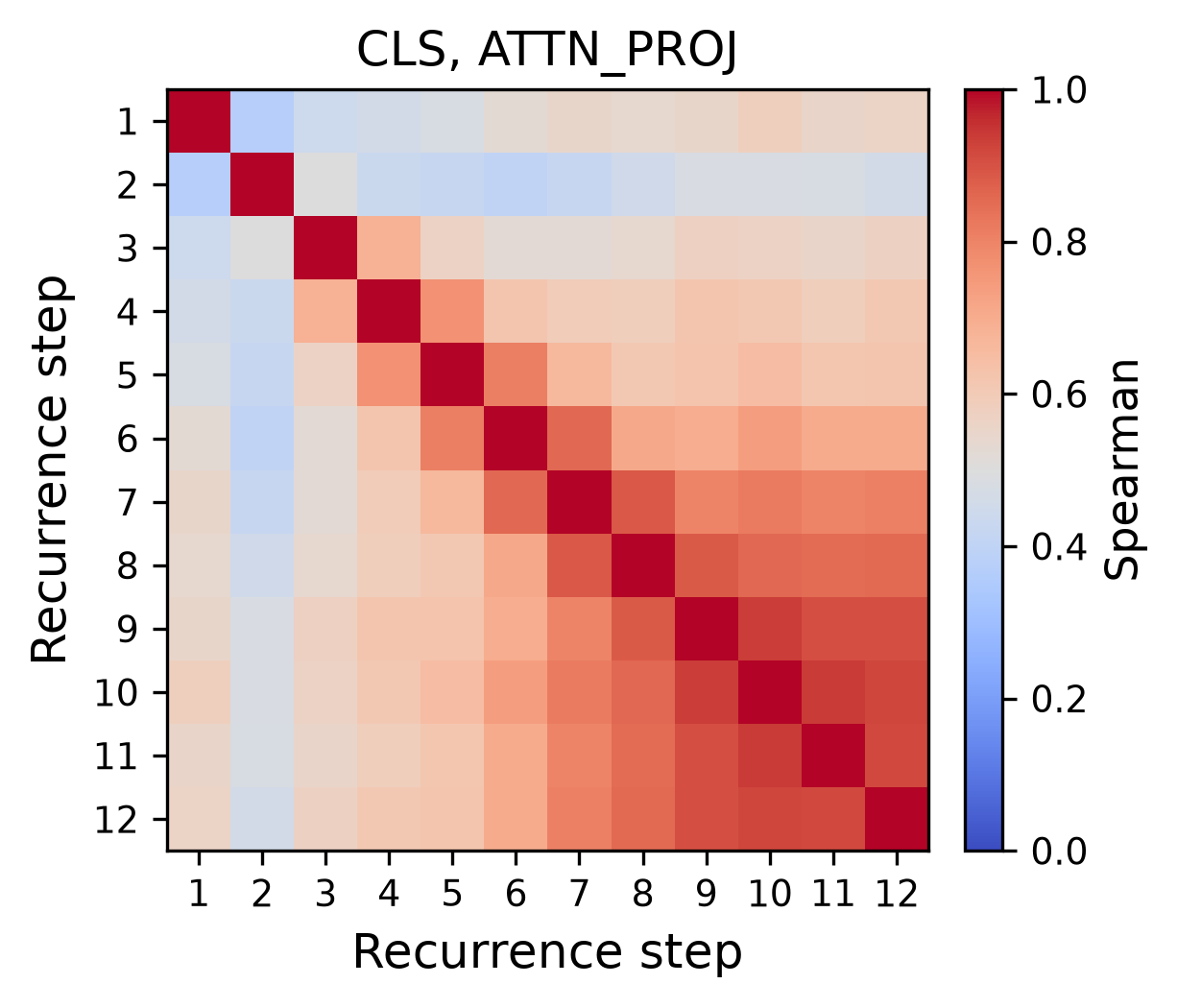}
    \end{minipage}
    \hfill
    \begin{minipage}{0.32\textwidth}
        \centering
        \includegraphics[width=\textwidth]{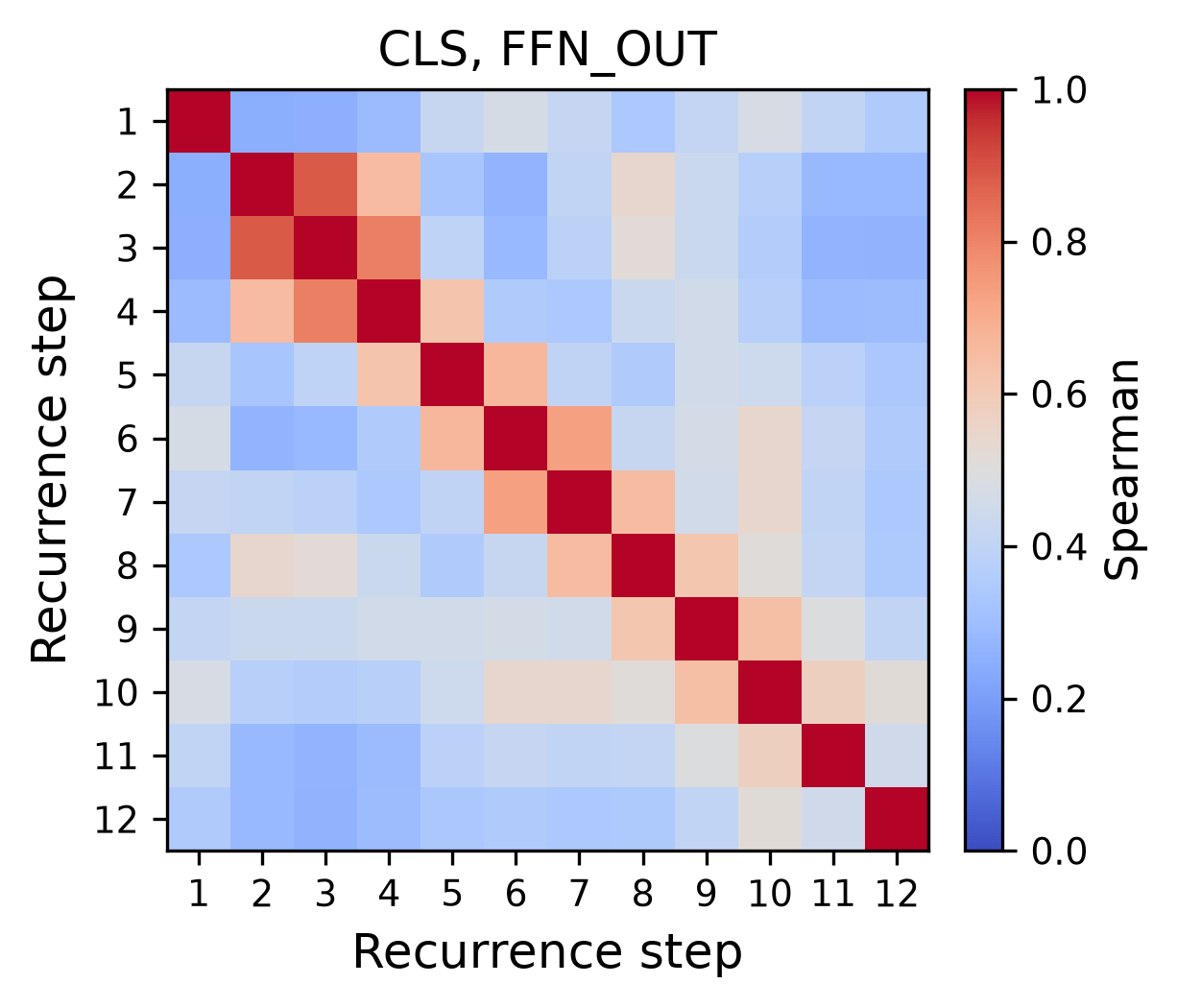}
    \end{minipage}

    \vspace{0.25cm}


    \begin{minipage}{0.32\textwidth}
        \centering
        \includegraphics[width=\textwidth]{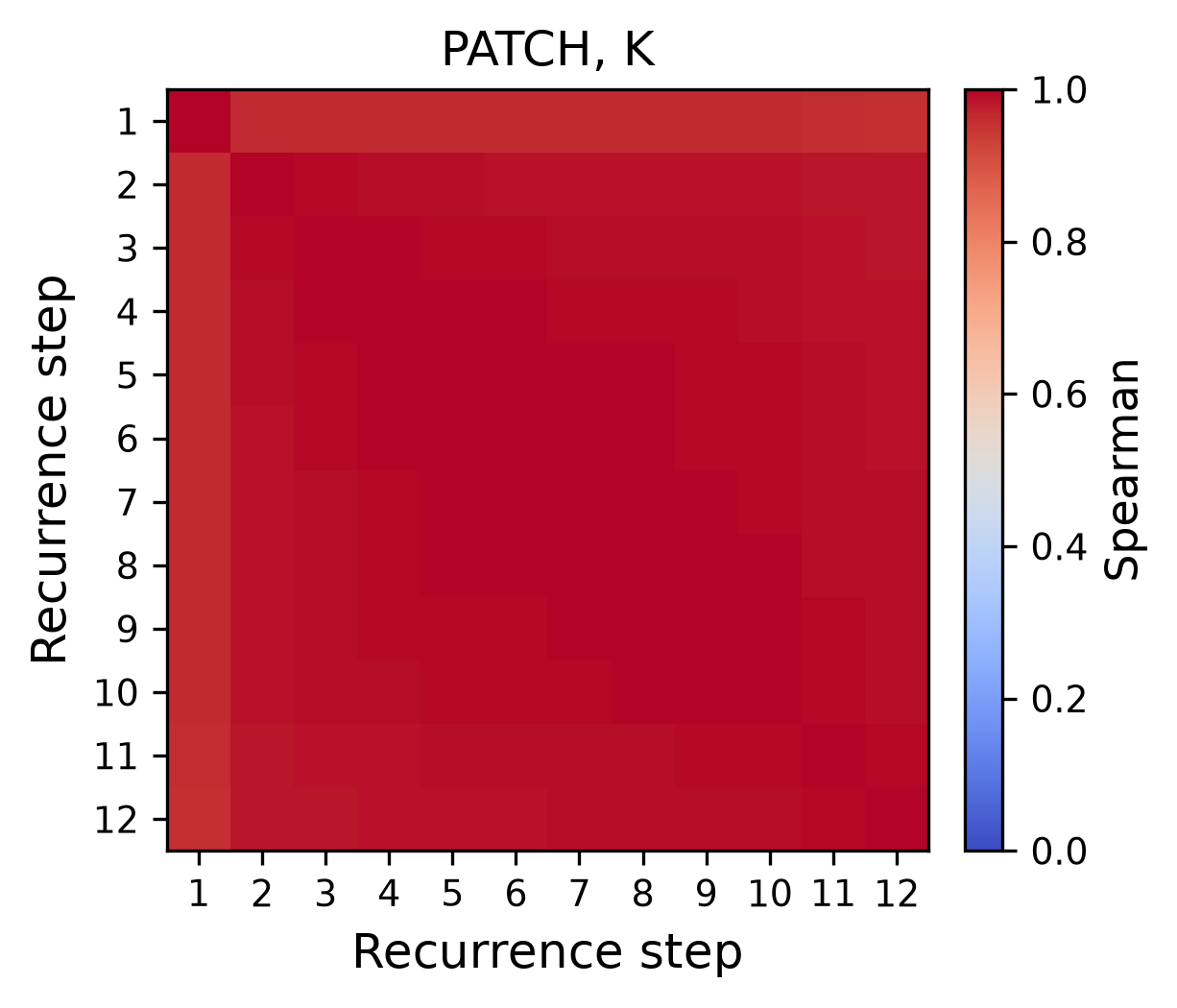}
    \end{minipage}
    \hfill
    \begin{minipage}{0.32\textwidth}
        \centering
        \includegraphics[width=\textwidth]{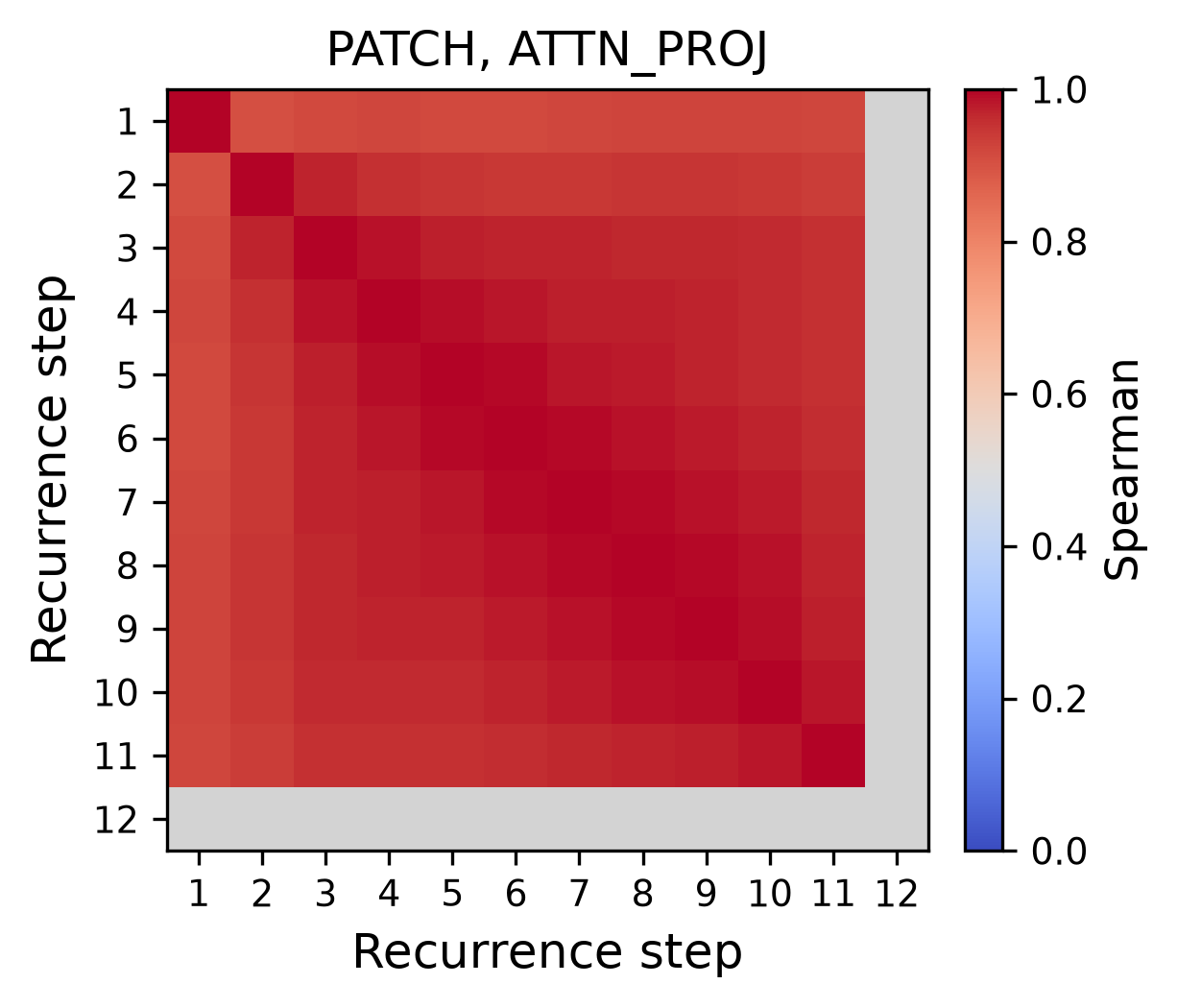}
    \end{minipage}
    \hfill
    \begin{minipage}{0.32\textwidth}
        \centering
        \includegraphics[width=\textwidth]{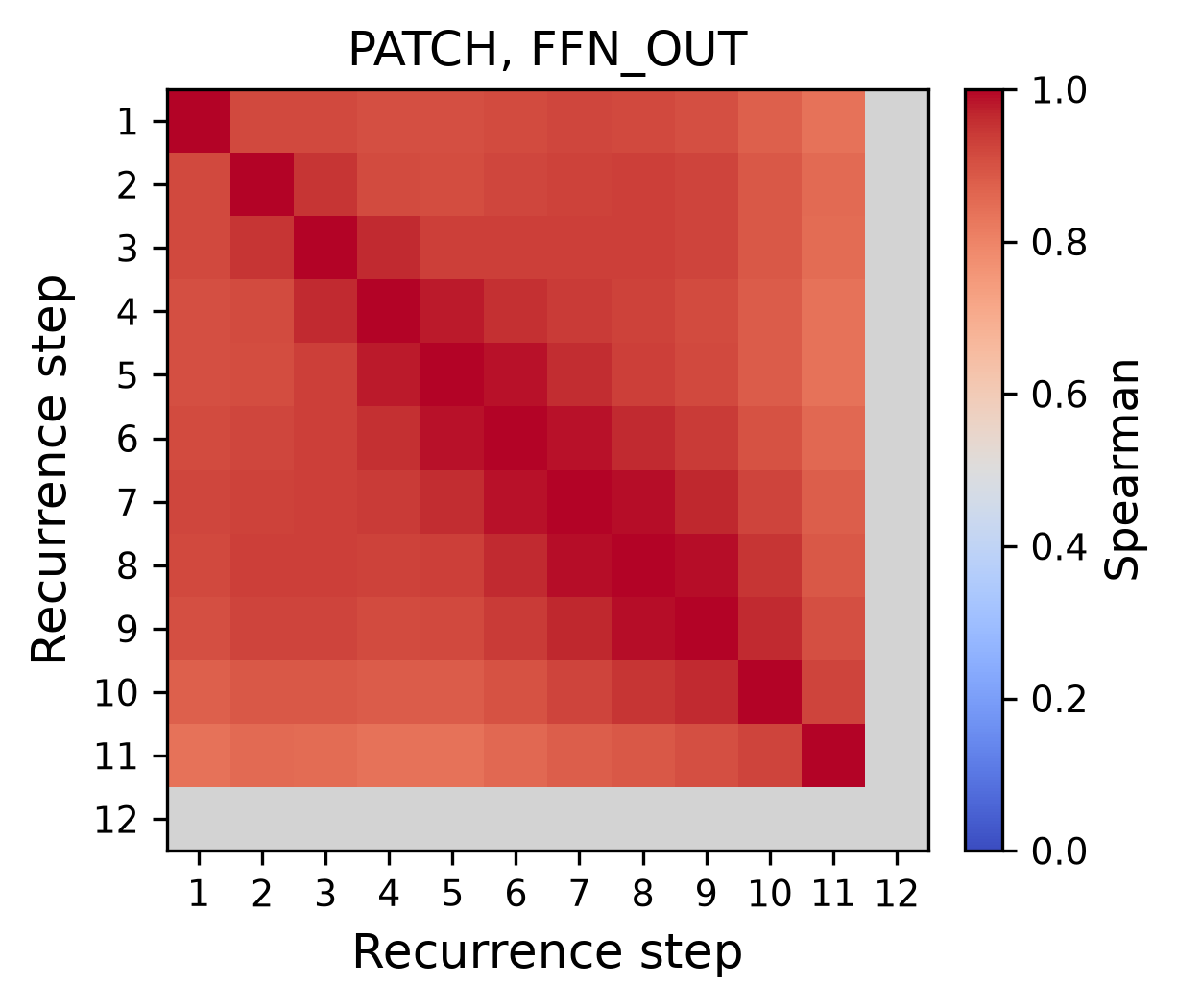}
    \end{minipage}

\caption{
Similarity of step-conditioned spectral importance profiles across recurrent depth for K, the attention output projection, and FFN\_OUT. The top row shows importance profiles derived from the CLS token and the bottom row shows profiles derived from patch tokens. Each matrix reports the Spearman correlation between profiles $p_t$ obtained at different recurrent steps. Patch-token profiles are comparatively stable across recurrence, while CLS profiles exhibit stronger step-dependent variation, particularly for FFN\_OUT. Gray entries indicate undefined correlations.}

    \label{fig:spearman_app}
\end{figure*}

\begin{figure*}[!h]
    \centering

    \begin{minipage}{0.32\textwidth}
        \centering
        \includegraphics[width=\textwidth]{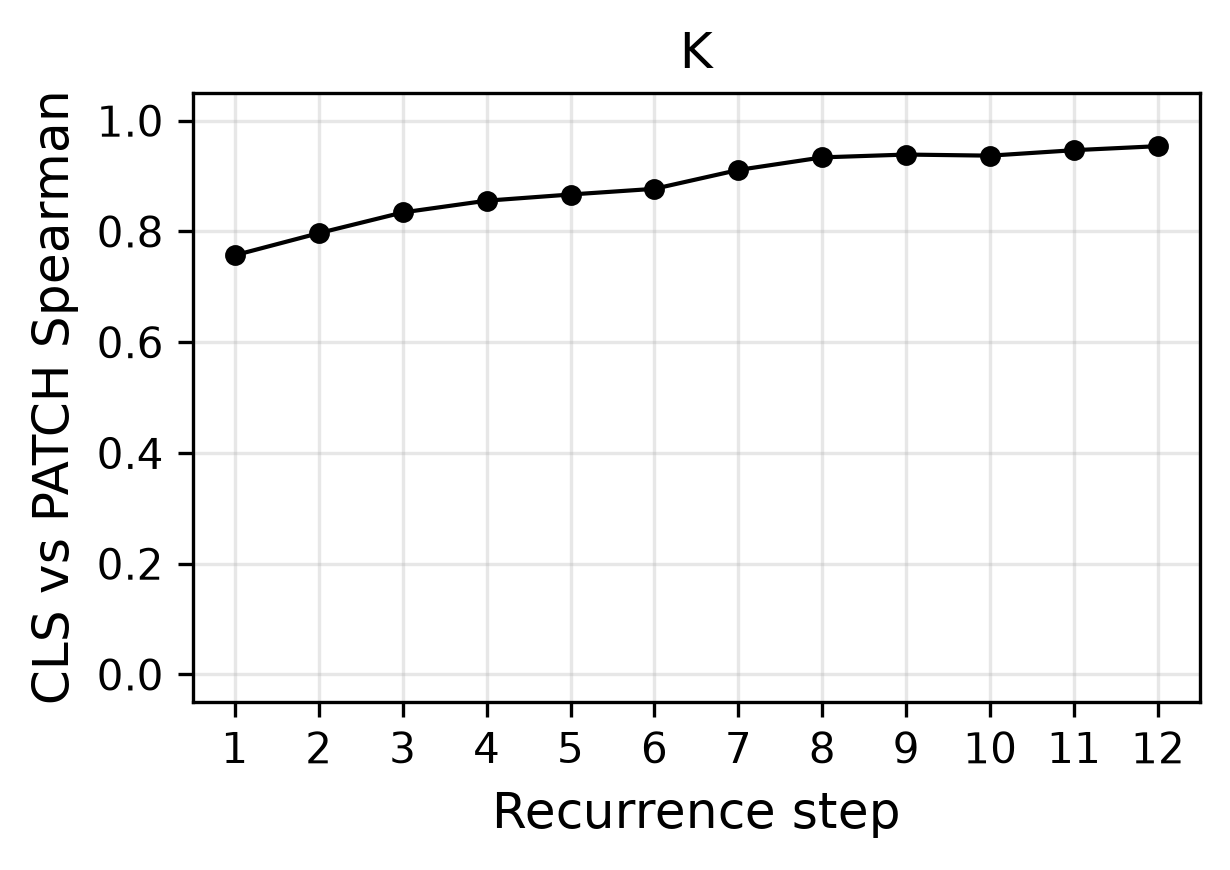}
    \end{minipage}
    \hfill
    \begin{minipage}{0.32\textwidth}
        \centering
        \includegraphics[width=\textwidth]{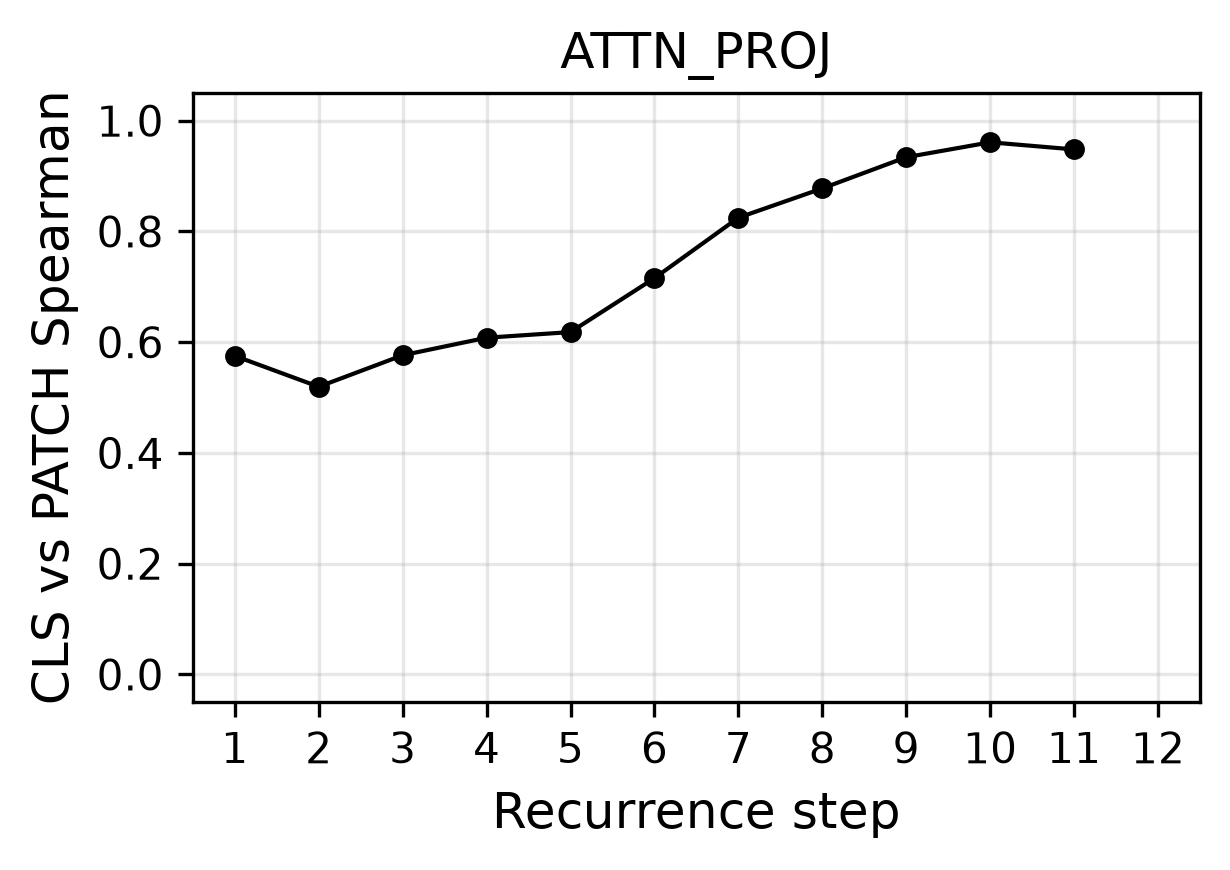}
    \end{minipage}
    \hfill
    \begin{minipage}{0.32\textwidth}
        \centering
        \includegraphics[width=\textwidth]{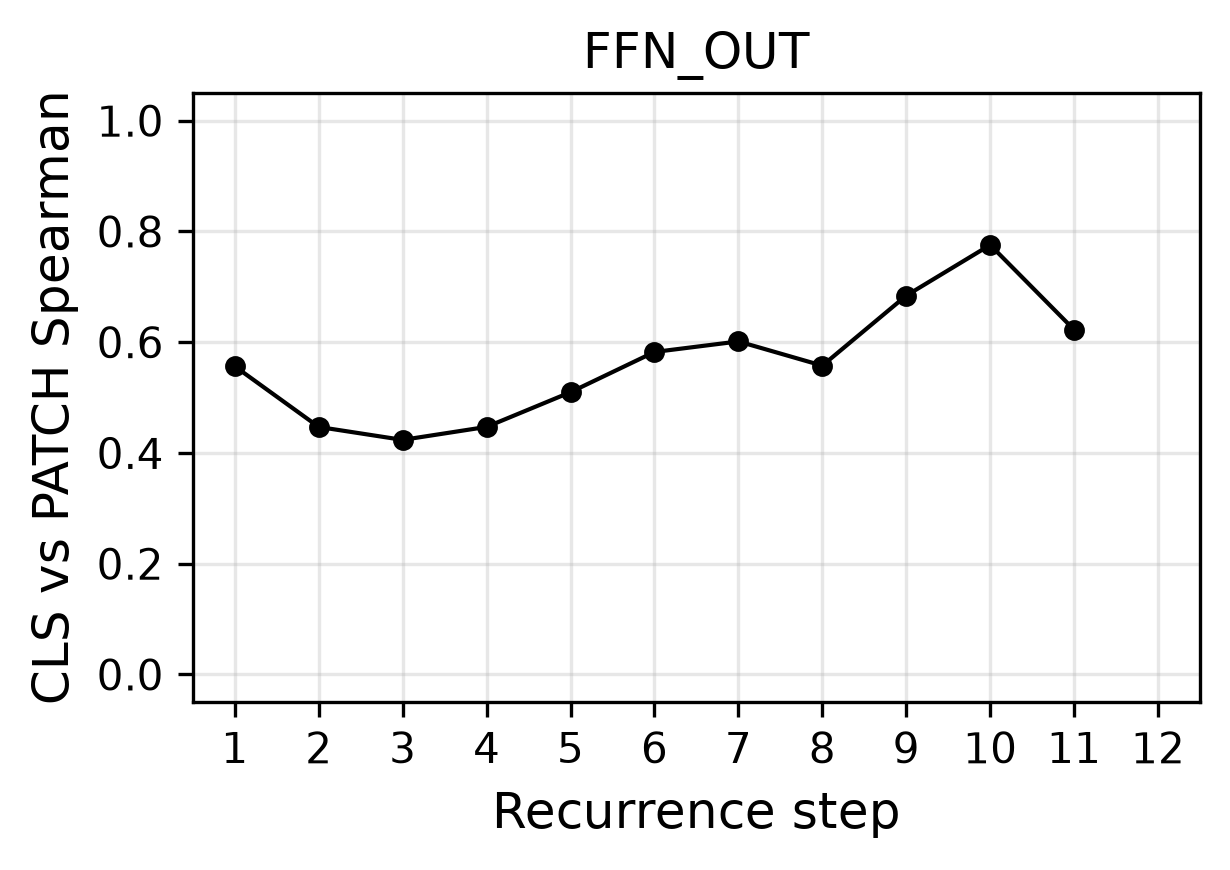}
    \end{minipage}

\caption{
Similarity between CLS and patch spectral importance profiles across recurrent depth for K, the attention output projection, and FFN\_OUT. Each curve reports the Spearman correlation between the CLS and patch importance profiles at a given recurrent step. K and the attention output projection become progressively more aligned across recurrence, while FFN\_OUT exhibits a less monotonic but clearly step-dependent pattern.
}

    \label{fig:clsvspatch_app}
\end{figure*}

\begin{figure*}[!h]
    \centering


    \begin{minipage}{0.32\textwidth}
        \centering
        \includegraphics[width=\textwidth]{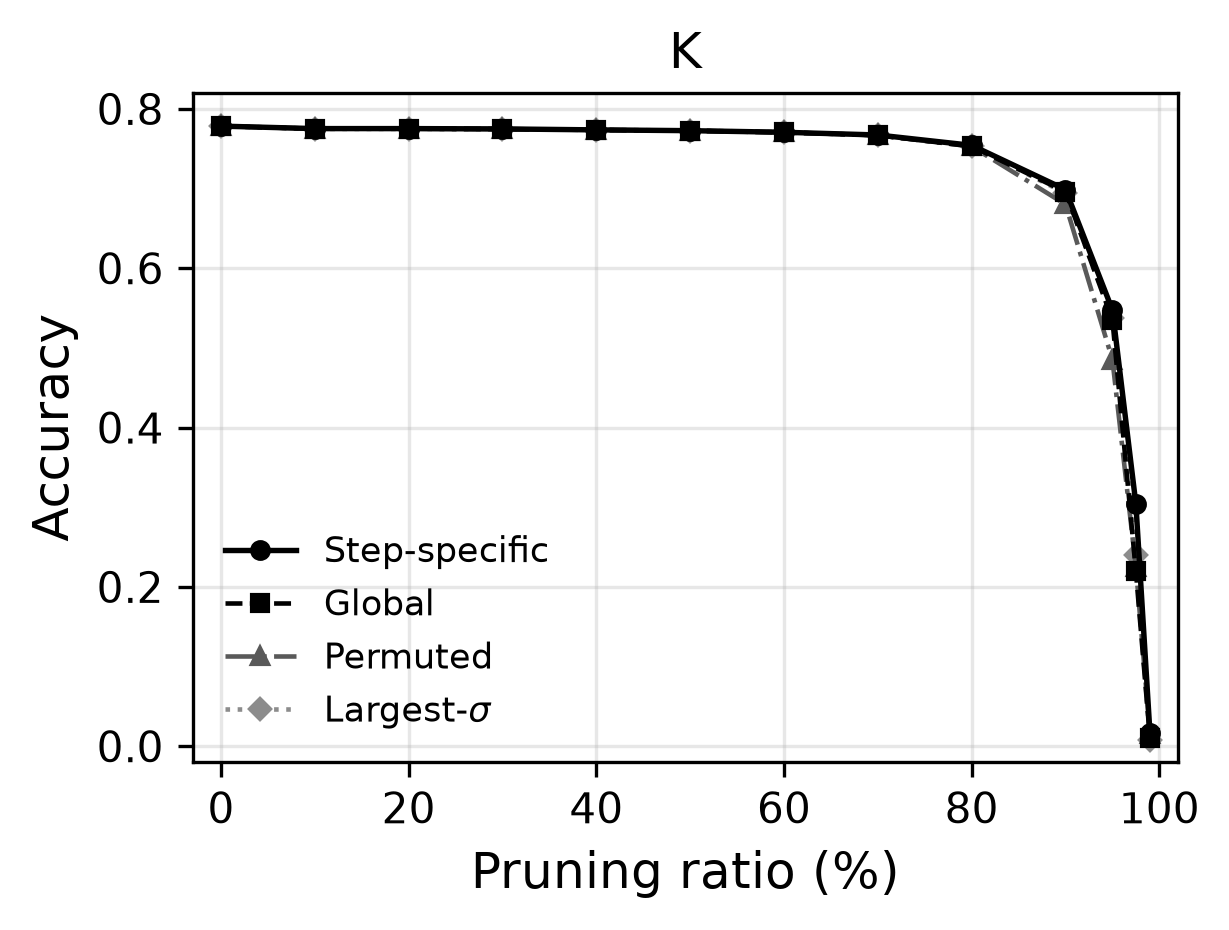}
    \end{minipage}
    \hfill
    \begin{minipage}{0.32\textwidth}
        \centering
        \includegraphics[width=\textwidth]{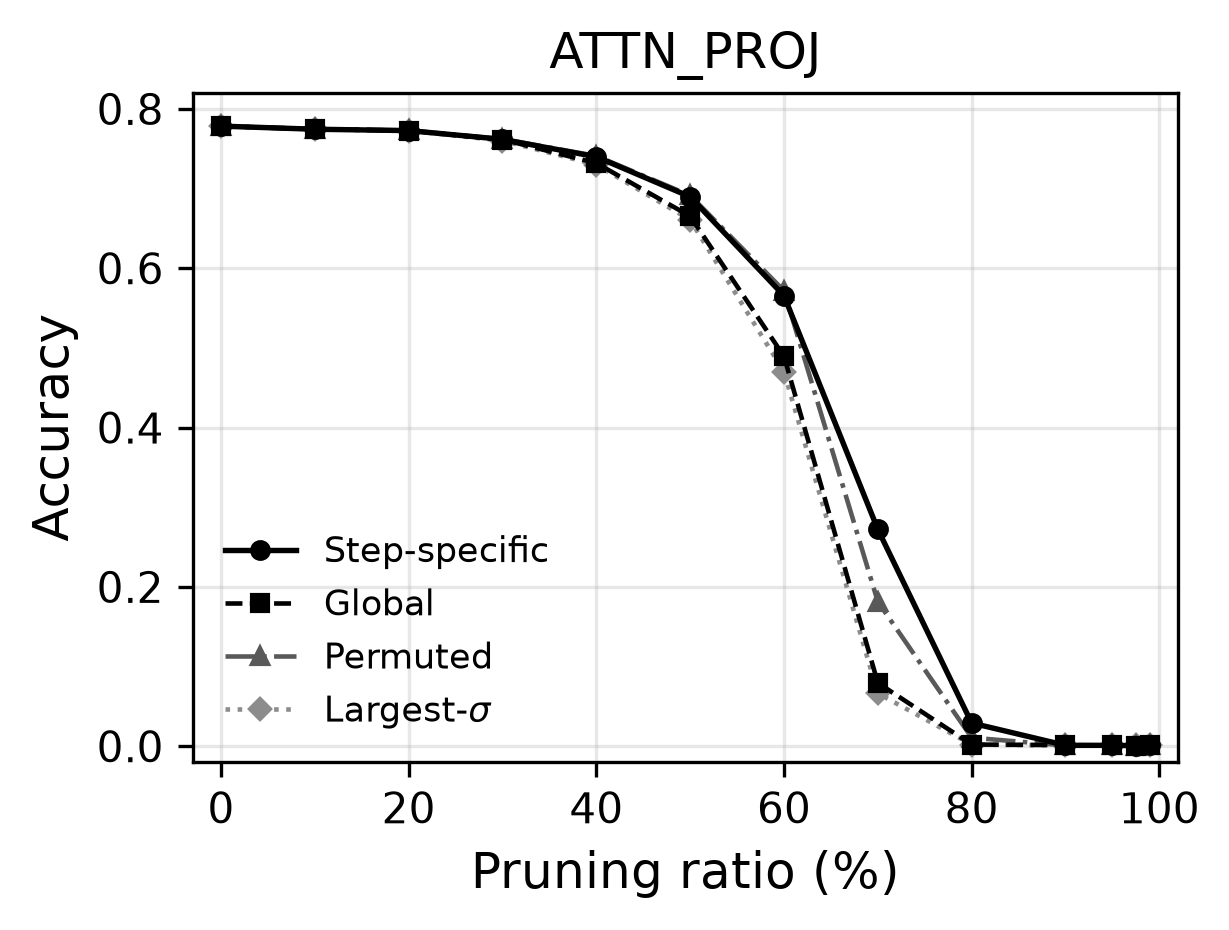}
    \end{minipage}
    \hfill
    \begin{minipage}{0.32\textwidth}
        \centering
        \includegraphics[width=\textwidth]{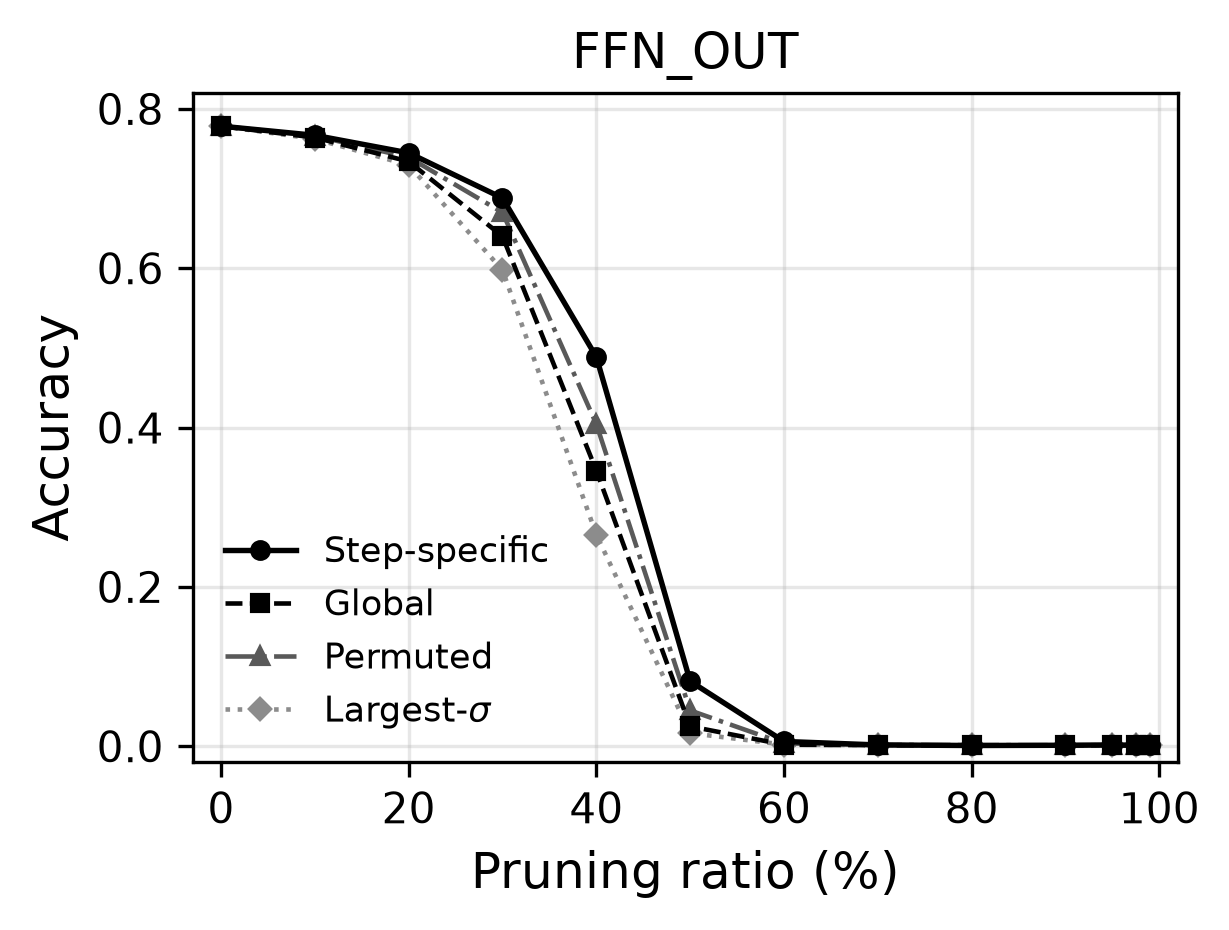}
    \end{minipage}

    \vspace{0.25cm}


    \begin{minipage}{0.32\textwidth}
        \centering
        \includegraphics[width=\textwidth]{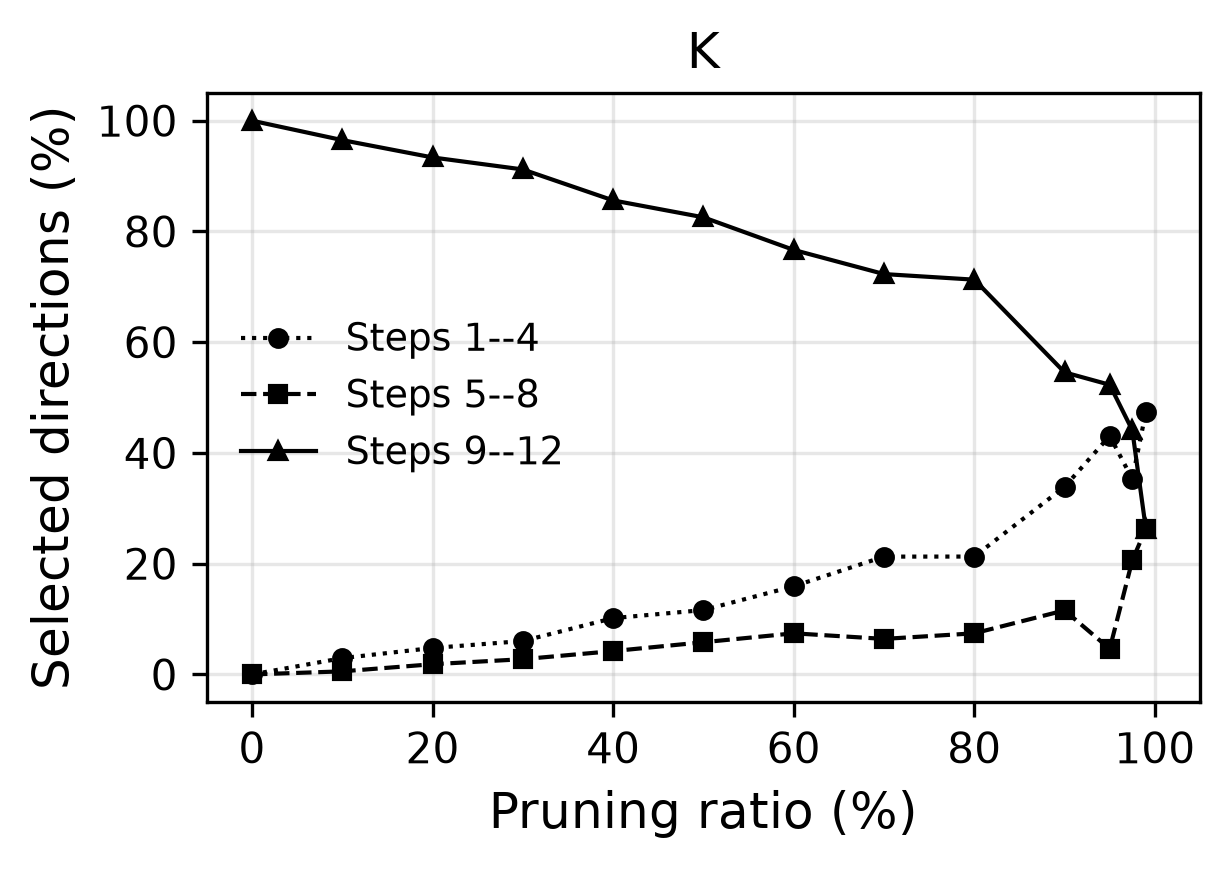}
    \end{minipage}
    \hfill
    \begin{minipage}{0.32\textwidth}
        \centering
        \includegraphics[width=\textwidth]{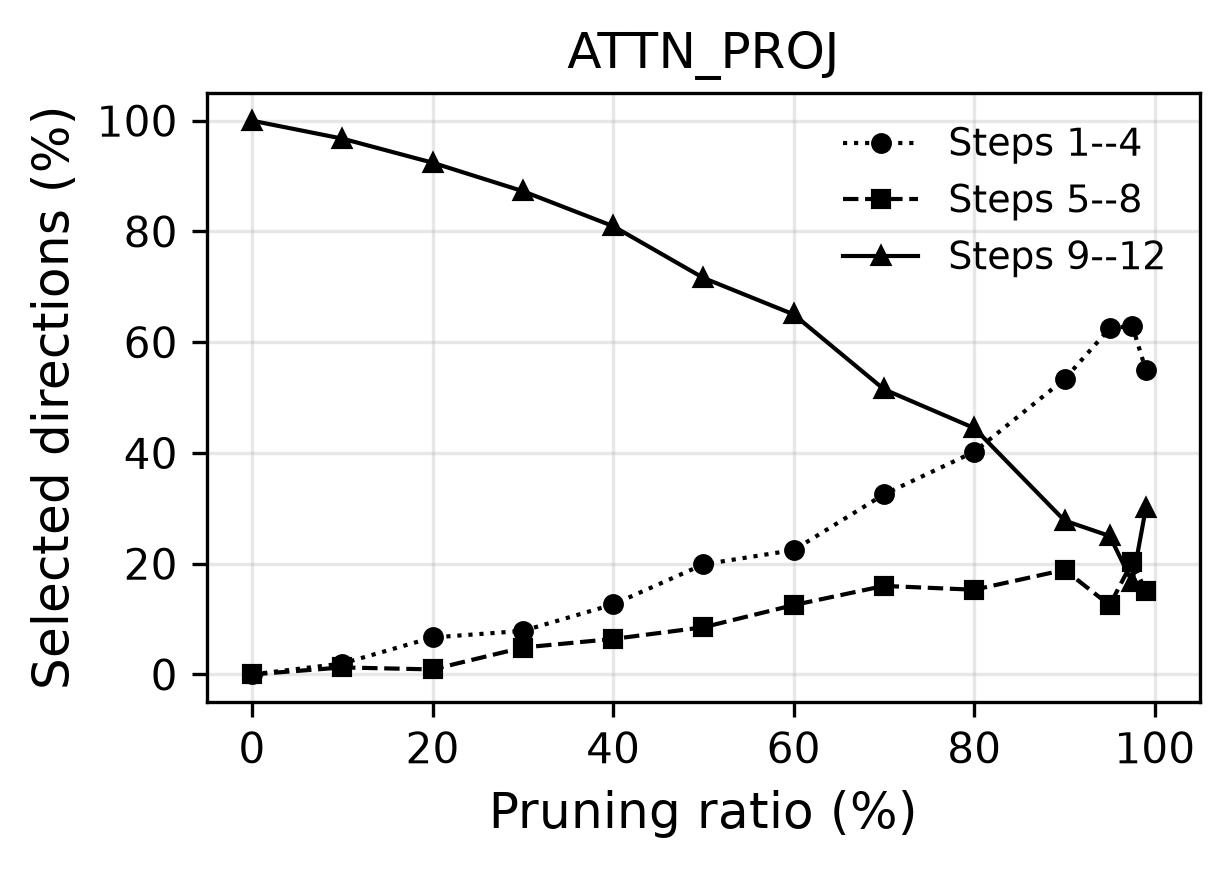}
    \end{minipage}
    \hfill
    \begin{minipage}{0.32\textwidth}
        \centering
        \includegraphics[width=\textwidth]{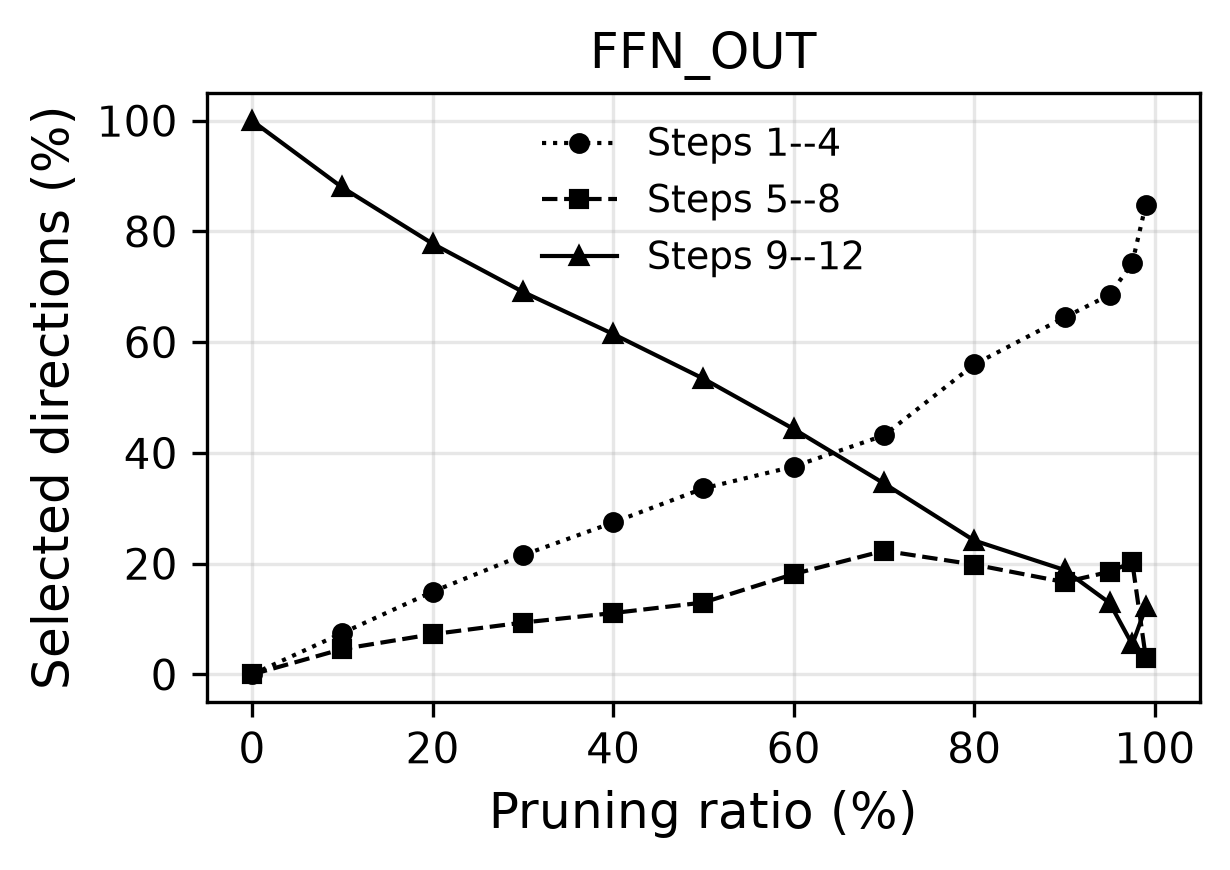}
    \end{minipage}

\caption{
Step-conditioned spectral utilization of K, the attention output projection, and FFN\_OUT. The top row shows validation accuracy after spectral pruning without retraining. Step-conditioned pruning is compared with a global gradient mask shared across recurrent steps, temporally permuted step-conditioned masks, and selection by singular-value magnitude. All methods retain the same number of spectral directions at each recurrent step. The bottom row shows the temporal persistence of directions selected by step-conditioned pruning, grouped according to whether they are retained in $1$--$4$, $5$--$8$, or $9$--$12$ recurrent steps. Across these projections, highly persistent directions remain substantially more frequent than expected under the random-selection baseline, while directions with intermediate persistence are comparatively rare.
}

    \label{fig:pruning_svd_app}
\end{figure*}

\begin{figure}[!h]
    \centering
    \includegraphics[width=0.5\linewidth]{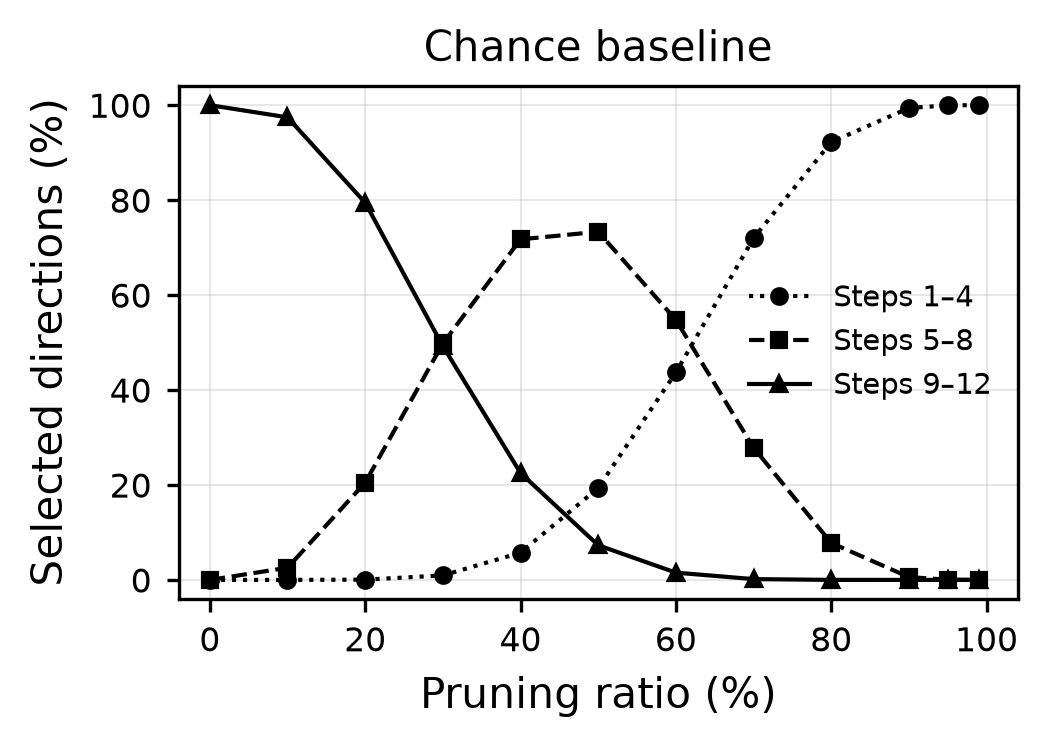}
    \caption{
    Chance baseline for the temporal persistence of selected spectral directions. Each recurrent step independently retains a direction with probability equal to the retained spectral fraction. The curves show the expected fraction of selected directions appearing in $1$--$4$, $5$--$8$, or $9$--$12$ recurrent steps, conditioned on being selected at least once. This provides a reference for interpreting the persistence patterns observed under step-conditioned spectral pruning.
    }
    \label{fig:null_persistence}
\end{figure}

\clearpage
\section{Transfer learning}
\label{app:transfer}

We compare the transfer learning capabilities of ViT-B and bViT-B using linear probing, full fine-tuning, and LoRA-based parameter-efficient fine-tuning. For bViT-B, we additionally evaluate time-embedding tuning. For each experiment, we initialize the backbone from an ImageNet-1K pretrained checkpoint and replace the original classification head with a newly initialized linear layer matching the number of classes in the target dataset. The eight downstream datasets used in our experiments are summarized in Table~\ref{tab:tl_appendix}.

In the linear probing setting, the pretrained backbone is frozen and only the newly initialized classification head is optimized. Full fine-tuning jointly optimizes the classification head and all backbone parameters. For LoRA-based adaptation, the pretrained backbone remains frozen and low-rank adapters are trained together with the classification head. We consider rank-8 LoRA applied either to the query and value projections of self-attention or to the FFN layers~\citep{hu2022lora}. For time-embedding tuning, we equip bViT-B with one learnable embedding for each recurrent step. These embeddings are initialized to zero, while the pretrained shared block remains frozen. Only the time embeddings and classification head are optimized for the downstream task. The shared block can therefore be viewed as a fixed recurrent operator, with the step-specific embeddings providing lightweight control signals that steer its computation across recurrent depth.

We approximately follow optimization setups used in recent parameter-efficient transfer learning studies~\citep{hu2022lora,liuni}. All models are optimized with AdamW using cross-entropy loss for 40 epochs. We use a batch size of 128 for CIFAR10 and CIFAR100 and 64 for the remaining datasets, with linear warmup during the first 5 epochs followed by cosine learning-rate decay. For linear probing, we use a learning rate of $1\times10^{-3}$ and weight decay of $0.001$. For full fine-tuning, we use a learning rate of $5\times10^{-5}$, a minimum learning rate of $1\times10^{-6}$, and weight decay of $0.05$. All backbone components are trainable in this setting, including attention and FFN layers, patch and positional embeddings, the \textsc{cls} token, and normalization layers. For LoRA fine-tuning, both adapter variants use rank $r=8$, scaling factor $\alpha=8$, and no LoRA dropout. We use a learning rate of $5\times10^{-4}$, a minimum learning rate of $5\times10^{-6}$, and weight decay of $10^{-4}$. All remaining backbone parameters are frozen. We report the mean and standard deviation over three training runs. All experiments are conducted on a single NVIDIA H100 GPU.

\begin{table}[h!]
\centering
\caption{Datasets used for the transfer learning experiments.}
\label{tab:tl_appendix}
\begin{tabular}{lrrr}
\toprule
Dataset & Train size & Test size & \#classes \\
\midrule
CIFAR10~\citep{Krizhevsky09learningmultiple}
    & 50,000 & 10,000 & 10 \\
CIFAR100~\citep{Krizhevsky09learningmultiple}
    & 50,000 & 10,000 & 100 \\
Oxford-IIIT Pets~\citep{parkhi2012cats}
    & 3,680 & 3,669 & 37 \\
DTD~\citep{cimpoi14describing}
    & 1,880 & 1,880 & 47 \\
Flowers-102~\citep{nilsback2008automated}
    & 2,040 & 6,149 & 102 \\
RESISC45~\citep{cheng2017remote}
    & 25,200 & 6,300 & 45 \\
FGVC Aircraft~\citep{maji2013fine}
    & 6,667 & 3,333 & 100 \\
EuroSAT~\citep{helber2019eurosat}
    & 21,600 & 5,400 & 10 \\
\bottomrule
\end{tabular}
\end{table}

\end{document}

%% file: math_commands.tex
\usepackage{amsmath,amsfonts,bm}

\def\eqref#1{equation~\ref{#1}}

\def\1{\bm{1}}

\DeclareMathAlphabet{\mathsfit}{\encodingdefault}{\sfdefault}{m}{sl}
\SetMathAlphabet{\mathsfit}{bold}{\encodingdefault}{\sfdefault}{bx}{n}

